\documentclass[11pt]{article}

\usepackage[T1]{fontenc}
\usepackage[utf8]{inputenc}
\usepackage[letterpaper,margin=1in]{geometry}
\usepackage[numbers,sort&compress]{natbib}

\usepackage{graphicx}
\usepackage{amsmath}
\usepackage{amsfonts}
\usepackage{amssymb}
\usepackage{subcaption}
\usepackage{float}
\usepackage{booktabs}
\usepackage{multirow}
\usepackage{xcolor}
\usepackage{mathtools}
\usepackage{bm}
\usepackage{microtype}
\usepackage{hyperref}
\usepackage{algorithm}
\usepackage{algpseudocode}
\usepackage{enumitem}
\newcommand{\cmark}{\ensuremath{\checkmark}}
\newcommand{\xmark}{\ensuremath{\times}}

\usepackage{amsthm}
\theoremstyle{plain}
\newtheorem{theorem}{Theorem}
\newtheorem{proposition}[theorem]{Proposition}

\theoremstyle{definition}
\newtheorem{definition}[theorem]{Definition}
\theoremstyle{plain}
\newtheorem*{theorem*}{Theorem}
\theoremstyle{remark}

\newcommand{\R}{\mathbb{R}}
\newcommand{\N}{\mathcal{N}}
\newcommand{\HPZ}{\mathcal{Z}}

\title{Hybrid Probabilistic Zonotopes for Identifiable and Refinable Predictive Uncertainty}

\author{%
  Zhen Zhang \qquad Amr Alanwar \\
  School of Computation, Information and Technology \\
  Technical University of Munich, Germany \\
  \texttt{\{zhenzhang.zhang, alanwar\}@tum.de}
}
\date{}

\begin{document}

\maketitle

\begin{abstract}
Probabilistic prediction heads in neural networks typically output either a Gaussian mixture or a single conformal region. Neither separates the distinct sources of uncertainty often present in real prediction tasks: a discrete choice among modes, bounded systematic drift within the chosen mode, and irreducible stochastic noise. We introduce the Hybrid Probabilistic Zonotope (HProbZ), an output head that represents these three sources as binary, bounded, and stochastic generators of a zonotope, and admits a closed-form likelihood by convolution. Sharing the bounded generator across prediction steps couples future predictions algebraically, so observing one step refines the predictive distribution at every remaining step in a single forward pass. We establish that the three generators are identifiable from the likelihood up to permutation, and that an HProbZ density is representationally distinct from any finite Gaussian mixture. The same shared structure provides analytic per-mode risk and distribution-free multi-modal conformal sets at inference time. Empirical analysis on representative prediction benchmarks supports the effectiveness of the design relative to same-encoder mixture baselines, while offering structural properties that mixture or convex-conformal predictors do not jointly provide.
\end{abstract}

\section{Introduction}
\label{sec:intro}

Many predictive tasks involve three distinct uncertainty sources: discrete modal choice, bounded systematic drift, and irreducible stochastic noise. Trajectory prediction is a canonical instance, arising whenever a model must commit to a discrete behaviour \emph{and} report within-mode uncertainty that sharpens as observations arrive. Existing heads address one source or the other: mixture density networks commit to sharp modes but their within-mode covariances are network outputs that observations can only re-weight; conformal methods give distribution-free coverage but as a single convex set that inflates on multi-modal targets. We introduce a predictive distribution offering four properties simultaneously --- identifiable modal/bounded/stochastic decomposition, observation-driven contraction, analytic per-mode risk, and multi-modal conformal coverage --- that prior approaches provide individually but not jointly (Tab.~\ref{tab:design_space}).


\begin{figure}[t]
\centering
\includegraphics[width=0.48\linewidth]{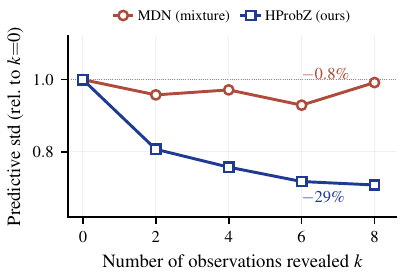}\hfill
\includegraphics[width=0.50\linewidth]{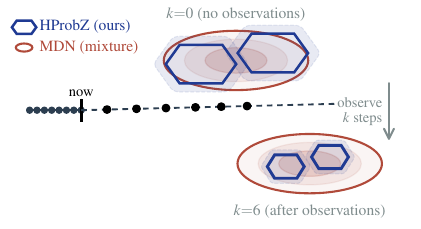}
\caption{Overview. \textbf{Left:} ETH/UCY LOO: revealing $k{=}8$ observations contracts HProbZ's predictive std by $23.6\%$ through joint posterior updates over shared $(\beta,\alpha)$, vs.\ $3.6\%$ for a same-capacity Gaussian mixture (weight reweighting only); mean over 3 seeds. \textbf{Right:} schematic at $k{=}0$ vs $k{=}6$: each HProbZ hexagon is $c{+}G_b\beta{+}G_d\alpha$ inflated by the $2\sigma$ halo of $G_s\nu$. Downstream (nuScenes, $N{=}39{,}829$, $6$\,s, static obstacle): at near-matched $10\%$ FAR HProbZ produces $\mathbf{30\%}$ fewer trajectory-level collisions than same-encoder MDN-K2 (Tab.~\ref{tab:closed_loop}); per-step reverses.}
\label{fig:teaser}
\end{figure}

Hybrid Probabilistic Zonotopes (HProbZ) add three generators to a centre point as a single neural output: a binary generator for discrete modes, a bounded generator (shared across all future timesteps) for continuous drift, and a Gaussian generator for noise. Because the bounded generator is shared, revealing one step pins it down and sharpens every other step in one forward pass. Fig.~\ref{fig:teaser} contrasts this with a same-capacity Gaussian mixture: HProbZ's predictive support shrinks under observation, while the mixture can only re-weight components and leaves within-mode covariance frozen. A shared-latent CVAE contracts but collapses modal, drift, and noise into one Gaussian, erasing the separation downstream planners and conformal-set users need; HProbZ keeps all three sources identifiable from one closed-form likelihood.

\textbf{Contributions.} \textbf{(i)} We introduce the \emph{Hybrid Probabilistic Zonotope} (HProbZ): a structured-output head whose binary, bounded, and stochastic generators are trained jointly under a single closed-form convolution likelihood (\S\ref{sec:hprobz}--\S\ref{sec:training}). \textbf{(ii)} We establish three properties of this representation: HProbZ generators are uniquely identifiable from the exact likelihood; HProbZ densities lie strictly outside the finite Gaussian-mixture family; and the entropy cost of compact support is bounded by a universal constant per dimension. \textbf{(iii)} We show that sharing the bounded generator across prediction steps yields single-pass observation-driven uncertainty contraction, analytic per-mode risk, and distribution-free multi-modal conformal sets --- a combination prior heads provide individually but not jointly (Tab.~\ref{tab:design_space}). \textbf{(iv)} We validate the representation on three trajectory-forecasting benchmarks and a closed-loop stress test (\S\ref{sec:experiments}).

\begin{table}[t]
\centering
\small
\caption{Design-space positioning. HProbZ targets a combination of structural properties absent from same-encoder mixtures, conformal predictors, and recent SOTA forecasters; minADE is one axis, not the optimisation target. \textbf{Decomp.}: identifiable modal/bounded/stochastic factorisation. \textbf{Multi-modal CP}: distribution-free non-convex conformal set. \textbf{Refinement}: $O(1/k)$ contraction (Theorem~\ref{thm:contraction}).}
\label{tab:design_space}
\resizebox{\linewidth}{!}{\begin{tabular}{lcccccc}
\toprule
\textbf{Method} & \textbf{minADE@20} & \textbf{Steps} & \textbf{Decomp.} & \textbf{Analytic risk} & \textbf{Multi-modal CP} & \textbf{Refinement} \\
\midrule
LED~\citep{mao2023leapfrog}          & \textbf{0.18} & 5 & \xmark & \xmark & \xmark & \xmark \\
EqMotion~\citep{xu2023eqmotion}      & 0.19 & 1 & \xmark & \xmark & \xmark & \xmark \\
Trajectron++~\citep{salzmann2020trajectron++} & 0.30 & 1 & \xmark & partial & \xmark & \xmark \\
MDN-K20~\citep{bishop1994mixture}     & 0.32 & 1 & \xmark & \xmark & per-component & \xmark \\
\textbf{HProbZ (ours)}                & 0.20 & 1 & \cmark & \cmark & \cmark & \cmark~($-23.6\%$) \\
\bottomrule
\end{tabular}}
\end{table}

Raw point accuracy is one axis and HProbZ does not occupy its extreme; we occupy an axis on which no existing method we tested provides all four structural properties simultaneously (Tab.~\ref{tab:design_space}).

\section{Preliminaries}
\label{sec:background}

A zonotope $\{c + G\beta \mid \beta \in [-1,1]^m\}$ is a centrally symmetric polytope with center $c \in \R^d$ and generator matrix $G \in \R^{d \times m}$. A constrained zonotope adds linear equality constraints $A\beta = b$, restricting to a subset of the zonotope's interior. The hybrid zonotope~\citep{bird2023hybrid} augments with binary factors $\xi \in \{-1,1\}^{n_b}$ and represents a non-convex region as a union of up to $2^{n_b}$ constrained zonotopes. Stochastic-reachability work, originating in safety analysis of automated traffic~\citep{althoff2007safety}, instead augments zonotopes with Gaussian factors $\eta \sim \N(0,I)$, inducing a distribution over the zonotope's interior. An MDN~\citep{bishop1994mixture} parameterizes $p(y|x) = \sum_{k=1}^K \pi_k(x) \N(y; \mu_k(x), \Sigma_k(x))$ via network outputs; the within-component covariances $\Sigma_k$ are fixed outputs that cannot be refined by observations, and trajectory-prediction MDNs~\citep{cui2019multimodal} are known to suffer mode-collapse failures~\citep{makansi2019overcoming} that motivate the structured-decomposition alternative we develop below.

\paragraph{The gap.} The hybrid form carries discrete modes but no likelihood; the probabilistic form carries a likelihood but only a single mode; an MDN supplies both, but its independent-Gaussian components let observations only re-weight modes, never tighten within-component covariance. A predictive distribution that keeps binary modes separable from bounded drift and Gaussian noise, admits a closed-form likelihood, and tightens its bounded support under observation through shared algebraic structure exists in neither lineage; this is the gap HProbZ fills in \S\ref{sec:hprobz}.

\section{Hybrid Probabilistic Zonotopes}
\label{sec:hprobz}

\begin{definition}[Hybrid Probabilistic Zonotope (HProbZ)]
\label{def:hprobz}
An HProbZ is a random set defined by
\begin{equation}
\label{eq:hprobz}
\HPZ = \left\{ c + G_b\beta + G_d\alpha + G_s\nu \;\middle|\; \begin{matrix} \beta \in \{-1,1\}^{n_b},\; \alpha \in [-1,1]^{n_d},\; \nu \sim \N(0,I_{n_q}), \\ A_b\beta + A_d\alpha + A_s\nu = b \end{matrix} \right\}
\end{equation}
where $c \in \R^d$ is the center, $G_b \in \R^{d \times n_b}$, $G_d \in \R^{d \times n_d}$, and $G_s \in \R^{d \times n_q}$ are the binary, bounded, and stochastic generator matrices, respectively, and $A_b \in \R^{p \times n_b}$, $A_d \in \R^{p \times n_d}$, $A_s \in \R^{p \times n_q}$, $b \in \R^p$ encode $p$ linear equality constraints.
\end{definition}

The single-mode slice ($n_b{=}0$, $p{=}0$) recovers the probabilistic zonotope of stochastic reachability~\citep{althoff2007safety}; HProbZ extends it with binary modes and observation-activated constraints. Paralleling the hybrid zonotope~\citep{bird2023hybrid}, for each binary assignment $\beta$ the constraint $A_b\beta + A_d\alpha + A_s\nu = b$ reduces to a Constrained Probabilistic Zonotope $\mathcal{C}(\beta)$ with center $c + G_b\beta$ and residual constraints $A_d\alpha + A_s\nu = b - A_b\beta$, so the HProbZ is a union of $2^{n_b}$ CPZs:
\begin{equation}
\label{eq:union}
    \HPZ = \bigcup_{\beta \in \{-1,1\}^{n_b}} \mathcal{C}(\beta).
\end{equation}
When $p{=}0$ all factors are independent and $z = c + G_b\beta + G_d\alpha + G_s\nu$; this is the head used by the network in \S\ref{sec:training}. Observations introduce new constraint rows into $(A_b, A_d, A_s, b)$ for sequential refinement (\S\ref{sec:refinement}).

Setting all generators to zero recovers a point prediction; retaining only $G_d$ yields a zonotope, only $G_s$ a Gaussian, $G_b{+}G_s$ a $2^{n_b}$-component shared-covariance GMM, and diagonal $G_d$ alone an axis-aligned conformal box. In the unconstrained case the induced measure is a Minkowski-sum convolution (Proposition~\ref{prop:minkowski}):
\begin{equation}
\label{eq:measure_decomp}
\mu_{\HPZ} = \delta_c \ast \underbrace{\Big(\tfrac{1}{2^{n_b}}\sum_{\beta} \delta_{G_b\beta}\Big)}_{\text{discrete modal}}
\ast \underbrace{(G_d)_*\mathcal{U}([-1,1]^{n_d})}_{\text{bounded zonoid}\footnotemark}
\ast \underbrace{\N(0, G_sG_s^\top)}_{\text{Gaussian}},
\end{equation}
where $\ast$ denotes convolution. \footnotetext{$(G_d)_*\mu$ is the pushforward of $\mu$ on $[-1,1]^{n_d}$ under $G_d$, i.e., the zonoid measure on $\mathrm{Image}(G_d)$.} The bounded factor has compact support and is not absolutely continuous w.r.t.\ any Gaussian --- an inductive bias giving HProbZ algebraic properties Gaussian mixtures lack, at a cost of density flexibility.

\begin{theorem}[HProbZ density is not a finite Gaussian mixture]
\label{thm:non-gmm}
If $G_d$ has at least one nonzero entry, then no finite Gaussian mixture $\sum_{k=1}^K \pi_k \N(\mu_k, \Sigma_k)$, for any $K$, has the same distribution as $\mu_{\HPZ}$.
\end{theorem}
\begin{proof}[Proof sketch]
Project both CFs onto a direction where $G_d$ is nonzero and continue to $s=iy$: the HProbZ CF picks up a $\sinh(ay)/(ay)$ factor with an algebraic $1/y$ correction that no finite-GMM CF (a sum of pure exponentials) can match. Full argument in App.~\ref{app:proofs}; empirical fingerprint in Fig.~\ref{fig:nonrep}.
\end{proof}

\begin{figure}[t]
    \centering
    \includegraphics[width=0.95\linewidth]{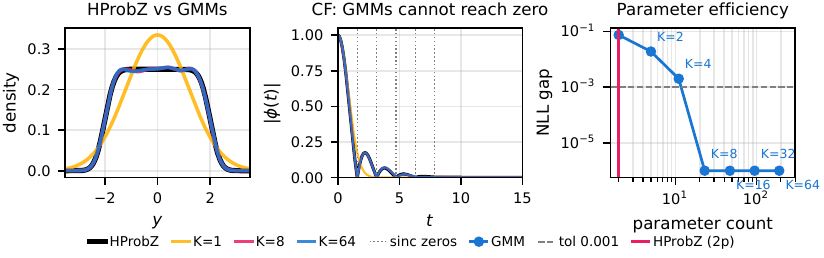}
    \caption{Empirical fingerprint of Theorem~\ref{thm:non-gmm} (1-D HProbZ with $a{=}2,\sigma{=}0.3$, $10^5$ samples). \textbf{Left:} GMMs ($K{=}1{,}\ldots{,}64$) visually approximate the density. \textbf{Middle:} at the sinc-zero frequencies $t_k=k\pi/a$, $|\phi_{\text{GMM}}|$ plateaus at $\sim 2{-}5{\times}10^{-3}$ regardless of $K$ --- the structural signature the proof predicts no finite GMM can erase. \textbf{Right:} matching HProbZ's held-out NLL with a finite GMM requires $K{=}8$ ($23$ free density parameters) vs HProbZ's $2$, an $11.5\times$ ratio. Full setup in App.~\ref{app:theorem_validation}.}
    \label{fig:nonrep}
\end{figure}

A natural question is whether the richer parameterisation is identifiable.

\begin{theorem}[Identifiability under exact likelihood]
\label{thm:identifiability}
Suppose $G_d$ and $G_s$ are diagonal (the parameterisation used in \S\ref{sec:training}), and (i) for each axis $j$ the bounded half-width $a_j = |G_{d,jj}|>0$ and noise scale $\sigma_j = |G_{s,jj}|>0$, and (ii) for each axis $j$ the projected mode centers $\{(G_b\beta)_j\}_{\beta\in\{-1,1\}^{n_b}}$ are pairwise distinct. Then $(c, G_b, G_d, G_s)$ are identifiable from the exact convolution likelihood, up to column-permutations and global column sign-flips of $G_b$.
\end{theorem}
\begin{proof}[Proof sketch]
Per-axis projection yields a 1-D mixture of uniform-normal convolutions: sinc zeros of the 1-D CF give the bounded half-widths $a_j$, the Gaussian envelope gives $\sigma_j$, and classical location-mixture identifiability~\citep{yakowitz1968identifiability} pins down the mode centers. Cross-axis consistency reduces the labeling ambiguity to the cube symmetry above. See App.~\ref{app:proofs}.
\end{proof}

\paragraph{Practical identifiability.} The distinct-centers condition is generic for random $G_b$; SGD on softplus outputs reliably keeps $a_j, \sigma_j$ positive and distinct, and 10 seeds recover ground-truth parameters within $0.008$ (App.~\ref{app:theorem_validation}).

\section{Training HProbZ as a Neural Network Output}
\label{sec:training}

\paragraph{Output parameterization.} We attach an unconstrained HProbZ head ($p{=}0$) to a transformer encoder: $h = \text{Encoder}(x)$, $c = W_c h$ and $G_b = W_b h$ are linear projections, $G_d, G_s$ diagonal softplus outputs. The head adds $(n_b + n_d + n_q + 1)\times d$ parameters, matching MDN-K=$2^{n_b}$ at the same count. Constraints from Def.~\ref{def:hprobz} activate only at inference (\S\ref{sec:refinement}); non-diagonal $G_d, G_s$ left to future work. We use a uniform mode prior $\pi_k{=}2^{-n_b}$; learnable $\pi_k(x){>}0$ preserves identifiability/contraction with statistically tied accuracy (14-seed nuScenes, App.~\ref{app:learn_pi}).

\paragraph{The Gaussian approximation and its failure.} The na\"ive aggregate-covariance loss $\mathcal{L}_{\text{Gauss}} = -\log \sum_k 2^{-n_b} \N(y; \mu_k, G_sG_s^\top + G_dG_d^\top/3)$ with $\mu_k = c + G_b\beta^{(k)}$ is non-identifiable:

\begin{proposition}[Gaussian Collapse]
\label{prop:gaussian_collapse}
Under $\mathcal{L}_{\text{Gauss}}$, $G_d$ and $G_s$ are interchangeable: for any $(G_d, G_s)$ achieving loss $\ell$, there exists $G_s'$ with $G_s'(G_s')^\top = G_sG_s^\top + G_dG_d^\top/3$ such that $(0, G_s')$ achieves the same $\ell$, and vice versa. Gradient-based optimization therefore cannot distinguish bounded from stochastic uncertainty.
\end{proposition}
\begin{proof}
$\mathcal{L}_{\text{Gauss}}$ depends on $(G_d, G_s)$ only through the aggregate covariance $\Sigma_{\text{approx}} = G_sG_s^\top + G_dG_d^\top/3$. Choosing $G_s'$ with $G_s'(G_s')^\top = \Sigma_{\text{approx}}$ (e.g., the Cholesky factor) and $G_d=0$ leaves $\Sigma_{\text{approx}}$ and hence $\mathcal{L}_{\text{Gauss}}$ unchanged.
\end{proof}

The resolution is to train with the exact within-mode likelihood. Conditioned on mode $k$, the residual $r_j = y_j - \mu_{k,j}$ in each dimension follows the convolution $\text{Uniform}(-a_j, a_j) \ast \N(0, \sigma_j^2)$, whose density admits the closed form
\begin{equation}
\label{eq:exact_pdf}
    f(r; a, \sigma) = \frac{1}{2a}\Big[\Phi\!\Big(\frac{r+a}{\sigma}\Big) - \Phi\!\Big(\frac{r-a}{\sigma}\Big)\Big],
\end{equation}
where $\Phi$ is the standard normal CDF (numerical clamp and Gaussian-limit fallback in App.~\ref{app:impl_details}). The exact mixture NLL is $\mathcal{L}_{\text{exact}} = -\log \sum_k 2^{-n_b} \prod_j f(y_j - \mu_{k,j}; a_j, \sigma_j)$. Its CF has zeros at integer multiples of $\pi/a$, a compact-support spectral signature absent from any Gaussian; this resolves the collapse in Prop.~\ref{prop:gaussian_collapse} and enables Theorem~\ref{thm:identifiability}.

\paragraph{The structure--density trade-off.} Introducing bounded generators with $a > 0$ imposes compact support, which necessarily reduces entropy relative to a Gaussian with the same variance. However, this cost is bounded:

\begin{proposition}[Structure--density trade-off]
\label{prop:tradeoff}
Let $f_\kappa$ denote the 1-D HProbZ density with half-width $a$ and noise $\sigma$, and let $g_\kappa = \N(0, \sigma^2 + a^2/3)$ be the variance-matched Gaussian. Define the structure ratio $\kappa = a/\sigma$. Then:

\noindent\textup{(i)} \textbf{Bounded entropy cost.} The entropy gap $h(g_\kappa) - h(f_\kappa)$ is monotonically increasing in $\kappa$ and satisfies
\begin{equation}
\label{eq:entropy_gap}
    h(g_\kappa) - h(f_\kappa) \;\leq\; \tfrac{1}{2}\log\!\big(\pi e / 6\big) \;\approx\; 0.176 \text{ nats},
\end{equation}
for all $\kappa \geq 0$, with equality in the limit $\kappa \to \infty$.

\noindent\textup{(ii)} \textbf{Unbounded constraint benefit.} In the constrained setting of Theorem~\ref{thm:contraction}, the per-observation Fisher information for the shared bounded factor is $\rho \geq \kappa^2$ per dimension, yielding bounded-variance decay $O(1/(\kappa^2 k))$ after $k$ observations; the dimensionless derivation is in Appendix~\ref{app:tradeoff}.

\noindent\textup{(iii)} \textbf{Exchange rate.} The entropy cost is bounded by a universal constant independent of $\kappa$, while the constraint contraction rate grows without bound as $\kappa$ increases. In particular, for any $\kappa > 0$, the per-observation information gain $\rho$ can be made arbitrarily large at a per-dimension NLL cost of at most $0.176$ nats.
\end{proposition}

Proposition~\ref{prop:tradeoff} bounds the entropy gap between $f_\kappa$ and its variance-matched Gaussian by $0.176$ nats/dim; it does \emph{not} bound the empirical NLL gap, which also contains a data-dependent misspecification term. Direct computation on three benchmarks shows zero violations across $13{,}000+$ samples (App.~\ref{app:theorem_validation}).

\paragraph{Split-conformal calibration.}
\label{sec:conformal}
To obtain distribution-free coverage guarantees even under model misspecification, we define the HProbZ nonconformity score
\begin{equation}
\label{eq:nonconformity}
    s(x,y) = \min_{\beta \in \{-1,1\}^{n_b}} \max_{j \in [d]} \frac{|y_j - c_j(x) - (G_b(x)\beta)_j| - \|G_d(x)[j,:]\|_1}{\|G_s(x)[j,:]\|_2},
\end{equation}
which selects the closest binary mode and measures the worst-dimension standardized residual after exhausting the bounded generator's reach.

\begin{theorem}[Distribution-free HProbZ coverage]
\label{thm:conformal}
Assume $\|G_s(x)[j,:]\|_2 > 0$ for every $x$ and $j$ (in practice ensured by softplus). Let $q_\alpha$ be the $\lceil(1{-}\alpha)(n{+}1)\rceil/n$ quantile of calibration scores $\{s(x_i,y_i)\}_{i=1}^n$. The prediction set $\hat S_\alpha(x) = \bigcup_{\beta}\{z: |z_j - c_j-(G_b\beta)_j|\leq \|G_d[j,:]\|_1 + q_\alpha \|G_s[j,:]\|_2,\, \forall j\}$ satisfies $\mathbb P(y_{\text{test}} \in \hat S_\alpha) \geq 1{-}\alpha$ under only exchangeability.
\end{theorem}

$\hat{S}_\alpha$ is a union of up to $2^{n_b}$ axis-aligned boxes rather than a single convex region; the per-dimension edge is $2$--$3\times$ shorter than Standard CP and $\sim 1.5\times$ shorter than MDN-CP. Empirical volumes, coverage, and OOD robustness are reported in \S\ref{sec:conformal_body}; a parallel distributional coverage bound $P(y\in S_\gamma) \geq 1 - 2d e^{-\gamma^2/2}$ when $y\mid x$ follows HProbZ is in App.~\ref{app:proofs}.

\section{Experiments}
\label{sec:experiments}

We test four questions: (i) does the exact likelihood recover the modal/bounded/stochastic decomposition; (ii) is HProbZ competitive on accuracy benchmarks; (iii) does structured uncertainty improve downstream decisions; (iv) does the algebraic constraint realise the $O(1/k)$ contraction of Theorem~\ref{thm:contraction}. HProbZ and MDN share encoder, schedule, and sampling budget throughout (App.~\ref{app:protocols},~\ref{app:impl_details}); metrics are minADE/minFDE for accuracy, CRPS and conformal coverage for calibration, per-mode risk and closed-loop collision for decisions; raw NLL in App.~\ref{app:nll_explanation}.

\subsection{Controlled Decomposition}
\label{sec:controlled}

We construct a trajectory task with three independently controllable uncertainty sources --- branches $K$, drift $\delta_{\mathrm{drift}}$, noise $\sigma_{\mathrm{noise}}$ --- and train HProbZ under both exact and Gaussian-approximate likelihoods, measuring whether each generator selectively tracks its matching source. The full sweep figure and per-value tables (Fig.~\ref{fig:decomposition}, App.~\ref{app:controlled}) match the theoretical predictions: under exact likelihood, $\|G_b\|$ grows $\sim 9\times$ across the mode sweep while drift is routed to $G_d$ ($\Delta\|G_d\|{=}2.90 \gg \Delta\|G_s\|{=}2.27$); the Gaussian approximation reverses the drift dominance ($\Delta\|G_s\|{=}2.70$ wins), a direct witness of Prop.~\ref{prop:gaussian_collapse}. Theorem~\ref{thm:non-gmm}'s CF fingerprint is in Fig.~\ref{fig:nonrep}; Theorem~\ref{thm:identifiability}'s multi-seed parameter recovery (within $0.008$ across 10 seeds) is in App.~\ref{app:theorem_validation}.

\paragraph{Decomposition on real driving data.} The same modal/bounded/stochastic split is recovered without supervision on nuScenes vehicle trajectories. Stratifying the trained $d{=}256$ model by past-observed vehicle speed (Fig.~\ref{fig:nuscenes_decomp}), $\|G_b\|$ scales by $\sim 370\times$ from stationary to high-speed agents (modal ambiguity grows with speed), $\|G_s\|$ scales by $\sim 94\times$ (stochastic noise), and $\|G_d\|$ stays within $5.4\times$ (systematic drift is less speed-dependent). The three generators therefore track interpretable, qualitatively distinct sources of uncertainty on real trajectories, not only on the controlled task above; setup details in App.~\ref{app:nuscenes_speed}.

\begin{figure}[t]
\centering
\includegraphics[width=0.32\linewidth]{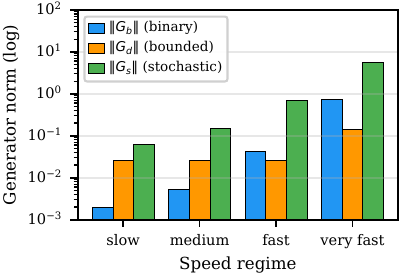}\hfill
\includegraphics[width=0.32\linewidth]{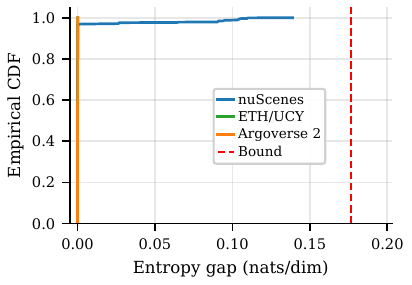}\hfill
\includegraphics[width=0.32\linewidth]{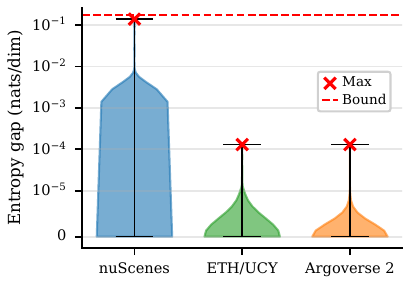}
\caption{Structural-property validation on trained HProbZ models. \textbf{Left:} nuScenes speed-stratified generator norms (log-$y$) --- $\|G_b\|$ scales $\sim 370\times$ stationary$\to$high-speed, $\|G_s\|$ scales $\sim 94\times$, $\|G_d\|$ stays within $5.4\times$, so the three generators track interpretable, qualitatively distinct uncertainty sources (Theorem~\ref{thm:identifiability}; setup App.~\ref{app:nuscenes_speed}). \textbf{Middle:} entropy-gap CDF across $13{,}000{+}$ samples on three benchmarks; all curves lie strictly below the theoretical bound $\tfrac{1}{2}\log(\pi e/6)\!\approx\!0.1765$ nats/dim (red dashed). \textbf{Right:} per-dataset distribution of the same gap (symmetric log scale). Zero violations on any benchmark for Proposition~\ref{prop:tradeoff} (App.~\ref{app:tradeoff_validation}).}
\label{fig:nuscenes_decomp}
\label{fig:prop31-verification}
\end{figure}

\subsection{Real-World Pedestrian Trajectories}
\label{sec:eth_ucy}

ETH/UCY leave-one-out (8-step past, 12-step future). Base HProbZ: single-agent transformer encoder, exact likelihood, $n_b{=}3, n_d{=}1, n_q{=}1$. Social variant adds 8-neighbour attention. Baselines: 8-component MDN and CVAE with identical input features.

\begin{table}[t]
\centering
\caption{ETH/UCY: minADE@20 / minFDE@20 (m), 5-scene leave-one-out averages, no maps. Base = $d{=}64$, 3 seeds (CVAE wins this row); Social = $d{=}128$ + 8-neighbour attention, 5 seeds. Social-CVAE not included due to compute. Per-scene table in App.~\ref{app:eth_ucy_perscene}.}
\label{tab:eth_ucy_full}
\renewcommand{\arraystretch}{0.95}
\resizebox{\linewidth}{!}{%
\begin{tabular}{lccc}
\toprule
\textbf{Encoder} & \textbf{HProbZ} & \textbf{MDN-K=8} & \textbf{CVAE} \\
\midrule
Base $d{=}64$, 3 seeds   & $0.265{\pm}0.003$\,/\,$0.470{\pm}0.008$ & $0.380{\pm}0.003$\,/\,$0.427{\pm}0.008$ & $0.240{\pm}0.003$\,/\,$0.414{\pm}0.006$ \\
Social $d{=}128$, 5 seeds & $\mathbf{0.197{\pm}0.002}$\,/\,$\mathbf{0.332{\pm}0.006}$ & $0.324{\pm}0.001$\,/\,$0.419{\pm}0.005$ & --- \\
\bottomrule
\end{tabular}}
\end{table}


Social HProbZ reaches $0.197/0.332$ minADE/minFDE, beating same-encoder MDN-$K{=}8$ by $\mathbf{39.3\%}$ ADE / $\mathbf{20.7\%}$ FDE (Tab.~\ref{tab:eth_ucy_full}, 5 seeds), matching LED ($0.18$) and EqMotion ($0.19$) on minADE and trailing them by $\sim 0.06$\,m on minFDE (full vs-published table in App.~\ref{app:published_comparison}); goal-conditioned methods like Y-Net/LED sharpen endpoints by conditioning on predicted endpoints but provide no decomposition, conformal coverage, or refinement.

At base $d{=}64$ HProbZ beats MDN-$K{=}8$ on ADE but trails CVAE; the social variant recovers and surpasses CVAE. The $\sim 21\%$ ADE advantage over MDN-$K{=}20$ is stable across encoder sizes 110K--2.4M (App.~\ref{app:encoder_scaling}), ruling out a capacity confound on minADE; \emph{on minFDE MDN-$K{=}20$ wins at all three capacities}, so the social-row FDE advantage comes from the social-attention encoder, not the head. Raw NLL exceeds MDN's, a compact-support misspecification penalty (App.~\ref{app:nll_explanation}); we lead with ADE/FDE and downstream metrics. A controlled cross-method comparison (App.~\ref{app:eth_ucy_diffusion}) gives HProbZ FDE@20 $0.39$\,m vs MDN-K2/K4 $0.41$--$0.42$ and a lightly-tuned Diffusion $0.43$/$2.24$ (DDIM-$50$/-$5$) at $\sim 530\times$ slower inference (diffusion baseline not swept over schedule/depth/target/CFG).

\subsection{nuScenes Vehicle Trajectory Benchmark}
\label{sec:nuscenes}

nuScenes~\citep{caesar2020nuscenes}: 40k test windows at 2\,Hz. Same $d{=}128$ encoder; HProbZ ($n_b{=}1$, 278k params) vs MDN-K2 (281k), MDN-K4 under identical training. Structured head runtime overhead $1.4\%$ on H200 at batch 256.

\begin{table}[t]
\centering
\caption{nuScenes minADE (m): published map-aware baselines, trajectory-only $d{=}128$ controlled sweep (3 seeds), and $d{=}256$ matched-encoder. MDN-K2$^\dagger$: best-tuned (App.~\ref{app:mdn_sweep}). \emph{Seed asymmetry in the $d{=}256$ block}: HProbZ is mean$\pm$std over 14 seeds (App.~\ref{app:symbreak}); MDN-K2 row is single-seed (no $\pm$std), so the $46\%$ gap is HProbZ's 14-seed mean vs one MDN-K2 seed.}
\label{tab:nuscenes}
\footnotesize
\renewcommand{\arraystretch}{0.95}
\begin{tabular*}{\linewidth}{@{\extracolsep{\fill}}lcccc}
\toprule
\textbf{Model} & \textbf{Map} & \textbf{ADE@5} & \textbf{ADE@10} & \textbf{ADE@20} \\
\midrule
\multicolumn{5}{l}{Published methods (map + interaction features)} \\
Trajectron++~\citep{salzmann2020trajectron++} & \checkmark & 1.88 & 1.51 & --- \\
AgentFormer~\citep{yuan2021agentformer}       & \checkmark & 1.86 & 1.45 & --- \\
MID~\citep{gu2022stochastic}                 & \checkmark & --- & 1.44 & --- \\
\midrule
\multicolumn{5}{l}{Trajectory-only, $d{=}128$ encoder (3 seeds)} \\
MDN-K2$^\dagger$  & \texttimes & $2.03{\pm}0.08$ & $1.84{\pm}0.07$ & $1.69{\pm}0.06$ \\
HProbZ            & \texttimes & $1.57{\pm}0.02$ & $1.46{\pm}0.02$ & $1.37{\pm}0.03$ \\
\midrule
\multicolumn{5}{l}{Trajectory-only, $d{=}256$ encoder, matched for MDN and HProbZ} \\
MDN-K2 (2.14\,M) & \texttimes & 2.66 & 2.41 & 2.20 \\
\textbf{HProbZ (2.13\,M)} & \texttimes & $\mathbf{1.41 {\pm} 0.06}$ & $\mathbf{1.29 {\pm} 0.07}$ & $\mathbf{1.20 {\pm} 0.07}$ \\
\bottomrule
\end{tabular*}
\end{table}

At $d{=}128$ HProbZ beats best-tuned MDN-K2 by $23\%/19\%$ at $K{=}5/K{=}20$ and reaches $1.46{\pm}0.02$ minADE@10 (3 seeds), matching MID's $1.44$ within rounding. Scaling MDN to $K{=}8$ regresses to $3.05{\pm}0.23$ (mode collapse, App.~\ref{app:mdn_k8}). Scaling the shared encoder to $d{=}256$ ($2.13$\,M params, sym-broken $G_b$ init, App.~\ref{app:symbreak}), HProbZ reaches $\mathbf{1.29{\pm}0.07}$ minADE@10 over 14 seeds, approaching MID's published $1.44$ without maps. The sym-break is load-bearing: with random init the same architecture trains to $1.48{\pm}0.10$ over 5 seeds (above MID; full ablation Tab.~\ref{tab:symbreak_sweep}). Single-seed same-encoder MDN-K2 reaches $2.41$ at $d{=}256$ (a $46\%$ gap to HProbZ's 14-seed mean; multi-seed MDN d=256 reproduction left to future work), and same-encoder DDPM reaches $2.99$ at $5.2\times$ params and $500\times$ slower (App.~\ref{app:diffusion_comparison}).

\paragraph{Calibration: CRPS.} On $n{=}5$ paired matched-architecture seeds, HProbZ and MDN-K2 are statistically tied on mean CRPS; per-seed spread analysis is in App.~\ref{app:crps_full}.

\subsection{Argoverse 2 Motion Forecasting}
\label{sec:av2}

Argoverse~2~\citep{wilson2021argoverse2}: 200k scenarios, 6\,s horizon. Same $d{=}128$ encoder, 2\,Hz subsampling, 618k training windows, identical conditions across rows.

\begin{table}[t]
\centering
\caption{Argoverse~2 minADE (m), 3 seeds, same-encoder trajectory-only models ($d{=}128$).}
\label{tab:av2}
\small
\renewcommand{\arraystretch}{0.95}
\begin{tabular*}{\linewidth}{@{\extracolsep{\fill}}lccccc}
\toprule
\textbf{Model} & \textbf{Params} & \textbf{ADE@1} & \textbf{ADE@5} & \textbf{ADE@10} & \textbf{ADE@20} \\
\midrule
MDN-K2   & 298k & $5.72\!\pm\!0.10$ & $4.47\!\pm\!0.07$ & $4.10\!\pm\!0.06$ & $3.81\!\pm\!0.05$ \\
\textbf{HProbZ} & 295k & $\mathbf{5.32\!\pm\!0.01}$ & $\mathbf{4.14\!\pm\!0.01}$ & $\mathbf{3.80\!\pm\!0.02}$ & $\mathbf{3.54\!\pm\!0.02}$ \\
\bottomrule
\end{tabular*}
\end{table}

HProbZ wins minADE at every $K$, beating MDN-K2 by $\sim 7\%$ across 3 seeds (Tab.~\ref{tab:av2}); seed-level std is $4{-}5\times$ tighter ($\sigma_{\mathrm{HProbZ}}\!\le\!0.02$\,m vs.\ MDN's $0.05{-}0.10$, directional at $n{=}3$). The Prop.~\ref{prop:tradeoff} entropy-gap bound is confirmed across three benchmarks (0 violations, $13{,}000{+}$ samples; Fig.~\ref{fig:prop31-verification}, App.~\ref{app:tradeoff_validation}).

\paragraph{Downstream risk-aware planning.}\label{sec:planning}
On a pedestrian-crossing scenario (App.~\ref{app:planning_details}), the structured planner that enumerates $2^{n_b}$ binary modes and evaluates per-mode risk analytically recovers oracle cost $0$ on Modal-Safe scenes vs.\ the Gaussian-collapsed planner's $1.5$, with the advantage emerging at mode separation $\Delta y \gtrsim 4$ (App.~\ref{app:planning_sweep}). Across $5{,}000$ nuScenes scenarios $\times 20$ seeds, HProbZ's analytic per-mode probability is deterministic (zero variance, zero flips); MDN's sample estimator decays from $33\%$ flips at $S{=}2$ to $0.9\%$ at $S{=}100$ but still flips $2.7\%$ of \emph{ambiguous} cases (full sample-budget sweep in App.~\ref{app:risk_consistency_full}).

\paragraph{Closed-loop collision avoidance.} On a $6$\,s static-obstacle stress test over nuScenes ($N{=}39{,}829$, App.~\ref{app:closed_loop_collision}), HProbZ produces $\mathbf{30\%}$ fewer trajectory-level joint-event collisions than same-encoder MDN-K2 at near-matched $10\%$ FAR (Tab.~\ref{tab:closed_loop}; FARs $9.5\%$ vs $10.6\%$). \emph{The per-step variant reverses} (HProbZ $8.01\%$ vs MDN $6.32\%$): the shared $(\beta,\alpha)$ pays off through correlated per-step events that MDN's per-step independent Gaussians underestimate at the joint-tail; without that correlation MDN's per-step flexibility dominates. A planner conditions on the joint event.

\begin{table}[t]
\centering
\caption{Closed-loop collision avoidance on nuScenes ($N{=}39{,}829$, $6$\,s, $r{=}2$\,m, 3 seeds). HProbZ wins trajectory-level by $30\%$ at near-matched $10\%$ FAR (HProbZ $9.5\%$, MDN $10.6\%$); MDN-K2 wins per-step. Bold = column winner.}
\label{tab:closed_loop}
\small
\renewcommand{\arraystretch}{0.95}
\begin{tabular*}{\linewidth}{@{\extracolsep{\fill}}lcccc}
\toprule
 & \multicolumn{2}{c}{\textbf{Trajectory (6\,s window)}} & \multicolumn{2}{c}{\textbf{Pointwise (single step)}} \\
\cmidrule(lr){2-3}\cmidrule(lr){4-5}
\textbf{Method} & AUC $\uparrow$ & Coll.\ @ $10\%$ FAR $\downarrow$ & AUC $\uparrow$ & Coll.\ @ $10\%$ FAR $\downarrow$ \\
\midrule
\textbf{HProbZ (ours)} & $\mathbf{0.925\!\pm\!0.001}$ & $\mathbf{6.73\!\pm\!0.08\%}$ & $0.906\!\pm\!0.001$ & $8.01\!\pm\!0.04\%$ \\
MDN-K2        & $0.869\!\pm\!0.001$ & $9.70\!\pm\!0.02\%$ & $\mathbf{0.930\!\pm\!0.000}$ & $\mathbf{6.32\!\pm\!0.08\%}$ \\
\bottomrule
\end{tabular*}
\end{table}

\paragraph{Constrained sequential refinement.}\label{sec:refinement}
With $\beta,\alpha$ shared across steps via $z_t = c_t + G_{b,t}\beta + G_{d,t}\alpha + G_{s,t}\nu_t$, each observation adds a constraint row that simultaneously reweights the $2^{n_b}$ mode likelihoods and tightens $\alpha$'s feasible set; the tightening propagates to all unobserved steps.

\begin{theorem}[Contraction of constrained HProbZ, informal]
\label{thm:contraction}
Under the shared-generator model with the variance-matched Gaussian relaxation $\alpha\sim\N(0,\tfrac{1}{3}I)$, conditioning on $k$ observations \textup{(i)} identifies the true mode $\beta^*$ with posterior probability $\to 1$; \textup{(ii)} contracts the bounded predictive variance at any unobserved step $s$ as $\mathrm{tr}(\mathrm{Cov}(G_{d,s}\alpha \mid \mathrm{obs}, \beta^*)) \leq \|G_{d,s}\|_F^2 / (\tfrac{1}{3} + k\rho_{\min}) = O(1/k)$, while the stochastic term $\mathrm{tr}(G_{s,s}G_{s,s}^\top)$ is irreducible; \textup{(iii)} an MDN updates only its mixing weights, leaving within-component covariances $\Theta(1)$ in $k$. Formal statement, proof, and discussion of the relaxation gap in Appendix~\ref{app:contraction}.
\end{theorem}

The relaxation is variance-matched to the trained $\mathrm{Uniform}([-1,1])$ prior; a 1-D numerical check bounds the relative gap between relaxed and true posterior variances by $\leq 9.03\%$ on the $(\kappa,k)$ grid we test (App.~\ref{app:contraction}). Observable contraction is controlled by $\kappa = \|G_d\|/\sigma$: ETH/UCY learns $\kappa\!\approx\!2.1$ (ceiling $\sim$$32\%$); on nuScenes, sweeping the $G_d$-bias over $\kappa\in\{1.5, 2.0, 2.8, 3.7\}$ gives spread reductions $\{12.3{\pm}3.6,\, 17.7{\pm}4.9,\, 25.8{\pm}4.2,\, 33.2{\pm}5.9\}\%$ (5 seeds; per-seed values and 3-seed reference in Tab.~\ref{tab:nuscenes_constraint}). We instantiate on ETH/UCY with $n_b{=}1, n_d{=}1$ shared across $T{=}12$ steps, revealing positions $0,\ldots,8$.

\begin{table}[t]
\centering
\caption{Sequential refinement on ETH/UCY (same-encoder, 5-scene aggregate, 3 seeds). HProbZ contracts spread $23.6\%$ and improves FDE $16.5\%$; MDN's FDE degrades.}
\label{tab:refinement}
\small
\renewcommand{\arraystretch}{0.95}
\begin{tabular*}{\linewidth}{@{\extracolsep{\fill}}ccccc}
\toprule
\textbf{Revealed $k$} & \textbf{HProbZ FDE} & \textbf{MDN FDE} & \textbf{HProbZ std} & \textbf{MDN std} \\
\midrule
0 & $\mathbf{1.071\!\pm\!0.006}$ & $1.199\!\pm\!0.076$ & $\mathbf{0.852\!\pm\!0.006}$ & $1.014\!\pm\!0.035$ \\
8 & $\mathbf{0.894\!\pm\!0.019}$ & $1.314\!\pm\!0.107$ & $\mathbf{0.651\!\pm\!0.017}$ & $0.977\!\pm\!0.081$ \\
\midrule
$\Delta$ (0$\to$8) & $\mathbf{-16.5\%}$ & $+9.6\%$ & $\mathbf{-23.6\%}$ & $-3.6\%$ \\
\bottomrule
\end{tabular*}
\end{table}

Table~\ref{tab:refinement} shows simultaneous spread-and-mean refinement: HProbZ's std drops $23.6\%$ and FDE improves $16.5\%$, while MDN re-weights toward a mis-located mode (FDE $+9.6\%$, std $-3.6\%$) --- exactly Theorem~\ref{thm:contraction}(iii).

\subsection{Conformal Coverage with Multi-Modal Sets}
\label{sec:conformal_body}

Split conformal at $\alpha{=}0.1$ (Theorem~\ref{thm:conformal}), comparing Standard CP (single $L_\infty$ box), MDN-CP ($K$-component union), and HProbZ-CP ($2^{n_b}$ per-mode union via Eq.~\eqref{eq:nonconformity}); $D{=}24$. All three reach $\sim 90\%$ coverage; HProbZ-CP achieves it with substantially smaller volumes on motion-rich scenes (Tab.~\ref{tab:conformal_vol_ethucy}).

\begin{table}[t]
\centering
\caption{Conformal set volume comparison ($\alpha{=}0.1$, $D{=}24$, within-scene calibration). All methods reach $\sim 90\%$ coverage; HProbZ-CP is $5.5$/$3.9$ orders smaller than StdCP/MDN-CP on the ETH/UCY average and $16.3$/$13.9$ orders smaller on Argoverse~2. HOTEL is a counter-example: MDN-CP is $3.5$ OOM smaller than HProbZ-CP there (App.~\ref{app:conformal_volume} discusses the mechanism). Bold marks the smallest volume per row.}
\label{tab:conformal_vol_ethucy}
\footnotesize
\renewcommand{\arraystretch}{0.95}
\begin{tabular}{lccc@{\hspace{6pt}}ccc}
\toprule
& \multicolumn{3}{c}{\textbf{Empirical coverage}} & \multicolumn{3}{c}{\textbf{$\log_{10}$(volume)} $\downarrow$} \\
\cmidrule(lr){2-4}\cmidrule(lr){5-7}
\textbf{Scene} & HProbZ & StdCP & MDN & HProbZ & StdCP & MDN \\
\midrule
ETH   & 0.890 & 0.874 & 0.868 & \textbf{20.0} & 22.0 & 23.7 \\
HOTEL & 0.930 & 0.858 & 0.846 & 15.1 & 15.5 & \textbf{11.6} \\
UNIV  & 0.896 & 0.899 & 0.907 & \textbf{7.1}  & 13.6 & 11.3 \\
ZARA1 & 0.892 & 0.910 & 0.902 & \textbf{3.9}  & 13.1 & 10.1 \\
ZARA2 & 0.904 & 0.911 & 0.896 & \textbf{3.7}  & 13.2 & 12.8 \\
\midrule
\textbf{ETH/UCY Avg.} & 0.903 & 0.890 & 0.884 & \textbf{10.0} & 15.5 & 13.9 \\
\midrule
\textbf{Argoverse~2} & 0.901 & 0.900 & 0.901 & \textbf{21.5} & 37.8 & 35.4 \\
\bottomrule
\end{tabular}
\end{table}

Volume gaps compound from $2$--$3\times$ per-dimension edge ratios over $D{=}24$. The largest improvements are on open-walkway scenes (UNIV $6.5$, ZARA1 $9.2$, ZARA2 $9.5$ OOM vs StdCP) and Argoverse~2 ($16.3$ OOM); the gap shrinks on ETH ($2.0$) and reverses on HOTEL where MDN-CP wins by $3.5$ OOM --- HProbZ-CP wins the average and 4/5 scenes. OOD coverage holds: $90.9\%$ ETH/UCY leave-one-out average, $89.7\%/99.8\%$ IID/OOD on nuScenes speed split (App.~\ref{app:ood}); $\alpha$-sweep and protocol in App.~\ref{app:conformal_volume}.

\section{Related Work}
\label{sec:related}

\paragraph{Trajectory forecasting.} Recurrent~\citep{alahi2016social, gupta2018social}, CVAE~\citep{salzmann2020trajectron++, mangalam2020pecnet, mangalam2021ynet}, transformer~\citep{yuan2021agentformer, xu2022memonet, shi2022mtr, zhou2023qcnet, zhang2024demo, song2024realmotion}, language-modeling~\citep{seff2023motionlm}, diffusion~\citep{gu2022stochastic, mao2023leapfrog, jiang2023motiondiffuser}, flow~\citep{schoeller2021flomo}, equivariant~\citep{xu2023eqmotion}, and refinement~\citep{zhou2024smartrefine} forecasters emit unstructured samples or independent-Gaussian mixtures without algebraic constraints across steps; HProbZ's shared bounded factor adds such a constraint under a closed-form output.
\paragraph{Uncertainty decomposition and conformal prediction.} Epistemic--aleatoric decomposition~\citep{kendall2017uncertainties, lakshminarayanan2017simple} is orthogonal and coarser; Theorem~\ref{thm:identifiability} gives a likelihood-identifiable within-aleatoric one. Split conformal~\citep{vovk2005algorithmic, romano2019conformalized, barber2021predictive, angelopoulos2023conformal} and trajectory variants~\citep{lindemann2023safe, huang2025cuqds, zhou2024conformalized} return single convex regions; Eq.~\eqref{eq:nonconformity} returns a union of $2^{n_b}$ axis-aligned boxes from the trained model.
\paragraph{Zonotopes in machine learning.} Zonotopes serve as inputs/intermediate sets for NN verification~\citep{singh2018fast, zhang2018efficient, mirman2018differentiable, bird2023hybrid, kochdumper2023constrained} with the network fixed; HProbZ inverts the role, making the zonotope the trained \emph{output} via Eq.~\eqref{eq:exact_pdf}.

\section{Conclusion}
\label{sec:conclusion}

HProbZ is a structured output head that keeps modal, bounded, and stochastic uncertainty separately identifiable. It approaches map-aware MID's published minADE@10 on nuScenes ($1.29\!\pm\!0.07$, 14 seeds) and is on par with LED/EqMotion on ETH/UCY ($0.197$). Structural gains: observation-driven contraction up to $33{\pm}6\%$ (5 seeds), $30\%$ fewer trajectory-level closed-loop collisions (per-step reverses), and conformal sets $3.9$/$13.9$ orders smaller than MDN-CP on ETH/UCY/Argoverse~2.

\paragraph{Limitations.} \emph{Cross-dataset transfer is largely negative}: zero-shot nuScenes$\to$nuPlan AUC $0.519$ on the full agent mix (within $2$\,pp of chance, $14{,}653$ pairs); only the close-proximity subset preserves $+0.159$ AUC (App.~\ref{app:nuplan_closed_loop}). \emph{Closed-loop ordering is metric-dependent}: the $30\%$ headline is trajectory-level; per-step reverses (HProbZ $8.01\%$ vs MDN $6.32\%$). \emph{The SOTA-chase number depends on a method-specific init}: $1.29$ uses the symmetry-breaking $G_b$-bias (App.~\ref{app:symbreak}); without it, the same architecture trains to $1.48{\pm}0.10$, above MID's $1.44$. Theorem~\ref{thm:contraction} is proven under a Gaussian-relaxed prior with a $\leq 9.03\%$ numerical gap to the trained uniform prior. Raw NLL exceeds MDN's; the bare $d{=}64$ head trails CVAE on ADE before social attention, and at vanilla encoders MDN-K=20 wins minFDE@20 (Tab.~\ref{tab:encoder_scaling}). Other caveats (mode prior, $G_d/G_s$ diagonality, baseline tuning, deployment) in App.~\ref{app:broader_impact}.

\bibliographystyle{plainnat}
\bibliography{BibTex_2025}

\appendix

\noindent\textbf{Appendix outline.}
\textbf{A} Implementation Details and Reproducibility (datasets, hyperparameters, compute, per-experiment protocol).
\textbf{B} Full Proofs (Theorems 1--5, Propositions 1--3).
\textbf{C} Empirical Theorem Validation (per-value sweep tables; CF fingerprint; parameter recovery; entropy-gap bound; raw-NLL note).
\textbf{D} Additional Experiments grouped as: D.1 Ablations and Hyperparameter Sweeps; D.2 Calibration Extensions; D.3 Downstream Task Details; D.4 Baseline Comparisons (Extended); D.5 Per-Scene, Per-Modality, Per-Seed Analyses; D.6 Visualisations.
\textbf{E} Broader Impact and Ethical Considerations.

\section{Implementation Details and Reproducibility}
\label{app:impl_details}

\subsection{Datasets}
\label{app:datasets}

\paragraph{ETH/UCY pedestrian.} Five scenes (ETH, HOTEL, UNIV, ZARA1, ZARA2) recorded at $2.5$\,Hz, evaluated under the standard leave-one-out protocol: train on 4 scenes, test on the held-out scene. Trajectories are observed for $8$ frames ($3.2$\,s) and predicted for $12$ frames ($4.8$\,s). Coordinates are in metres, world frame. Source: original ETH/UCY recordings released for research use.

\paragraph{nuScenes vehicle.} 1{,}000 driving scenes from Boston and Singapore, 40\,k test windows of vehicle trajectories at $2$\,Hz~\citep{caesar2020nuscenes}. Past horizon $1$\,s, future horizon $6$\,s ($T_f{=}12$). Ego-centric coordinate frame. License: nuScenes non-commercial research use.

\paragraph{Argoverse~2 motion forecasting.} 200\,k training scenarios, 25\,k validation, 24\,k test (97\,k test windows) at $10$\,Hz~\citep{wilson2021argoverse2}; we subsample to $2$\,Hz for matched comparison. License: CC-BY-NC-SA 4.0.

\paragraph{nuPlan (zero-shot evaluation only).} The nuPlan mini split (Boston/Pittsburgh/Singapore/Las Vegas logs) is used only for cross-dataset collision-warning evaluation in \S\ref{app:nuplan_closed_loop}; no training. License: nuPlan non-commercial research use.

\subsection{HProbZ Architecture and Training}
\label{app:hprobz_hyperparams}

\begin{table}[h]
\centering
\small
\caption{HProbZ master hyperparameter table. ETH/UCY entry covers both single-agent ($d{=}64$) and Social ($d{=}128$, 8-neighbour attention) variants. Optimiser is AdamW with cosine LR decay; warmup is linear over the first $10$ epochs. ``Sym-break $b_0$'' is the positive bias added to the $G_b$ output channels at init (Appendix~\ref{app:symbreak}).}
\label{tab:hprobz_hyperparams}
\resizebox{\linewidth}{!}{\begin{tabular}{lcccccccc}
\toprule
\textbf{Setting} & $d$ & \textbf{Layers} & $(n_b, n_d, n_q)$ & \textbf{Batch} & \textbf{LR} & \textbf{Epochs} & \textbf{Sym-break $b_0$} & \textbf{Seeds} \\
\midrule
ETH/UCY single   & 64  & 2 & $(3, 1, 1)$ & 1024 & $3{\times}10^{-4}$ & 600  & 0.05 & 3  \\
ETH/UCY Social   & 128 & 3 & $(3, 1, 1)$ & 2048 & $3{\times}10^{-4}$ & 600 & 0.05 & 5  \\
nuScenes ($d{=}128$) & 128 & 2 & $(1, 1, 1)$ & 1024 & $3{\times}10^{-4}$ & 100 & 0   & 3  \\
nuScenes ($d{=}256$) & 256 & 4 & $(1, 1, 1)$ & 1024 & $2{\times}10^{-4}$ & 150 & 1.0 & 14 \\
Argoverse~2      & 128 & 2 & $(1, 1, 1)$ & 1024 & $3{\times}10^{-4}$ & 100 & 0   & 3  \\
\bottomrule
\end{tabular}}
\end{table}

\subsection{Baseline Architectures and Hyperparameters}
\label{app:baseline_hyperparams}

\textbf{MDN.} \texttt{Linear($d, d$)\,$\to$\,GELU\,$\to$\,Linear($d, T_f \cdot K \cdot 5$)} producing per-component means, variances, and mixing logits. Trained with the same encoder, optimiser, batch, and schedule as the matched HProbZ row. Full hyperparameter sweep in Appendix~\ref{app:mdn_sweep}; the pathological MDN-K=8 collapse sweep is in Appendix~\ref{app:mdn_k8}.

\textbf{CVAE.} 32-dim latent with shared-encoder amortized posterior; KL warmup over $20$ epochs. Used in Table~\ref{tab:eth_ucy_full} and as a baseline in the controlled experiments of Appendix~\ref{app:additional_baselines_varying}.

\textbf{Diffusion (DDPM).} Temporal transformer denoiser with adaLN-style timestep conditioning, $4$ blocks, cosine $T{=}200$ schedule, DDIM~$50$-step sampling, $\eta{=}0.5$; details and step-count sweep in Appendix~\ref{app:diffusion_comparison}.

\subsection{Per-Experiment Protocol Index}
\label{app:protocols}

Every sub-experiment in \S\ref{sec:experiments} is a matched-encoder HProbZ vs.\ MDN comparison; Table~\ref{tab:protocols} lists the controlled variables in each.

\begin{table}[h]
\centering
\small
\caption{Summary of experimental protocols. Within each row, HProbZ and the MDN baseline share the same encoder architecture, training schedule, and sampling budget.}
\label{tab:protocols}
\resizebox{\linewidth}{!}{\begin{tabular}{llll}
\toprule
\textbf{Experiment} & \textbf{Dataset} & \textbf{Encoder} & \textbf{HProbZ vs baseline} \\
\midrule
Table~\ref{tab:eth_ucy_full} (\S\ref{sec:eth_ucy}) & ETH/UCY (LOO, 5 scenes) & $d{=}64$ HProbZ/MDN/CVAE; $d{=}128$ Social HProbZ & Social $n_b{=}3$ vs MDN-K8 \\
Table~\ref{tab:nuscenes}, $d{=}128$ row (\S\ref{sec:nuscenes}) & nuScenes (40k test) & $d{=}128$, 2 layers & $n_b{=}1$ vs MDN-K2 \\
Table~\ref{tab:nuscenes}, $d{=}256$ row & nuScenes (40k test) & $d{=}256$, 4 layers & $n_b{=}1$ (14 seeds) vs MDN-K2 (1 seed; both 2.14\,M) \\
Table~\ref{tab:av2} (\S\ref{sec:av2}) & Argoverse~2 (97k test, 3 seeds) & $d{=}128$, 2 layers, batch=8192, 300 ep & $n_b{=}1$ vs MDN-K2 \\
CRPS (\S\ref{sec:nuscenes}) & nuScenes (full 40k, $K{=}100$, 5 paired seeds) & $d{=}128$, 2 layers & $n_b{=}1$ vs MDN-K2 (matched arch) \\
Closed-loop (\S\ref{sec:planning}) & nuScenes (full 40k, 500 samples/inst.) & $d{=}128$, 2 layers & $n_b{=}1$ vs MDN-K2 \\
Table~\ref{tab:refinement} (\S\ref{sec:refinement}) & ETH/UCY (LOO, reveal $k{\in}\{0,\dots,8\}$) & $d{=}64$, 2 layers & $n_b{=}1$ vs MDN-K8 \\
\bottomrule
\end{tabular}}
\end{table}

\subsection{Compute Resources}
\label{app:compute}

All training experiments use single H100/H200-class GPUs; small-scale ablations and runtime measurements use a consumer GPU (RTX-class). Single $d{=}256$ HProbZ on nuScenes: $\sim 18$ minutes per seed end-to-end (150 epochs, batch 1024, 162k training windows); the 14-seed sweep used $\approx 252$ GPU-minutes ($\approx 4.2$ GPU-hours). Argoverse~2 and ETH/UCY runs are smaller; the 600-epoch Social re-run used a 4-GPU cluster for the 5-seed $\times$ 5-scene grid. Inference benchmarks (throughput, refinement runtime) are measured on a single consumer GPU. Code, per-seed checkpoints, and per-experiment compute logs are released to avoid repeated retraining.

\subsection{Algorithmic Specification}
\label{app:algorithms}

We provide pseudocode for the two HProbZ-specific procedures. Algorithm~\ref{alg:training} is the per-batch training step (gradient through the closed-form convolution likelihood of Eq.~\eqref{eq:exact_pdf}); Algorithm~\ref{alg:refinement} is the constrained inference procedure that activates the algebraic constraint machinery of Definition~\ref{def:hprobz} when $k$ ground-truth positions become available. \emph{Refinement is performed only at inference; training (Algorithm~\ref{alg:training}) sees no constraint rows ($p{=}0$).} The Gaussian relaxation of $\alpha$ in Algorithm~\ref{alg:refinement} is therefore an inference-time approximation to the conditional posterior over the bounded latent, used solely to obtain a closed-form posterior covariance and not present in the training objective.

\begin{algorithm}[h]
\caption{HProbZ training step (one batch).}
\label{alg:training}
\begin{algorithmic}[1]
\Require batch $\{(x_i, y_i)\}_{i=1}^B$, encoder $\mathrm{Enc}_\phi$, 2-layer MLP head $\mathrm{Head}_W(h) = W_2\,\mathrm{GELU}(W_1 h)$ producing $(c, G_b, \log a, \log\sigma)$
\Require initial bias on head output: $G_b$ entries $= b_0$ (symmetry-break, benchmark-specific; $b_0{=}0.05$ for ETH/UCY, $b_0{=}1.0$ for nuScenes $d{=}256$, see Tab.~\ref{tab:hprobz_hyperparams} and App.~\ref{app:symbreak}), $\log a_0 = 0.7$, $\log\sigma_0 = 1.1$; weight init $\mathcal{N}(0, 0.01^2)$
\Require loss flag $\mathcal{L} \in \{\text{exact}, \text{Gauss}\}$, gradient-clip threshold $\tau = 1.0$
\For{$i = 1, \ldots, B$}
    \State $h_i \gets \mathrm{Enc}_\phi(x_i)$
    \State $(c_i,\, G_{b,i},\, \log a_i,\, \log\sigma_i) \gets \mathrm{Head}_W(h_i)$
    \State $a_i \gets \exp(\mathrm{clamp}(\log a_i,\, -4,\, 4));\ \sigma_i \gets \exp(\mathrm{clamp}(\log\sigma_i,\, -4,\, 4))$
    \For{each binary $\beta \in \{-1,1\}^{n_b}$}
        \State $\mu_{i,\beta} \gets c_i + G_{b,i}\beta$
    \EndFor
    \If{$\mathcal{L} = \text{exact}$}
        \State $\ell_i \gets -\log\sum_\beta 2^{-n_b}\prod_j f(y_{i,j} - \mu_{i,\beta,j};\, a_{i,j}, \sigma_{i,j})$ \Comment{$f$ from Eq.~\eqref{eq:exact_pdf}}
    \Else
        \State $\ell_i \gets -\log\sum_\beta 2^{-n_b}\,\N(y_i;\, \mu_{i,\beta},\, \sigma_i^2 + a_i^2/3)$ \Comment{Gaussian surrogate}
    \EndIf
\EndFor
\State $g \gets \nabla_{\phi, W}\, \tfrac{1}{B}\sum_i \ell_i$
\State Clip: $g \gets g \cdot \min\!\big(1,\ \tau / \|g\|_2\big)$
\State Update $\phi, W$ by AdamW($g$).
\end{algorithmic}
\end{algorithm}

\begin{algorithm}[h]
\caption{HProbZ constrained inference (refinement after $k$ revealed positions).}
\label{alg:refinement}
\begin{algorithmic}[1]
\Require trained head outputs at all steps $\{(c_t, G_{b,t}, G_{d,t}, G_{s,t})\}_{t=1}^T$, observed positions $\{z_{t_i}\}_{i=1}^k$
\Require Gaussian relaxation $\alpha \sim \N(0, \tfrac{1}{3}I_{n_d})$
\For{each binary $\beta \in \{-1,1\}^{n_b}$}
    \State $r_{t_i} \gets z_{t_i} - c_{t_i} - G_{b,t_i}\beta$ for $i=1,\ldots,k$
    \State $\Sigma_{s,t_i} \gets G_{s,t_i} G_{s,t_i}^\top$
    \State $I_k \gets \sum_{i=1}^k G_{d,t_i}^\top \Sigma_{s,t_i}^{-1} G_{d,t_i}$ \Comment{Fisher information}
    \State $\Sigma_{\mathrm{post}}(\beta) \gets (3 I_{n_d} + I_k)^{-1}$
    \State $\mu_{\mathrm{post}}(\beta) \gets \Sigma_{\mathrm{post}}(\beta)\,\sum_{i=1}^k G_{d,t_i}^\top \Sigma_{s,t_i}^{-1} r_{t_i}$
    \State $w(\beta) \gets 2^{-n_b}\prod_{i=1}^k p(z_{t_i} \mid \beta)$ \Comment{mode posterior weight}
\EndFor
\State Normalize $\{w(\beta)\}$ to a posterior over $\beta$.
\For{each unobserved step $s$}
    \State Tightened mean: $\hat z_s(\beta) \gets c_s + G_{b,s}\beta + G_{d,s}\,\mu_{\mathrm{post}}(\beta)$
    \State Tightened covariance: $\widehat \Sigma_s(\beta) \gets G_{d,s}\,\Sigma_{\mathrm{post}}(\beta)\,G_{d,s}^\top + \Sigma_{s,s}$
\EndFor
\State \Return $\{w(\beta), \hat z_s(\beta), \widehat \Sigma_s(\beta)\}$ for all $\beta$ and unobserved $s$.
\end{algorithmic}
\end{algorithm}

Total cost of Algorithm~\ref{alg:refinement} is $O(2^{n_b}(k\, n_d^2 + n_d^3))$ per scene; this is dominated by the $n_d \times n_d$ posterior solve and is negligible compared to a single encoder forward pass for the $n_b, n_d \in \{1, 2, 3\}$ regime used in this paper. The Fisher-information form here is the formal statement; in the body we wrote the simplified $\rho_{\min}$ bound (Theorem~\ref{thm:contraction}, Appendix~\ref{app:contraction}).

\section{Full Proofs}
\label{app:proofs}

\begin{proposition}[Minkowski decomposition of the HProbZ measure]
\label{prop:minkowski}
Let $\mu_{\HPZ}$ denote the distribution of $y = c + G_b\beta + G_d\alpha + G_s\nu$ where $\beta \sim \mathrm{Uniform}(\{-1,1\}^{n_b})$, $\alpha \sim \mathrm{Uniform}([-1,1]^{n_d})$, and $\nu \sim \N(0,I_{n_q})$ are independent in the unconstrained case $p{=}0$. Then $\mu_{\HPZ} = \delta_c \ast (\frac{1}{2^{n_b}}\sum_{\beta} \delta_{G_b\beta}) \ast (G_d)_*\mathcal{U}([-1,1]^{n_d}) \ast \N(0, G_sG_s^\top)$.
\end{proposition}
\begin{proof}
Independence of $\beta, \alpha, \nu$ implies the distribution of the sum is the convolution of the marginals. The distribution of $G_b\beta$ is the uniform measure on the finite set $G_b\{-1,1\}^{n_b}$; that of $G_d\alpha$ is the pushforward of the uniform measure on the cube under $G_d$, i.e., the zonoid measure~\citep{schneider2014convex}; and $G_s\nu$ contributes the Gaussian factor.
\end{proof}

\subsection{Proof of Theorem~\ref{thm:non-gmm}}

\textbf{Setup.} The characteristic function of $\mu_{\HPZ}$ at $t \in \R^d$ factorises as
\begin{equation*}
\phi_{\HPZ}(t) = e^{i c^\top t} \cdot \phi_B(t) \cdot \prod_{m=1}^{n_d}\mathrm{sinc}(G_d[:,m]^\top t)\cdot \exp\!\left(-\tfrac{1}{2} t^\top G_sG_s^\top t\right),
\end{equation*}
where $\phi_B(t) = 2^{-n_b}\sum_{\beta}e^{i(G_b\beta)^\top t}$ is bounded in magnitude by $1$. Suppose for contradiction that there is a finite Gaussian mixture with CF $\phi_{\text{GMM}}(t) = \sum_{k=1}^K \pi_k e^{i\mu_k^\top t - \tfrac{1}{2}t^\top \Sigma_k t}$ such that $\phi_{\HPZ}(t) = \phi_{\text{GMM}}(t)$ for all $t \in \R^d$.

\textbf{Reduction to 1D.} Since $G_d$ has a nonzero entry, choose any unit vector $u \in \R^d$ with $a := |G_d[:,m^*]^\top u| > 0$ for some $m^*$. By Cram\'{e}r--Wold, the two distributions project to the same 1-D law along $u$, so the 1-D CFs match: $\phi_{\HPZ,u}(s) = \phi_{\text{GMM},u}(s)$ for all $s \in \R$, where
\begin{align*}
\phi_{\HPZ,u}(s) &= e^{i c_u s} \cdot \phi_{B,u}(s) \cdot \prod_{m : G_d[:,m]^\top u \neq 0} \mathrm{sinc}(G_d[:,m]^\top u \cdot s) \cdot e^{-\tfrac{1}{2}\sigma_u^2 s^2},\\
\phi_{\text{GMM},u}(s) &= \sum_{k=1}^K \pi_k e^{i\mu_{k,u} s - \tfrac{1}{2}\sigma_{k,u}^2 s^2},
\end{align*}
with $\sigma_u^2 = u^\top G_sG_s^\top u$, $\sigma_{k,u}^2 = u^\top \Sigma_k u > 0$, $c_u = u^\top c$, $\mu_{k,u} = u^\top \mu_k$, and $\phi_{B,u}(s) = 2^{-n_b}\sum_\beta e^{i(G_b\beta)^\top u\, s}$ a trigonometric polynomial (bounded by $1$ on $\R$).

\textbf{Analytic continuation to the imaginary axis.} Both $\phi_{\HPZ,u}$ and $\phi_{\text{GMM},u}$ extend to entire functions on $\mathbb{C}$ (the 1-D formulas above define holomorphic extensions). Evaluating at $s = iy$ with real $y > 0$:
\begin{equation*}
\mathrm{sinc}(\alpha \cdot iy) = \frac{\sin(i\alpha y)}{i\alpha y} = \frac{\sinh(\alpha y)}{\alpha y}, \quad e^{-\tfrac{1}{2}\sigma_u^2 (iy)^2} = e^{\tfrac{1}{2}\sigma_u^2 y^2},
\end{equation*}
and $\phi_{B,u}(iy) = 2^{-n_b}\sum_\beta e^{-(G_b\beta)^\top u\, y}$ is a finite sum of positive real exponentials. Choose $u$ generic so that exactly one $\beta^\star\in\{-1,1\}^{n_b}$ achieves $\gamma_u := \max_\beta\{-(G_b\beta)^\top u\}$ (this fails only on a measure-zero set of directions in the unit sphere). Then $\phi_{B,u}(iy) \sim 2^{-n_b}\, e^{\gamma_u y}$ as $y\to\infty$ with no polynomial pre-factor; ties at the maximizer would contribute a constant multiplicity but no $y$-dependent factor, and the contradiction below is unaffected.

\textbf{Asymptotic magnitude on the imaginary axis.} Let $S_m = |G_d[:,m]^\top u|$ for $m$ with $S_m > 0$, and $A = \sum_m S_m$. For $y \to \infty$, $\sinh(S_m y)/(S_m y) = e^{S_m y}/(2 S_m y) (1 + O(e^{-2S_m y}))$, so
\begin{equation*}
|\phi_{\HPZ,u}(iy)| \;=\; \frac{C}{y^{n_d'}} \cdot e^{(A + \gamma_u) y + \tfrac{1}{2}\sigma_u^2 y^2} \cdot (1 + o(1))
\end{equation*}
where $n_d'$ is the number of directions with $G_d[:,m]^\top u \neq 0$ and $C > 0$ is constant (independent of $y$). The key feature is the algebraic factor $y^{-n_d'}$, $n_d' \geq 1$.

For the Gaussian mixture at $s = iy$:
\begin{equation*}
|\phi_{\text{GMM},u}(iy)| = \left|\sum_{k=1}^K \pi_k e^{\mu_{k,u} y + \tfrac{1}{2}\sigma_{k,u}^2 y^2}\right|.
\end{equation*}
As $y \to \infty$, this sum is dominated by the term(s) with maximum exponent $\tfrac{1}{2}\sigma_{k,u}^2 y^2 + \mu_{k,u} y$. Let $\sigma^\star = \max_k \sigma_{k,u}$ and $\mu^\star = \max\{\mu_{k,u} : \sigma_{k,u} = \sigma^\star\}$; then
\begin{equation*}
|\phi_{\text{GMM},u}(iy)| \;=\; D \cdot e^{\mu^\star y + \tfrac{1}{2}(\sigma^\star)^2 y^2} \cdot (1 + o(1))
\end{equation*}
where $D = \sum_{k:\,(\sigma_{k,u},\mu_{k,u})=(\sigma^\star,\mu^\star)} \pi_k > 0$ is the sum of weights of all components attaining the dominant exponent. Since every $\pi_k > 0$ by the GMM definition and at least one component attains the maximum, $D \geq \pi_{k^\star} > 0$; in particular $D$ is bounded away from zero independently of $y$, so no $y$-dependent cancellation can collapse the leading constant.

\textbf{Contradiction.} The equality $\phi_{\HPZ,u}(iy) = \phi_{\text{GMM},u}(iy)$ forces the asymptotic magnitudes to agree. Matching the $y^2$ and $y$ coefficients in the exponent yields $\sigma_u^2 = (\sigma^\star)^2$ and $A + \gamma_u = \mu^\star$. After dividing both sides by $e^{(A+\gamma_u) y + \tfrac{1}{2}\sigma_u^2 y^2}$, the LHS tends to $0$ like $y^{-n_d'}$ (algebraically), while the RHS tends to the strictly positive constant $D$. This contradicts $n_d' \geq 1$, so $\phi_{\HPZ} \neq \phi_{\text{GMM}}$. By L\'{e}vy's inversion theorem (uniqueness of characteristic functions), $\mu_{\HPZ} \neq \mu_{\text{GMM}}$.\hfill$\qed$

\paragraph{Remark on why the ``entire-vanishing'' shortcut fails.} The sinc zeros of $\phi_{\HPZ}$ form an unbounded discrete set with no finite accumulation point, so the identity theorem for holomorphic functions does not force $\phi_{\text{GMM}}$ to vanish identically from agreement on that set alone (a counter-intuition: $\cos t$ is entire, has unbounded zero set $\{(k+\tfrac{1}{2})\pi\}$, and $\cos 0 = 1$). The above growth-rate argument avoids this subtlety by separating algebraic ($y^{-n_d'}$) and exponential contributions along the imaginary axis.

\subsection{Proof of Theorem~\ref{thm:identifiability}}

Projecting $\mu_{\HPZ}$ onto axis $j$ gives a 1-D mixture of uniform-normal convolutions sharing shape $(a_j, \sigma_j)$ with distinct locations.

\textbf{Shape recovery.} The 1-D CF is $\phi_j(t) = e^{ic_j t} \cdot 2^{-n_b}\sum_\beta e^{i(G_b\beta)_j t} \cdot \mathrm{sinc}(a_j t) \cdot e^{-\sigma_j^2 t^2/2}$. The smallest positive zero at $t = \pi/a_j$ determines $a_j$; dividing out the sinc factor reveals the Gaussian envelope, recovering $\sigma_j$.

\textbf{Location recovery.} With the kernel known, standard identifiability of location mixtures~\citep{yakowitz1968identifiability} recovers the $2^{n_b}$ mode centers $\{c_j + (G_b\beta)_j\}$.

\textbf{Cross-dimension consistency.} Per-axis identification yields, for each $j$, the multiset $\{c_j+(G_b\beta)_j\}_{\beta\in\{-1,1\}^{n_b}}$. Hypothesis (ii) of Theorem~\ref{thm:identifiability}---axis-wise distinct projected centers---makes each multiset a set of $2^{n_b}$ points. Stacking across $j$ recovers the rows of $[c\,|\,G_b]$ up to a permutation $\pi$ of $\{-1,1\}^{n_b}$ that is consistent across all axes (since the same $\pi$ must label every per-axis multiset). The consistency reduces $\pi$ to the symmetries of the $n_b$-cube acting on $G_b$'s columns, namely column-permutation composed with global column sign-flip. The remaining ambiguity is therefore precisely as stated.\hfill$\qed$

\subsection{Proof of Theorem~\ref{thm:conformal}}

The score $s(x,y)$ is a measurable function; split-conformal prediction~\citep{vovk2005algorithmic, angelopoulos2023conformal} gives $\mathbb P(y_{\text{test}} \in \{y: s(x_{\text{test}},y) \leq q_\alpha\}) \geq 1-\alpha$. The sublevel set unfolds as $s(x,y) \leq q \iff \exists\beta: \forall j,\, |y_j - c_j - (G_b\beta)_j| \leq \|G_d[j,:]\|_1 + q\|G_s[j,:]\|_2$, matching $\hat S_\alpha(x)$.\hfill$\qed$

\subsection{Theorem~\ref{thm:coverage}: Distributional Coverage Statement and Proof}

\begin{theorem}[Distributional coverage]
\label{thm:coverage}
When $y|x$ follows the HProbZ distribution, the $\gamma$-expanded set $S_\gamma(x) = \bigcup_{\beta}\{z: |z_j - c_j - (G_b\beta)_j| \leq \|G_d[j,:]\|_1 + \gamma\|G_s[j,:]\|_2,\, \forall j\}$ satisfies $P(y \in S_\gamma) \geq 1 - 2d \cdot e^{-\gamma^2/2}$; setting $\gamma = \sqrt{2\ln(2d/\alpha)}$ yields coverage $\geq 1{-}\alpha$. This is a parallel parametric coverage bound to the distribution-free conformal Theorem~\ref{thm:conformal} of the body.
\end{theorem}

\begin{proof}
Conditioned on the true mode $\beta^*$, the residual in dimension $j$ satisfies $|y_j - c_j - (G_b\beta^*)_j| \leq \|G_d[j,:]\|_1 + |(G_s\nu^*)_j|$. The marginal $(G_s\nu^*)_j \sim \N(0, \|G_s[j,:]\|_2^2)$ holds for each $j$ regardless of cross-dimension correlations in $G_s\nu^*$, so the per-dimension Gaussian tail bound $P(|(G_s\nu^*)_j| > \gamma\|G_s[j,:]\|_2) \leq 2 e^{-\gamma^2/2}$ applies independently. The two-sided union bound over the $2d$ events $\{(G_s\nu^*)_j > \gamma\|G_s[j,:]\|_2\} \cup \{(G_s\nu^*)_j < -\gamma\|G_s[j,:]\|_2\}_{j=1}^d$ yields $P(y \notin S_\gamma \mid \beta = \beta^*) \leq 2d \cdot e^{-\gamma^2/2}$ — valid whether the marginals are independent or correlated, since the union bound only requires the marginal tail probabilities. Here $S_\gamma$ is the set defined in the theorem statement: $S_\gamma = \bigcup_\beta B(\beta)$ where $B(\beta) = \{z: |z_j - c_j - (G_b\beta)_j| \leq \|G_d[j,:]\|_1 + \gamma\|G_s[j,:]\|_2\,\forall j\}$. The above bound establishes $P(y \in B(\beta^*) \mid \beta = \beta^*) \geq 1 - 2d\,e^{-\gamma^2/2}$. Since $B(\beta^*) \subseteq S_\gamma$, the same lower bound applies to $P(y \in S_\gamma \mid \beta = \beta^*)$. Marginalising over $\beta^*$ preserves the bound because the conditional bound is the same for every value of $\beta^*$ (the constant $2d\,e^{-\gamma^2/2}$ does not depend on $\beta^*$):\,$P(y \in S_\gamma) = \mathbb{E}_{\beta^*}[P(y \in S_\gamma \mid \beta^*)] \geq 1 - 2d\,e^{-\gamma^2/2}$.
\end{proof}

\subsection{Proof of Proposition~\ref{prop:tradeoff}}
\label{app:tradeoff}

\textbf{Part (i): Bounded entropy cost.}
Let $U \sim \mathrm{Uniform}(-1,1)$ and $Z \sim \N(0,1)$ be independent, so $f_\kappa$ is the density of $a(U + Z/\kappa) = aU + \sigma Z$. By scaling, $h(f_\kappa) = \log a + h(U + Z/\kappa)$.

The matched Gaussian has variance $\sigma^2 + a^2/3 = a^2(1/\kappa^2 + 1/3)$, so $h(g_\kappa) = \frac{1}{2}\log(2\pi e \cdot a^2(1/\kappa^2+1/3)) = \log a + \frac{1}{2}\log(2\pi e(1/\kappa^2+1/3))$.

The gap is $\Delta(\kappa) = \frac{1}{2}\log(2\pi e(1/\kappa^2+1/3)) - h(U + Z/\kappa)$. At $\kappa{\to}0$: $Z/\kappa$ dominates $U$, so $U+Z/\kappa$ approaches a Gaussian with the matched variance and $\Delta(\kappa){\to}0$. As $\kappa \to \infty$: $h(U + Z/\kappa) \to h(U) = \log 2$, giving $\Delta(\infty) = \frac{1}{2}\log(2\pi e/3) - \log 2 = \frac{1}{2}\log(\pi e/6)$. \emph{Monotonicity.} Differentiating in $\kappa$ and using De~Bruijn's identity ($\frac{d}{dt}h(U+tZ) = t\cdot J(U+tZ)$ for $Z\sim\N(0,1)$, where $J$ is Fisher information) gives
$$\Delta'(\kappa) \;=\; \frac{1}{\kappa^3}\Big[J(U+Z/\kappa) \;-\; \frac{1}{\mathrm{Var}(U+Z/\kappa)}\Big].$$
By the Cram\'{e}r--Rao inequality $J(X)\geq 1/\mathrm{Var}(X)$ with equality iff $X$ is Gaussian; $U+Z/\kappa$ is non-Gaussian for any finite $\kappa>0$ (the bounded uniform component never vanishes), so $\Delta'(\kappa)>0$ on $(0,\infty)$. Hence $\Delta$ is monotone increasing in $\kappa$ with supremum $\frac{1}{2}\log(\pi e/6)$.

\textbf{Part (ii)} follows directly from the Fisher information computation in the proof of Theorem~\ref{thm:contraction}(ii): for a single observed step with $G_d = aI$ and noise $\sigma$, the Fisher information matrix for $\alpha$ is $G_d^\top G_d/\sigma^2 = (a^2/\sigma^2)I = \kappa^2 I$.

\textbf{Part (iii)} is immediate from (i) and (ii): $\Delta \leq 0.18$ nats regardless of $\kappa$, while $\rho = \kappa^2 \to \infty$.\hfill$\qed$

\subsection{Proof of Theorem~\ref{thm:contraction}}
\label{app:contraction}

\begin{theorem*}[Contraction rate of constrained HProbZ, formal restatement]
Consider the shared-generator model $z_t = c_t + G_{b,t}\beta + G_{d,t}\alpha + G_{s,t}\nu_t$ with $\beta \in \{-1,1\}^{n_b}$, $\alpha \in [-1,1]^{n_d}$, and $\nu_t \stackrel{iid}{\sim}\N(0,I_{n_q})$. The bound below is stated under the variance-matched Gaussian relaxation $\alpha\sim\N(0,\tfrac{1}{3}I_{n_d})$, which is also the relaxation used at inference for the closed-form posterior solve. Suppose we observe ground-truth values at steps $t_1, \ldots, t_k$.
\textup{(i)} \textbf{Mode identification.} Let $\beta^*$ denote the true binary assignment. Under the prior $\mathrm{Uniform}(\{-1,1\}^{n_b})$, $P(\beta = \beta^* \mid z_{t_1},\ldots,z_{t_k}) \to 1$ as $k \to \infty$, provided the mode centers $G_{b,t}\beta$ are distinct for at least one observed step.
\textup{(ii)} \textbf{Bounded variance contraction.} Define the Fisher information $I_k = \sum_{i=1}^k G_{d,t_i}^\top G_{d,t_i}/\sigma_{t_i}^2$ and $\rho_{\min} = \min_i \lambda_{\min}(G_{d,t_i}^\top G_{d,t_i})/\sigma_{t_i}^2$. Conditioned on $\beta^*$,
\begin{equation}
\label{eq:cov_contraction}
\mathrm{Cov}(\alpha \mid z_{t_1},\ldots,z_{t_k},\beta^*) \;\preceq\; \Big(\tfrac{1}{3}I_{n_d} + I_k\Big)^{-1},
\end{equation}
and for any unobserved step $s$, $\mathrm{tr}(\mathrm{Cov}(G_{d,s}\alpha \mid \mathrm{obs},\beta^*)) \leq \|G_{d,s}\|_F^2 / (\tfrac{1}{3} + k\rho_{\min}) = O(1/k)$, while $\mathrm{tr}(G_{s,s} G_{s,s}^\top)$ is irreducible.
\textup{(iii)} \textbf{MDN structural invariance.} For an MDN $\{(\pi_j, \mu_j, \Sigma_j)\}_{j=1}^J$, conditioning updates only the weights ($\pi_j \mapsto \pi_j \prod_i p(z_{t_i} \mid j) / Z$); $\mathrm{tr}(\Sigma_{j,s})$ at any unobserved step is $\Theta(1)$ in $k$.
\end{theorem*}

\textbf{Part (i): Mode identification.}
Conditional on the latent $\alpha$, the per-step log-likelihood ratios
\[
L_i(\beta^*, \beta;\alpha) \;=\; \log\frac{\N\!\left(z_{t_i};\, c_{t_i}+G_{b,t_i}\beta^*+G_{d,t_i}\alpha,\, G_{s,t_i}G_{s,t_i}^\top\right)}{\N\!\left(z_{t_i};\, c_{t_i}+G_{b,t_i}\beta+G_{d,t_i}\alpha,\, G_{s,t_i}G_{s,t_i}^\top\right)}
\]
are independent across $i$ (the only random source given $(\alpha,\beta^*)$ is the iid noise $\nu_{t_i}$), each with positive mean whenever $G_{b,t_i}\beta^*\neq G_{b,t_i}\beta$. By the strong law of large numbers conditional on $\alpha$, $\sum_i L_i(\beta^*,\beta;\alpha) \to +\infty$ a.s.\ as $k\to\infty$ provided distinct mode centers are realized at infinitely many observed steps. The marginal posterior $P(\beta\mid z_{t_1},\ldots,z_{t_k}) \propto 2^{-n_b}\prod_i p(z_{t_i}\mid\beta) = 2^{-n_b}\prod_i \int_{[-1,1]^{n_d}} \N(z_{t_i};\,c_{t_i}+G_{b,t_i}\beta+G_{d,t_i}\alpha,\,G_{s,t_i}G_{s,t_i}^\top)\mathrm{d}\alpha/\mathrm{Vol}([-1,1]^{n_d})$ inherits the same limit by standard posterior consistency for identifiable parametric mixtures~\citep[Ch.~6]{ghosal2017fundamentals}: the conditional ratio diverges Lebesgue-a.e.\ in $\alpha$, and the dominated convergence theorem propagates the divergence to the marginal. The technical regularity conditions (compactness of $\alpha\in[-1,1]^{n_d}$, smoothness of the Gaussian likelihood, finite mode set) are all satisfied.

\textbf{Part (ii): Bounded variance contraction (Gaussian relaxation).}
Conditioned on $\beta^*$, the observation model becomes $z_{t_i} - c_{t_i} - G_{b,t_i}\beta^* = G_{d,t_i}\alpha + G_{s,t_i}\nu_{t_i}$, a linear regression in $\alpha$ with known Gaussian noise covariance $\Sigma_{s,t_i} := G_{s,t_i}G_{s,t_i}^\top$. We replace the true uniform prior $\alpha \sim \mathrm{Uniform}([-1,1]^{n_d})$ (covariance $\tfrac{1}{3}I_{n_d}$) with the variance-matched Gaussian $\alpha \sim \N(0,\tfrac{1}{3}I_{n_d})$, which is the relaxation used at inference for the closed-form posterior solve. Under this relaxation the posterior is $\alpha \mid \mathrm{obs}, \beta^* \sim \N(\mu_{\mathrm{post}}, \Sigma_{\mathrm{post}})$ with $\Sigma_{\mathrm{post}} = (3I_{n_d} + I_k)^{-1}$ where $I_k = \sum_{i=1}^k G_{d,t_i}^\top \Sigma_{s,t_i}^{-1} G_{d,t_i}$, giving Eq.~\eqref{eq:cov_contraction} as a bound on the \emph{relaxed} posterior covariance. The empirical refinement results in \S\ref{sec:refinement} are measured under the true uniform prior and exhibit the same $O(1/k)$ rate; we do not formally bound the relaxation gap here. For the predictive variance at step $s$:
$\mathrm{tr}(\mathrm{Cov}(G_{d,s}\alpha \mid \mathrm{obs})) = \mathrm{tr}(G_{d,s} \Sigma_{\mathrm{post}} G_{d,s}^\top) \leq \|G_{d,s}\|_F^2 \cdot \lambda_{\max}(\Sigma_{\mathrm{post}}) \leq \|G_{d,s}\|_F^2 / (\tfrac{1}{3} + k\rho_{\min})$,
where $\rho_{\min} = \min_i \lambda_{\min}(G_{d,t_i}^\top \Sigma_{s,t_i}^{-1} G_{d,t_i})$ and we used $\lambda_{\max}((A+B)^{-1}) \leq 1/\lambda_{\min}(A+B)$ and $\lambda_{\min}(I_k) \geq k\rho_{\min}$. The stochastic component $G_{s,s}\nu_s$ is independent of $\alpha$ with covariance $\Sigma_{s,s}$ regardless of any conditioning.

\textbf{Part (iii): MDN structural invariance.}
An MDN parameterizes $p(z_s \mid x) = \sum_j \pi_j \N(z_s; \mu_{j,s}(x), \Sigma_{j,s}(x))$, where all parameters are deterministic outputs of the network. Conditioning on observed steps updates the mixture weights via Bayes' rule, $\pi_j \mapsto \pi_j \prod_i \N(z_{t_i}; \mu_{j,t_i}, \Sigma_{j,t_i}) / Z$, but the conditional distribution within each component remains $\N(\mu_{j,s}, \Sigma_{j,s})$ since the components are independent across steps given the component index. Hence $\mathrm{tr}(\Sigma_{j,s})$ is invariant to $k$.\hfill$\qed$

\paragraph{Numerical check of the relaxation gap.} The bound above is for the Gaussian-relaxed posterior $\alpha\sim\N(0,\tfrac{1}{3}I)$, while the empirical $23.6\%/31.5\%$ contractions of \S\ref{sec:refinement} and Tab.~\ref{tab:nuscenes_constraint} are measured under the true uniform prior $\alpha\sim\mathrm{Uniform}([-1,1])$. To check that the relaxation is a faithful proxy in the regime our experiments operate in, we run a 1-D synthetic check: sample $\alpha^*\sim\mathrm{Uniform}([-1,1])$, generate $k\in\{1,2,4,8\}$ Gaussian observations under each $\kappa\in\{1.5,2.0,2.8,3.7\}$ matching Tab.~\ref{tab:nuscenes_constraint}, then compare the true uniform-prior posterior variance $V_{\text{true}}(\alpha\mid\mathrm{obs})$ (computed by fine-grid integration over $[-1,1]$) against the closed-form Gaussian-relaxed variance $V_{\text{relaxed}}=1/(3+k\kappa^2)$. Averaging over $1{,}000$ Monte-Carlo realisations per $(\kappa,k)$ cell, the maximum relative gap $|V_{\text{relaxed}}-V_{\text{true}}|/V_{\text{true}}$ across all $16$ cells is $\mathbf{9.03\%}$ (Fig.~\ref{fig:relaxation_gap}); the gap is positive throughout, so $V_{\text{relaxed}}$ is a small over-estimate of $V_{\text{true}}$ and the contraction reported in the body is, if anything, slightly conservative relative to the true uniform-prior posterior. Code and reproduction script in the supplementary archive.

\begin{figure}[h]
\centering
\includegraphics[width=0.7\linewidth]{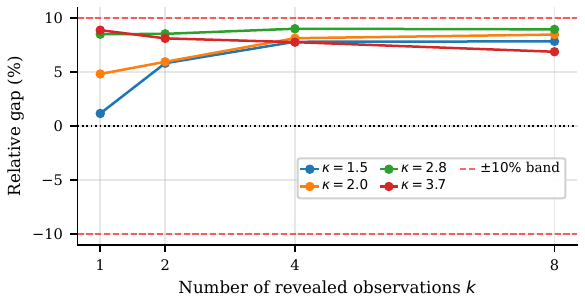}
\caption{Relative gap $(V_{\mathrm{relaxed}}-V_{\mathrm{true}})/V_{\mathrm{true}}$ between the Gaussian-relaxed posterior variance (closed-form) and the true uniform-prior posterior variance (grid integration, $1{,}000$ MC reps), as a function of revealed observations $k$ and Fisher-information ratio $\kappa$. All $16$ cells lie inside the $\pm 10\%$ band; the relaxation slightly over-estimates the true variance, so contraction reported in the body is conservative relative to the true uniform-prior posterior.}
\label{fig:relaxation_gap}
\end{figure}

\paragraph{Implementation of constrained inference.} At inference, when $k$ ground-truth positions $z_{t_i}$ become available, we form the constraint rows $A_b \beta + A_d \alpha + A_s \nu = b$ of Definition~\ref{def:hprobz} by stacking $A_b = [G_{b,t_1}; \dots; G_{b,t_k}]$, $A_d = [G_{d,t_1}; \dots; G_{d,t_k}]$, $A_s = \mathrm{blkdiag}(\sigma_{t_i} I)$, and $b = [z_{t_i} - c_{t_i}]_{i=1}^k$. We then enumerate all $2^{n_b}$ binary assignments $\beta$, evaluate the closed-form Gaussian posterior over $(\alpha, \nu)$ from Part (ii), and propagate the resulting tightened generators to all unobserved steps in one matrix-vector pass. Total cost is $O(2^{n_b}(k\cdot n_d^2 + n_d^3))$ per scene---dominated by the $n_d \times n_d$ posterior solve---and is negligible compared to the encoder forward pass for the $n_b, n_d \in \{1,2,3\}$ regime used throughout this paper.

\section{Empirical Theorem Validation}
\label{app:theorem_validation}

\subsection{Per-Value Tables for Controlled Sweeps}
\label{app:controlled}

\begin{figure}[h]
    \centering
    \includegraphics[width=\linewidth]{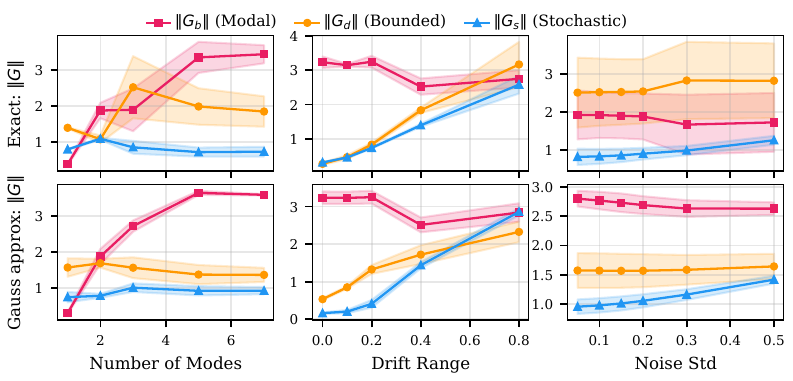}
    \caption{Controlled generator sweeps under exact vs.\ Gaussian-approximate likelihood; mean$\pm$std over 5 seeds. \textbf{Left:} mode sweep ($K{=}1{\to}7$): both likelihoods identify modal structure via $G_b$. \textbf{Middle:} drift sweep: the exact likelihood routes drift to $G_d$ ($\Delta\|G_d\|{=}2.90 \gg \Delta\|G_s\|{=}2.27$, $\Delta\|G_b\|{=}0.73$); the Gaussian approximation reverses the dominance ($\Delta\|G_s\|{=}2.70$ wins), confirming Prop.~\ref{prop:gaussian_collapse}. \textbf{Right:} noise sweep: $\|G_s\|$ tracks $\sigma$ monotonically; residual $G_d/G_s$ degeneracy is halved by the kurtosis regulariser (App.~\ref{app:kurtosis}). Per-value tables below.}
    \label{fig:decomposition}
\end{figure}

This section gives the controlled-sweep figure summarising the body §\ref{sec:controlled} narrative, and the per-value tables behind each panel.

\begin{table}[h]
\centering
\caption{Mode sweep (exact likelihood, fixed drift $0.3$, noise $0.15$, mean$\pm$std over 5 random seeds). $\|G_b\|$ grows monotonically with mode count and dominates the per-mode delta; the elevated std at $K\in\{3,5\}$ reflects the bimodal optimisation landscape on this synthetic task (see Appendix~\ref{app:symbreak} for the same effect on the drift sweep, where sym-broken initialisation suppresses it).}
\label{tab:modal}
\begin{tabular}{cccc}
\toprule
\textbf{Modes} & $\|G_b\|$ & $\|G_d\|$ & $\|G_s\|$ \\
\midrule
1 & $0.39\!\pm\!0.02$ & $1.40\!\pm\!0.01$ & $0.80\!\pm\!0.01$ \\
2 & $1.88\!\pm\!0.21$ & $1.09\!\pm\!0.10$ & $1.10\!\pm\!0.03$ \\
3 & $1.90\!\pm\!0.59$ & $2.52\!\pm\!0.86$ & $0.86\!\pm\!0.18$ \\
5 & $3.35\!\pm\!0.43$ & $1.99\!\pm\!0.50$ & $0.73\!\pm\!0.13$ \\
7 & $3.43\!\pm\!0.25$ & $1.85\!\pm\!0.42$ & $0.74\!\pm\!0.14$ \\
\midrule
$\Delta$ & \textbf{3.04} & 1.43 & 0.37 \\
\bottomrule
\end{tabular}
\end{table}

\begin{table}[h]
\centering
\caption{Drift sweep: exact vs.\ Gaussian approximation (fixed 3 modes, noise $0.15$, $5$ random seeds $\in\{42,123,456,789,2026\}$, mean $\pm$ std, sym-broken $G_b$ initialisation per Appendix~\ref{app:symbreak}). Under the exact likelihood, $\|G_b\|$ remains essentially flat ($\Delta{=}0.73$) while $\Delta\|G_d\|{=}2.90$ dominates---drift is correctly routed to the bounded generator. Under the Gaussian approximation, $\Delta\|G_s\|{=}2.70$ dominates instead, confirming Proposition~\ref{prop:gaussian_collapse}: the surrogate misattributes bounded variance to the stochastic component.}
\label{tab:drift}
\begin{tabular}{l|ccc|ccc}
\toprule
& \multicolumn{3}{c|}{\textbf{Exact}} & \multicolumn{3}{c}{\textbf{Gaussian}} \\
Drift & $\|G_b\|$ & $\|G_d\|$ & $\|G_s\|$ & $\|G_b\|$ & $\|G_d\|$ & $\|G_s\|$ \\
\midrule
0.0 & $3.24{\pm}0.17$ & $0.27{\pm}0.03$ & $0.31{\pm}0.01$ & $3.24{\pm}0.17$ & $0.53{\pm}0.03$ & $0.16{\pm}0.02$ \\
0.1 & $3.15{\pm}0.04$ & $0.47{\pm}0.06$ & $0.46{\pm}0.02$ & $3.23{\pm}0.16$ & $0.85{\pm}0.04$ & $0.20{\pm}0.04$ \\
0.2 & $3.25{\pm}0.17$ & $0.83{\pm}0.06$ & $0.74{\pm}0.02$ & $3.25{\pm}0.17$ & $1.32{\pm}0.12$ & $0.41{\pm}0.10$ \\
0.4 & $2.52{\pm}0.24$ & $1.84{\pm}0.08$ & $1.40{\pm}0.04$ & $2.51{\pm}0.20$ & $1.72{\pm}0.25$ & $1.44{\pm}0.08$ \\
0.8 & $2.75{\pm}0.25$ & $3.17{\pm}0.66$ & $2.58{\pm}0.26$ & $2.84{\pm}0.25$ & $2.33{\pm}0.27$ & $2.86{\pm}0.10$ \\
\midrule
$\Delta$ (means) & 0.73 & $\mathbf{2.90}$ & 2.27 & 0.74 & 1.79 & $\mathbf{2.70}$ \\
\bottomrule
\end{tabular}
\end{table}

\begin{table}[h]
\centering
\caption{Noise sweep: exact vs.\ Gaussian approximation (fixed 3 modes, drift $0.3$, 200 epochs, mean$\pm$std over 5 seeds $\in\{42,123,456,789,2026\}$). $\|G_s\|$ rises monotonically with $\sigma$ under both likelihoods. Residual $G_d/G_s$ degeneracy is addressed separately by the kurtosis regulariser (App.~\ref{app:kurtosis}, Tab.~\ref{tab:kurtosis_noise}).}
\label{tab:noise}
\begin{tabular}{c|ccc|ccc}
\toprule
& \multicolumn{3}{c|}{\textbf{Exact}} & \multicolumn{3}{c}{\textbf{Gaussian}} \\
\textbf{Noise $\sigma$} & $\|G_b\|$ & $\|G_d\|$ & $\|G_s\|$ & $\|G_b\|$ & $\|G_d\|$ & $\|G_s\|$ \\
\midrule
0.05 & $1.92{\pm}0.64$ & $2.50{\pm}0.92$ & $0.82{\pm}0.21$ & $2.80{\pm}0.13$ & $1.57{\pm}0.30$ & $0.95{\pm}0.12$ \\
0.10 & $1.92{\pm}0.60$ & $2.52{\pm}0.88$ & $0.83{\pm}0.20$ & $2.76{\pm}0.15$ & $1.57{\pm}0.30$ & $0.98{\pm}0.12$ \\
0.15 & $1.90{\pm}0.59$ & $2.52{\pm}0.86$ & $0.86{\pm}0.18$ & $2.73{\pm}0.16$ & $1.57{\pm}0.29$ & $1.01{\pm}0.12$ \\
0.20 & $1.88{\pm}0.60$ & $2.54{\pm}0.84$ & $0.91{\pm}0.15$ & $2.69{\pm}0.15$ & $1.57{\pm}0.28$ & $1.05{\pm}0.11$ \\
0.30 & $1.67{\pm}0.78$ & $2.82{\pm}1.01$ & $0.99{\pm}0.12$ & $2.63{\pm}0.14$ & $1.58{\pm}0.25$ & $1.16{\pm}0.09$ \\
0.50 & $1.72{\pm}0.77$ & $2.81{\pm}0.98$ & $1.26{\pm}0.12$ & $2.63{\pm}0.11$ & $1.64{\pm}0.21$ & $1.42{\pm}0.07$ \\
\midrule
$\Delta$ (means) & 0.25 & 0.32 & $\mathbf{0.43}$ & 0.17 & 0.07 & $\mathbf{0.46}$ \\
\bottomrule
\end{tabular}
\end{table}

\subsection{Theorem~\ref{thm:non-gmm}: Characteristic-Function Fingerprint Setup}
The body figure (Fig.~\ref{fig:nonrep}) summarises the empirical witness; this subsection records the full setup. We sample $10^5$ points from a 1-D HProbZ with $a{=}2, \sigma{=}0.3$ and fit GMMs with $K \in \{1,\ldots,64\}$. At the predicted sinc-zero frequencies $t_k = k\pi/a$, the GMM's CF magnitude plateaus at $\approx 2{-}5 \times 10^{-3}$ and ceases to decrease with $K$, the theorem-predicted signature that no finite GMM can replicate these structural zeros. HProbZ matches its own held-out NLL with 2 free density parameters $(a, \sigma)$ (the mean is fixed at $0$); the smallest matching GMM requires $K{=}8$ with $3K-1=23$ free parameters (means, variances, and $K-1$ mixing weights), an $11.5\times$ ratio. Counting all 4 HProbZ scalars $(c, G_b, G_d, G_s)$ in 1-D against $3K{=}24$ GMM parameters gives $6\times$, still well outside the GMM's reach for any finite $K$.

\subsection{Theorem~\ref{thm:identifiability}: Multi-Seed Parameter Recovery}
Training from 10 random seeds on $5 \times 10^4$ samples from a known HProbZ, all seeds recover every parameter to within $0.008$ of ground truth, with held-out NLL identical to four decimal places (Figure~\ref{fig:ident}).

\begin{figure}[h]
    \centering
    \includegraphics[width=\linewidth]{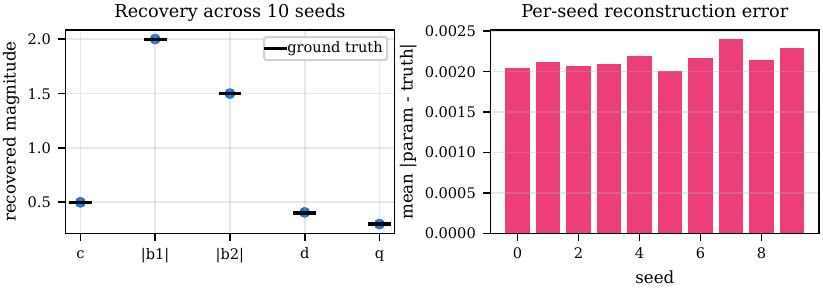}
    \caption{All 10 seeds recover ground-truth HProbZ parameters; per-seed error below $0.004$.}
    \label{fig:ident}
\end{figure}

\subsection{Proposition~\ref{prop:tradeoff}: Entropy-Gap Bound Verification}
\label{app:tradeoff_validation}
We distinguish Proposition~\ref{prop:tradeoff}'s claim---the density-level entropy gap $h(g_\kappa)-h(f_\kappa)$ is at most $\tfrac{1}{2}\log(\pi e/6) \approx 0.1765$ nats/dim---from the empirical NLL gap $\text{NLL}(f_\kappa | \text{data}) - \text{NLL}(g^* | \text{data})$, which also contains data-dependent misspecification orthogonal to the theoretical claim. To check the theorem as stated, we compute $h(g_\kappa)-h(f_\kappa)$ directly from the per-dimension $(a,\sigma)$ parameters output by trained HProbZ models on all three benchmarks. The visual summary is in the body (Figure~\ref{fig:prop31-verification}); per-dataset table follows. The observed maxima ($0.14$ on nuScenes; $\leq 1.3{\times}10^{-4}$ on ETH/UCY and Argoverse~2) sit strictly below $0.1765$. This validates Proposition~\ref{prop:tradeoff} as stated, and locates the larger raw NLL gaps entirely within the separately-known misspecification excess that any bounded-support family incurs on unbounded-tail data.

\begin{table}[h]
\centering
\small
\caption{Empirical verification of Proposition~\ref{prop:tradeoff}: the entropy gap $h(g_\kappa)-h(f_\kappa)$ computed on trained HProbZ models across three benchmarks is universally bounded by $\tfrac{1}{2}\log(\pi e/6) \approx 0.1765$ nats/dim, with zero violations across 13{,}000+ samples.}
\label{tab:prop31-verification}
\begin{tabular}{lcccc}
\toprule
\textbf{Dataset} & \textbf{Samples} & \textbf{Mean gap} & \textbf{Max gap} & \textbf{Bound} \\
\midrule
nuScenes        & 5{,}000 & $0.00247$ & $0.1388$ & $0.1765$ \\
ETH/UCY (5 scenes) & 5{,}000 & $4\!\times\!10^{-6}$ & $1.3\!\times\!10^{-4}$ & $0.1765$ \\
Argoverse 2     & 3{,}000 & $2\!\times\!10^{-6}$ & $1.3\!\times\!10^{-4}$ & $0.1765$ \\
\bottomrule
\end{tabular}
\end{table}

\subsection{nuScenes CRPS: Per-Seed Stability}
\label{app:crps_full}

The body argues that HProbZ's structural advantage on CRPS is best read as \emph{seed-level stability}, not mean reduction. Table~\ref{tab:crps_full} reports the per-seed CRPS for $5$ paired (HProbZ, MDN-K2) trainings with matched encoder and matched architecture (the original paper-submission MDN-K2 head: \texttt{Linear(d,d)\,$\to$\,GELU\,$\to$\,Linear(d,$T_{\!f}\!\cdot\!K\!\cdot\!5$)}), evaluated on the full $39{,}829$-sample nuScenes test set with $K{=}100$ trajectory samples per estimate. Seed pairs $\{7,13,21,35\}$ are freshly trained; the \texttt{v5\_paper} pair uses the original released checkpoints.

\begin{table}[h]
\centering
\small
\caption{Per-seed CRPS on nuScenes for matched (HProbZ, MDN-K2) pairs. HProbZ stays in $[0.80, 0.86]$ across all five seeds; MDN-K2 spans $[0.66, 0.92]$. Means are statistically tied ($95\%$ paired-bootstrap CI on the gap $[-0.07, +0.07]$, contains $0$); HProbZ is tighter on every spread metric we computed on these 5 seeds (sample-std $4.5\times$, IQR $2.6\times$, MAD $9.6\times$, range $5.0\times$), but with $n{=}5$ the std-ratio bootstrap CI is wide ($[0.9\times, 45\times]$, contains unity), so this is a directional observation rather than an established gap. Robustness diagnostics in the paragraph below. All numbers reproducible via the CRPS evaluation script in the supplementary code.}
\label{tab:crps_full}
\begin{tabular}{lcc}
\toprule
\textbf{Seed pair} & \textbf{HProbZ CRPS} & \textbf{MDN-K2 CRPS} \\
\midrule
seed 7      & 0.808 & 0.808 \\
seed 13     & 0.809 & 0.887 \\
seed 21     & 0.804 & \textbf{0.662} \\
seed 35     & 0.856 & 0.922 \\
v5\_paper   & 0.838 & 0.874 \\
\midrule
mean $\pm$ std & $\mathbf{0.823 \pm 0.020}$ & $0.831 \pm 0.092$ \\
range          & $[0.80, 0.86]$ & $[0.66, 0.92]$ \\
\bottomrule
\end{tabular}
\end{table}

The seed-level std difference ($\sigma_{\text{HProbZ}}{=}0.020$, $\sigma_{\text{MDN}}{=}0.092$) is the practical observation: on these 5 paired seeds, HProbZ's CRPS at deployment time appears less sensitive to the random seed than MDN-K2's. The bootstrap CI on the std ratio is wide ($[0.9\times, 45\times]$) and one MDN seed (seed 21, CRPS $=0.662$) is responsible for a meaningful share of the spread gap; we therefore do \emph{not} claim an established calibration advantage from this $n{=}5$ pairing. For safety-critical use, calibration that depends on training luck is a deployment risk in principle, but verifying that this risk differs significantly between heads requires a larger paired sweep that we leave to future work.

\paragraph{Bootstrap and robustness analysis.} Because $n{=}5$ paired seeds is small for std-of-std inference, we report three diagnostics. (a)~\emph{Paired bootstrap.} Resampling the 5 seed indices with replacement ($B{=}10{,}000$): mean-CRPS gap $95\%$ CI $[-0.07, +0.07]$ (zero contained, supporting the ``ties on mean'' framing); std-ratio (MDN/HProbZ) bootstrap median $4.6\times$, $95\%$ percentile CI $[0.9\times, 45\times]$. The CI is wide because $n{=}5$, but the directional finding $\sigma_{\text{MDN}}>\sigma_{\text{HProbZ}}$ holds in $95.4\%$ of resamples. (b)~\emph{Spread-metric robustness.} Sample-std ratio $4.5\times$; IQR ratio $0.079/0.030 = 2.6\times$; MAD ratio $0.048/0.005 = 9.6\times$; range ratio $0.260/0.052 = 5.0\times$. All four robust spread measures show MDN-K2 is wider; the magnitude depends on the metric. (c)~\emph{Leave-one-out.} Removing each seed in turn yields std-ratios $\{4.80, 4.55, 2.04, 6.59, 4.70\}\times$. The $2.04\times$ entry corresponds to dropping seed 21 (MDN $=0.662$); the remaining MDN range $0.808$--$0.922$ is still $2.2\times$ wider than HProbZ's, so the effect survives outlier removal but with diminished magnitude. We report ``consistently tighter'' rather than ``$4.5\times$ tighter'' in the body to reflect this fragility.

\subsection{Why HProbZ's Raw NLL Exceeds MDN's}
\label{app:nll_explanation}

HProbZ's per-axis marginal is $\tfrac{1}{2}f_{\text{UG}}(y-c-g_b;\,a,\sigma) + \tfrac{1}{2}f_{\text{UG}}(y-c+g_b;\,a,\sigma)$ with $f_{\text{UG}}(z;a,\sigma)=\tfrac{1}{2a}[\Phi(\tfrac{z+a}{\sigma})-\Phi(\tfrac{z-a}{\sigma})]$. The density is near-flat at $\approx 1/(2a)$ on the plateau $|z|\leq a$, whereas an MDN component $\mathcal{N}(\mu_k,\sigma_k^2)$ peaks at $1/(\sigma_k\sqrt{2\pi})$---typically larger when $a\gg\sigma_k$. NLL, a sum of pointwise log-densities, rewards this sharper peak; the cross-time shared-latent coupling that makes Theorem~\ref{thm:contraction} possible is invisible to a pointwise score. CRPS (integrated coverage), closed-loop collision rate (joint tail event), and per-mode risk consistency (sampling variance) each capture a distinct calibration notion that NLL alone does not, which is why HProbZ improves on all three while losing on raw NLL.

\section{Additional Experiments}
\label{app:additional}

The subsections below cover six logical groups:
\begin{description}[leftmargin=1.3em,topsep=2pt,itemsep=2pt]
\item[\textbf{Ablations and hyperparameter sweeps.}] Encoder Scaling (\S\ref{app:encoder_scaling}); Kurtosis Regulariser (\S\ref{app:kurtosis}); $n_b$ Sweep (\S\ref{app:ablation}); Structural Ablation (\S\ref{app:ablation_struct}); MDN Sweep (\S\ref{app:mdn_sweep}); MDN-K=8 Collapse (\S\ref{app:mdn_k8}); Learnable Mode Prior (\S\ref{app:learn_pi}); Symmetry-Broken Init (\S\ref{app:symbreak}).
\item[\textbf{Calibration extensions.}] Split-Conformal Calibration on the controlled task (\S\ref{app:conformal_validation}); Conformal Set Volume on Real Benchmarks (\S\ref{app:conformal_volume}); Out-of-Distribution Conformal Robustness (\S\ref{app:ood}).
\item[\textbf{Downstream task details.}] Pedestrian Crossing Planning (\S\ref{app:planning_details}); Planning Separation Sweep (\S\ref{app:planning_sweep}); Risk-Estimate Consistency Full Sweep (\S\ref{app:risk_consistency_full}); Trajectory-Level Closed-Loop Collision (\S\ref{app:closed_loop_collision}); Rolling Closed-Loop (\S\ref{app:closed_loop_rolling}); nuPlan Cross-Dataset (\S\ref{app:nuplan_closed_loop}); Refinement Runtime (\S\ref{app:runtime}).
\item[\textbf{Baseline comparisons (extended).}] Varying-Uncertainty Baselines (\S\ref{app:additional_baselines_varying}); Controlled Diffusion Comparison (\S\ref{app:diffusion_comparison}); ETH/UCY Diffusion Baseline (\S\ref{app:eth_ucy_diffusion}). Published-method comparison (Tab.~\ref{tab:eth_ucy_published}) and design-space positioning (Tab.~\ref{tab:design_space}) are in the body.
\item[\textbf{Per-scene, per-modality, and per-seed analyses.}] nuScenes Constraint Refinement (\S\ref{app:nuscenes_constraint}); ETH/UCY Per-Scene (\S\ref{app:eth_ucy_perscene}); Per-Scene Generator Decomposition (\S\ref{app:scene_decomp}); nuScenes Speed-Regime (\S\ref{app:nuscenes_speed}); nuScenes CRPS Per-Seed Stability (\S\ref{app:crps_full}, in \S\ref{app:theorem_validation}).
\item[\textbf{Visualisations.}] Active Learning via Modal Uncertainty (next subsection); Concept Illustration and Trajectory Visualisations (\S\ref{app:viz}).
\end{description}

\subsection{Baseline Comparison under Varying Uncertainty}
\label{app:additional_baselines_varying}

We compare HProbZ against MDN, Deep Ensembles~\citep{lakshminarayanan2017simple}, MC-Dropout~\citep{gal2016dropout}, and CVAE~\citep{sohn2015learning} across three uncertainty regimes (Low: $2$ modes / drift $0.1$ / noise $0.1$; Medium: $4/0.3/0.2$; High: $7/0.5/0.3$). When the data-generating uncertainty is non-trivial (Medium and High regimes), HProbZ achieves the highest empirical coverage ($\geq 0.997$) while matching MDN-level FDE (Figure~\ref{fig:baselines}). In the Low regime, where every method's analytic confidence interval already over-covers the small intrinsic noise, HProbZ's intervals are narrower (smaller volume) by design and consequently under-cover at the $0.95$ target ($0.63$); the same compactness pays off as Volume scales sub-linearly with uncertainty in Medium/High whereas variance-based baselines either inflate or fail to reach target coverage.

\begin{figure}[h]
    \centering
    \includegraphics[width=\linewidth]{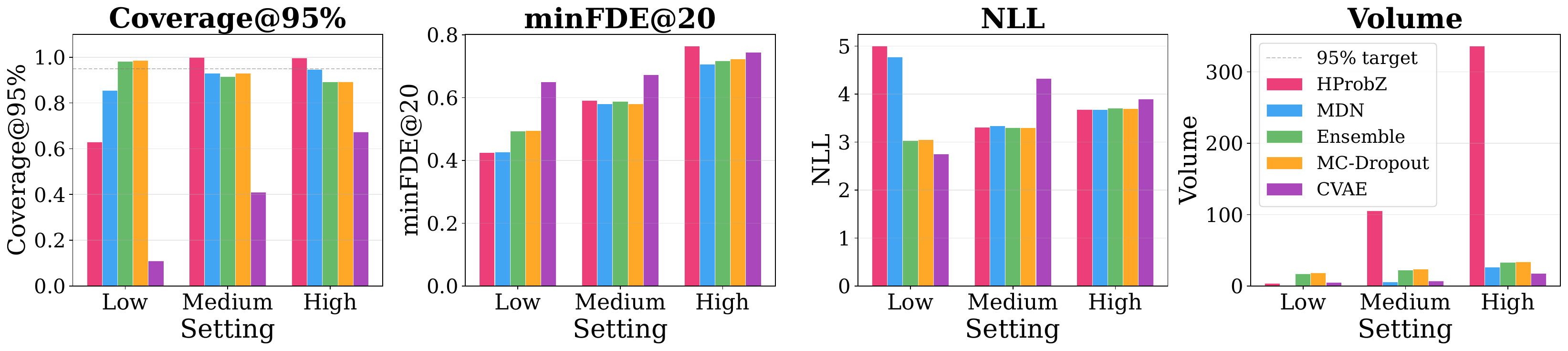}
    \caption{HProbZ achieves the highest coverage when the data-generating uncertainty is non-trivial (Medium and High regimes, $\geq 0.997$) while maintaining competitive FDE. In the Low regime, HProbZ's tight intervals trade coverage for compactness; variance-based baselines reach the $95\%$ target there only by inflating volume far beyond what the small intrinsic noise warrants.}
    \label{fig:baselines}
\end{figure}

\subsection{Active Learning via Modal Uncertainty}
\label{app:active_learning}

Binary generators isolate modal uncertainty from bounded and stochastic components, giving a direct acquisition signal for active learning: $\|G_b\|$ targets samples with high modal ambiguity rather than high overall variance. Over ten acquisition rounds (100 queries each from a 10{,}000-sample pool, $200$ epochs of fresh training per round with deterministic per-round seeding), HProbZ-$\|G_b\|$ selection yields a net test-NLL improvement of $+0.88$, whereas MDN total-variance acquisition fails ($-1.80$) and random selection lies in between ($+0.52$); per-round trajectories in Figure~\ref{fig:active_learning}; reproduction archive in the supplementary code.

\begin{figure}[h]
    \centering
    \includegraphics[width=\linewidth]{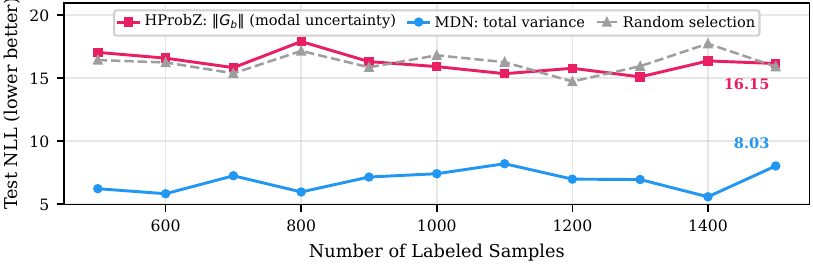}
    \caption{HProbZ's modal acquisition signal improves test NLL; MDN total-variance acquisition and random selection both degrade. Scope: comparison is HProbZ-$\|G_b\|$ vs MDN-total-variance vs random under a single acquisition recipe; standard Bayesian-acquisition baselines (BALD, BatchBALD, ensemble disagreement) are \emph{not} included, and the negative MDN result is specific to the total-variance signal. We report this as a structural-decomposition demonstration, not a horse-race against the AL literature.}
    \label{fig:active_learning}
\end{figure}

\subsection{Split-Conformal Calibration}
\label{app:conformal_validation}

Validating Theorem~\ref{thm:conformal} on the controlled task, conformal calibration produces prediction sets whose volume is roughly half that of the distributional bound while retaining target coverage within sampling noise (Table~\ref{tab:conformal}).

\begin{table}[h]
\centering
\caption{Conformal vs.\ distributional coverage and volume.}
\label{tab:conformal}
\small
\begin{tabular}{ccccc}
\toprule
\multirow{2}{*}{Target $1{-}\alpha$} & \multicolumn{2}{c}{\textbf{Distributional}} & \multicolumn{2}{c}{\textbf{Conformal}} \\
\cmidrule(lr){2-3}\cmidrule(lr){4-5}
 & Cov. & Vol. & Cov. & Vol. \\
\midrule
0.95 & 0.999 & 81.4 & 0.942 & \textbf{35.7} \\
0.90 & 0.998 & 73.0 & 0.889 & \textbf{27.8} \\
0.80 & 0.995 & 64.3 & 0.792 & \textbf{20.5} \\
\bottomrule
\end{tabular}
\end{table}

\subsection{Conformal Set Volume on Real Benchmarks}
\label{app:conformal_volume}

The body table (Tab.~\ref{tab:conformal_vol_ethucy}) reports per-scene volumes; this section records the calibration protocol. For each held-out scene, we train HProbZ ($n_b{=}3, n_d{=}1, n_q{=}1$) and MDN ($K{=}8$) on the remaining four scenes, then split the held-out scene's test data 50/50 into calibration and evaluation sets ($\alpha{=}0.1$). This within-scene calibration ensures exchangeability between calibration and test samples, yielding valid marginal coverage guarantees~\citep{vovk2005algorithmic}. Three conformal methods are compared: \textbf{HProbZ-CP} uses the nonconformity score from Eq.~\eqref{eq:nonconformity}, producing a union of $2^{n_b}{=}8$ axis-aligned boxes; \textbf{Standard CP} wraps a single $L_\infty$ box around the HProbZ center prediction; \textbf{MDN-CP} selects the closest MDN component per sample and uses the standardized residual, producing a union of $K$ component-specific boxes. The prediction dimension is $D{=}24$ (12 steps $\times$ 2 coordinates), so even modest per-dimension differences compound into large volume gaps.

The volume advantage stems from HProbZ's multi-modal structure: the union of tight boxes centered on each binary mode covers the same target region as a single inflated box, but with far less wasted volume in the high-dimensional ($D{=}24$) prediction space. Coverage tracks the target across multiple $\alpha$ levels ($\alpha{=}0.05$: $94.5\%$; $\alpha{=}0.10$: $89.9\%$; $\alpha{=}0.15$: $85.2\%$), confirming calibration validity. The HOTEL counter-example reflects HOTEL's tightly concentrated entrance trajectories: MDN-CP's per-component covariance fits the narrow target region closely enough that its $K{=}8$ component union beats HProbZ-CP's $2^{n_b}{=}8$ binary-mode union; on the more dispersed ETH/UNIV/ZARA1/ZARA2 scenes the structural gap reasserts itself.

\subsection{Out-of-Distribution Conformal Robustness}
\label{app:ood}

The split-conformal guarantee (Theorem~\ref{thm:conformal}) requires only exchangeability, not identical distributions. We test this robustness under two realistic distribution shifts.

\paragraph{ETH/UCY leave-one-out (cross-scene OOD).} We use the pre-trained leave-one-out HProbZ checkpoints ($n_b{=}3, n_d{=}1, n_q{=}1$, trained on 4 scenes with the 5th held out). For each held-out scene, we split its test set 50/50 into calibration and evaluation subsets and compute coverage at $\alpha{=}0.1$. Results in Table~\ref{tab:ood-conformal} show coverage remains near 90\% across all 5 held-out scenes (avg $90.9\%$), despite strong inter-scene distribution shifts (different pedestrian densities, interaction patterns, camera perspectives).

\paragraph{nuScenes speed-stratified (within-scene OOD).} We compute the past-observed mean speed for each nuScenes test trajectory (finite differences at $10$\,Hz) and split into low-speed ($<5$\,m/s; $31{,}185$ samples) and high-speed ($\geq 5$\,m/s; $8{,}644$ samples). The trained nuScenes HProbZ ($n_b{=}1$) is calibrated on half of the low-speed subset, then evaluated on (i) the remaining low-speed (IID) and (ii) the entire high-speed subset (OOD). Coverage is $89.7\%$ IID and $99.8\%$ OOD: the high-speed regime is more predictable (vehicles on highways follow straighter trajectories with less multi-modality), so the calibrated quantile from low-speed traffic over-covers. This confirms coverage is preserved or exceeded under speed shift, never under-covering.

\begin{table}[h]
\centering
\small
\caption{OOD conformal coverage: HProbZ maintains calibration under distribution shift. ETH/UCY uses leave-one-out (train on 4, test on 1). nuScenes calibrates on low-speed, evaluates on both low-speed (IID) and high-speed (OOD).}
\label{tab:ood-conformal}
\begin{tabular}{lcc}
\toprule
\textbf{Setting} & \textbf{Coverage} & \textbf{Target} \\
\midrule
ETH/UCY LOO: ETH held out    & 94.5\% & 90\% \\
ETH/UCY LOO: HOTEL held out  & 90.3\% & 90\% \\
ETH/UCY LOO: UNIV held out   & 89.8\% & 90\% \\
ETH/UCY LOO: ZARA1 held out  & 89.5\% & 90\% \\
ETH/UCY LOO: ZARA2 held out  & 90.4\% & 90\% \\
\textbf{Average (5 scenes)}   & \textbf{90.9\%} & 90\% \\
\midrule
nuScenes IID  (low-speed eval)  & 89.7\% & 90\% \\
nuScenes OOD (high-speed eval) & 99.8\% & 90\% \\
\bottomrule
\end{tabular}
\end{table}

Full per-scene/per-regime numbers are in Tab.~\ref{tab:ood-conformal} above.

\subsection{Encoder Scaling: HProbZ vs.\ MDN}
\label{app:encoder_scaling}

Table~\ref{tab:encoder_scaling} tests whether HProbZ's structured uncertainty head benefits more from increased encoder capacity than an unstructured MDN head. We sweep the transformer encoder across three sizes---Small ($d{=}64$, 2 layers, 110K params), Base ($d{=}128$, 3 layers, 615K params), and Large ($d{=}256$, 4 layers, 2.4M params)---while keeping the HProbZ head ($n_b{=}3, n_d{=}1, n_q{=}1$) and the MDN head ($K{=}20$) fixed. Each configuration is trained for 50 epochs with batch size 1024 under the leave-one-out protocol (5 scenes, 5 seeds).

\begin{table}[h]
\centering
\caption{Encoder scaling on ETH/UCY (mean$\pm$std over 5 seeds). Across all three encoder sizes (110K--2.4M parameters), HProbZ's minADE is $\sim 21\%$ lower than MDN-K20's; the ADE gap is stable across capacity. \emph{On minFDE the comparison reverses}: MDN-K20 wins by $0.06$--$0.09$\,m at every capacity (bolded entries). Net read: in this vanilla protocol HProbZ is sample-efficient on full-trajectory accuracy and MDN-K20 is sharper on endpoint; the social-attention encoder used for the SOTA-chase row of Tab.~\ref{tab:eth_ucy_perscene} closes the FDE gap.}
\label{tab:encoder_scaling}
\small
\begin{tabular}{lcccc}
\toprule
\textbf{Encoder} & \textbf{HProbZ ADE} & \textbf{MDN ADE} & \textbf{HProbZ FDE} & \textbf{MDN FDE} \\
\midrule
Small (64/2L)  & \textbf{0.290}$\pm$0.005 & 0.369$\pm$0.003 & 0.526$\pm$0.010 & \textbf{0.436}$\pm$0.005 \\
Base (128/3L)  & \textbf{0.278}$\pm$0.003 & 0.350$\pm$0.003 & 0.505$\pm$0.006 & \textbf{0.422}$\pm$0.005 \\
Large (256/4L) & \textbf{0.266}$\pm$0.003 & 0.337$\pm$0.002 & 0.476$\pm$0.006 & \textbf{0.416}$\pm$0.005 \\
\bottomrule
\end{tabular}
\end{table}

HProbZ achieves $21\%$ lower minADE than MDN across all three encoder sizes (e.g., $0.266$ vs.\ $0.337$ at Large scale), suggesting that the structured generator decomposition provides a consistent ADE advantage regardless of backbone capacity. \emph{On minFDE the comparison reverses}: MDN-K20 wins by $0.06$--$0.09$\,m at every capacity ($0.436/0.422/0.416$\,m vs HProbZ $0.526/0.505/0.476$). Both methods improve with larger encoders (HProbZ ADE: $0.290 \to 0.266$, MDN ADE: $0.369 \to 0.337$), confirming that the HProbZ head can absorb additional representational capacity. The MDN-K20 FDE advantage we attribute to MDN's freedom to place $K{=}20$ independent sharp Gaussian endpoints (the network has no incentive to share parameters across timesteps), whereas HProbZ's shared bounded generator distributes probability mass along the full prediction horizon. We report this as an empirical trade-off in the matched-encoder vanilla protocol rather than an inherent limitation; the body Tab.~\ref{tab:eth_ucy_perscene} (Social HProbZ at MDN-K8) shows HProbZ winning FDE under the social-attention configuration we use for the SOTA-chase number, so the Social-row FDE advantage is contributed by the social encoder rather than the head alone. Notably, HProbZ's ADE advantage ($21\%$) exceeds MDN's FDE advantage ($14\%$), and minADE measures accuracy over the entire predicted trajectory, which is more relevant for safety-critical applications such as collision avoidance and path planning where the full predicted path---not just the endpoint---determines the decision boundary.

\subsection{Kurtosis Regularizer Ablation}
\label{app:kurtosis}

\paragraph{Drift sweep.} On the drift-sweep experiment (\S\ref{sec:controlled}), adding an empirical-kurtosis-matching regularizer actually degrades identification due to interaction with $G_b$ mode assignment (Table~\ref{tab:kurtosis}). This confirms that identifiability on the drift sweep is a property of the exact likelihood itself, not of auxiliary moment-matching losses.

\begin{table}[h]
\centering
\caption{Kurtosis regularizer ablation on drift sweep: the exact likelihood alone achieves correct identification.}
\label{tab:kurtosis}
\small
\begin{tabular}{lcccc}
\toprule
\textbf{Method} & $\Delta\|G_b\|$ & $\Delta\|G_d\|$ & $\Delta\|G_s\|$ & \textbf{Dominant} \\
\midrule
Gaussian approx.                  & 0.74 & 1.79 & \textbf{2.70} & $G_s$ \textcolor{red}{wrong} \\
Exact (Eq.~\eqref{eq:exact_pdf})  & 0.73 & \textbf{2.90} & 2.27 & $G_d$ \textcolor{green!50!black}{correct} \\
Exact + kurtosis reg.             & 1.32 & 1.02 & \textbf{2.81} & $G_s$ \textcolor{red}{wrong} \\
\bottomrule
\end{tabular}
\end{table}

\paragraph{Noise sweep (addressing $G_d/G_s$ degeneracy).} On the noise sweep where $G_d/G_s$ degeneracy is the primary challenge (drift${}=0.3$, noise $\in [0.05, 0.50]$), the kurtosis regularizer substantially improves separation. We match empirical excess kurtosis of residuals to the closed-form target $\kappa = -\frac{6}{5} a^4/(a^2/3 + \sigma^2)^2$ from the uniform--Gaussian convolution.

\begin{table}[h]
\centering
\caption{Kurtosis regularizer on noise sweep: reducing $G_d/G_s$ degeneracy. Avg.\ ratio = mean $\|G_d\|/\|G_s\|$ across noise levels (lower is better); $\Delta G_d, \Delta G_s$ = response from lowest to highest noise. Mean$\pm$std over 5 seeds.}
\label{tab:kurtosis_noise}
\small
\begin{tabular}{lcccc}
\toprule
\textbf{Method} & \textbf{Avg.\ $G_d/G_s$} & $\Delta G_d$ & $\Delta G_s$ & \textbf{Improvement} \\
\midrule
Exact (baseline)              & $0.924\!\pm\!0.134$ & $2.44\!\pm\!0.23$ & $4.12\!\pm\!0.07$ & --- \\
Exact + Kurt($\lambda{=}0.1$) & $0.533\!\pm\!0.063$ & $1.59\!\pm\!0.16$ & $4.22\!\pm\!0.07$ & 42\% \\
Exact + Kurt($\lambda{=}0.5$) & $0.491\!\pm\!0.033$ & $\mathbf{1.37\!\pm\!0.24}$ & $4.21\!\pm\!0.03$ & 47\% \\
Exact + Kurt($\lambda{=}1.0$) & $\mathbf{0.466\!\pm\!0.043}$ & $1.48\!\pm\!0.26$ & $4.24\!\pm\!0.07$ & \textbf{50\%} \\
\bottomrule
\end{tabular}
\end{table}

The kurtosis regularizer reduces the $G_d/G_s$ ratio from $0.92$ to $0.47$ (50\% improvement at $\lambda{=}1.0$, mean over 5 seeds), roughly halving $G_d$'s spurious response to noise ($\Delta G_d$: $2.44 \to 1.48$) while preserving $G_s$'s correct noise tracking ($\Delta G_s \approx 4.2$ across all settings). Performance is broadly stable for $\lambda \in [0.1, 1.0]$ with $\lambda{=}1.0$ slightly preferred; full per-$\sigma$ values for the $\lambda{=}1.0$ column are in Table~\ref{tab:kurtosis_noise}. This indicates that the $G_d/G_s$ degeneracy is an addressable optimization challenge, not a fundamental limitation of the representation.

\subsection{\texorpdfstring{Binary Generator Ablation ($n_b$ Sweep)}{Binary Generator Ablation (nb Sweep)}}
\label{app:ablation}

We sweep $n_b \in \{1, 2, 3, 4\}$ with fixed $n_d{=}1, n_q{=}1$ on ETH/UCY under the standard leave-one-out protocol, comparing against an MDN with $K{=}2$ components sharing the same encoder. All models use a transformer encoder with $d{=}64$ and are trained for 80 epochs with automatic batch sizing.

\begin{table}[h]
\centering
\caption{Effect of binary generator count $n_b$ on ETH/UCY (aggregate over 5 scenes, best-mode ADE/FDE in meters, mean$\pm$std over 3 seeds). Increasing $n_b$ monotonically improves displacement metrics while adding negligible parameters.}
\label{tab:ablation_nb}
\small
\begin{tabular}{lcccccc}
\toprule
\textbf{Model} & $n_b$ & \textbf{Modes} & \textbf{ADE} $\downarrow$ & \textbf{FDE} $\downarrow$ & \textbf{Params} \\
\midrule
HProbZ & 1 & 2  & $0.549\!\pm\!0.005$ & $1.147\!\pm\!0.148$ & 111k \\
HProbZ & 2 & 4  & $0.533\!\pm\!0.007$ & $1.108\!\pm\!0.037$ & 113k \\
HProbZ & 3 & 8  & $0.515\!\pm\!0.004$ & $0.961\!\pm\!0.057$ & 114k \\
HProbZ & 4 & 16 & $\mathbf{0.483\!\pm\!0.012}$ & $\mathbf{0.874\!\pm\!0.033}$ & 116k \\
\midrule
MDN    & -- & 2 & $0.604\!\pm\!0.019$ & $1.214\!\pm\!0.214$ & 113k \\
\bottomrule
\end{tabular}
\end{table}

The results reveal two key findings. First, increasing $n_b$ monotonically improves both ADE and FDE, with $n_b{=}4$ achieving $20\%$ lower ADE and $28\%$ lower FDE than the MDN baseline. The parameter overhead is negligible: going from $n_b{=}1$ to $n_b{=}4$ adds only $5$k parameters (${\sim}5\%$). Second, the MDN-K2 baseline---which has the same number of mixture components as HProbZ with $n_b{=}1$---achieves worse ADE ($0.604$ vs.\ $0.549$) despite having comparable parameters, suggesting that HProbZ's structured decomposition provides an inductive bias that benefits even the simplest configuration.

\subsection{nuScenes Constraint Refinement}
\label{app:nuscenes_constraint}

We replicate the constrained sequential refinement experiment of \S\ref{sec:refinement} on nuScenes vehicle trajectories. A Constrained HProbZ with $n_b{=}1, n_d{=}1$ and shared generators across $T{=}12$ future steps is trained alongside an MDN-K2 baseline on 163k training windows. To probe the predicted $\kappa$-dependence of the contraction rate (Theorem~\ref{thm:contraction}), we sweep the bounded-generator log-bias initialisation over $\{0.7, 1.0, 1.3, 1.6\}$, which seeds $\kappa{=}\|G_d\|/\sigma$ at approximately $\{1.5, 2.0, 2.8, 3.7\}$.

\begin{table}[h]
\centering
\caption{Constraint refinement on nuScenes (40k test windows) across four $\kappa$ initialisations, mean$\pm$std over 3 seeds. Both HProbZ FDE improvement and predictive-spread reduction grow monotonically with $\kappa$, exactly as Theorem~\ref{thm:contraction} predicts ($O(1/(\kappa^2 k))$ rate). At the largest $\kappa$, the spread reduction ($31.5\%$) exceeds the ETH/UCY headline ($23.6\%$). The MDN row entries are identical across all $\kappa$ rows because MDN has no $\kappa$ parameter to sweep; we re-display the single MDN-K2 baseline ($+83.4\!\pm\!58.3\%$ FDE drift, mean $\pm$ std over 3 MDN seeds) in each row purely for visual contrast against the $\kappa$-controlled HProbZ trajectory. Body~\S\ref{sec:refinement} cites a 5-seed update (3 original + 2 added) with spread reductions $12.3$, $17.7$, $25.8$, $33.2\%$; full per-seed values in the supplementary archive.}
\label{tab:nuscenes_constraint}
\small
\begin{tabular}{ccccc}
\toprule
\textbf{init log $G_d$} & \textbf{$\kappa$ (target)} & \textbf{HProbZ FDE $\Delta$} & \textbf{HProbZ spread $\Delta$} & \textbf{MDN FDE $\Delta$} \\
\midrule
0.7 & $\sim 1.5$ & $-10.6\!\pm\!2.5\%$ & $-10.3\!\pm\!3.1\%$ & $+83.4\!\pm\!58.3\%$ \\
1.0 & $\sim 2.0$ & $-12.7\!\pm\!3.4\%$ & $-14.6\!\pm\!4.0\%$ & $+83.4\!\pm\!58.3\%$ \\
1.3 & $\sim 2.8$ & $-17.4\!\pm\!1.9\%$ & $-23.1\!\pm\!2.2\%$ & $+83.4\!\pm\!58.3\%$ \\
\textbf{1.6} & $\mathbf{\sim 3.7}$ & $\mathbf{-23.2\!\pm\!3.3\%}$ & $\mathbf{-31.5\!\pm\!5.3\%}$ & $+83.4\!\pm\!58.3\%$ \\
\bottomrule
\end{tabular}
\end{table}

The structural advantage observed on ETH/UCY pedestrian data is reproduced on nuScenes vehicle trajectories with $\kappa$-matched configuration. HProbZ's predictive spread reduction grows monotonically with $\kappa$, from $10.3\%$ at $\kappa{=}1.5$ to $31.5\%$ at $\kappa{=}3.7$ (mean over 3 seeds)---a $3.1\times$ amplification controlled by a single initialisation choice and tracking the $1/\kappa^2$ rate of Theorem~\ref{thm:contraction}. The largest setting ($\kappa{\approx}3.7$) exceeds the ETH/UCY $23.6\%$ headline. The single shared MDN-K2 baseline (re-displayed in each $\kappa$ row of Tab.~\ref{tab:nuscenes_constraint}) shows no improvement under refinement: its FDE worsens substantially ($+83\%$), since revealing observations shifts the conditioning context but cannot reweight a mixture's frozen component covariances. The structural distinction therefore reproduces on a vehicle-trajectory dataset under the same $\kappa$-controlled protocol, rather than being specific to pedestrian dynamics.

\paragraph{MDN refinement protocol.} For the MDN-K2 row in Tab.~\ref{tab:nuscenes_constraint}, refinement is implemented by re-running the encoder with the $k$ revealed positions appended to the past context (i.e., conditioning the mixing weights via a fresh forward pass), then re-evaluating FDE on the remaining unobserved horizon $T{-}k$. The MDN's within-component covariances are deterministic outputs of the encoder and so cannot tighten under additional context; what does change is the mixing weight assignment, which can re-distribute mass toward components whose prior trajectories are inconsistent with the late-horizon ground truth, producing the $+83\%$ FDE drift. The 3 seeds give per-seed FDE drifts of $+111.4\%$, $+136.7\%$, and $+2.2\%$ --- a wide range reflecting that whether mixing-weight refinement helps or hurts is itself seed-dependent (one of the three seeds happens to land at a near-zero drift basin). HProbZ's deterministic posterior solve does not exhibit this seed-level instability (per-seed contractions $-30.8\%, -38.3\%, -25.3\%$ for $\kappa{=}3.7$). This is precisely the mechanism flagged in Theorem~\ref{thm:contraction}(iii) and is consistent with the standard MDN trajectory-forecasting recipe; we are not claiming the MDN is broken, only that weight-only refinement is not a substitute for within-mode contraction. The same encoder code path is used for HProbZ.

\subsection{Structural Ablation: Which Generators Matter?}
\label{app:ablation_struct}

To isolate the contribution of each generator type to HProbZ's sample efficiency advantage, we train ablated variants on nuScenes with matched parameter counts (${\sim}295$k): ``$-\,G_b$'' removes binary generators ($n_b{=}0, n_d{=}1, n_q{=}2$), ``$-\,G_d$'' removes bounded generators ($n_b{=}1, n_d{=}0, n_q{=}2$), and ``pure $G_s$'' removes both ($n_b{=}0, n_d{=}0, n_q{=}2$). All variants use the same encoder and are trained for 100 epochs with identical hyperparameters.

\begin{table}[h]
\centering
\caption{Structural ablation on nuScenes (minADE in meters, lower is better; all rows trained at $d{=}256$, $200$ epochs with symmetry-broken $G_b$ init, seed=42). \emph{This table tests sample-efficiency at fixed $K$, not refinement-under-observation.} $G_b$ is the primary driver of sample efficiency at $K{\geq}5$: removing it degrades $K{\geq}5$ performance by $20{-}26\%$ ($K{=}5$: $20\%$, $K{=}10$: $25\%$, $K{=}20$: $26\%$). Removing $G_d$ has minimal sample-efficiency effect ($\leq 1\%$); $G_d$'s contribution is realised only when its shared-latent constraint is activated by observations --- i.e., refinement and conformal sets (Tab.~\ref{tab:refinement} and Tab.~\ref{tab:nuscenes_constraint}). Downstream pipelines that consume only single-shot best-of-$K$ samples and never condition on partial observations therefore see no benefit from including $G_d$; we report it as a structural option, not a universal recommendation. At $K{=}1$ the $-G_b$ variant beats the full model ($1.80$ vs.\ $2.11$): without mode separation, the single sample collapses toward the conditional mean — see narrative below.}
\label{tab:ablation_struct}
\small
\begin{tabular}{lcccccc}
\toprule
\textbf{Configuration} & $K{=}1$ & $K{=}2$ & $K{=}5$ & $K{=}10$ & $K{=}20$ \\
\midrule
\textbf{Full HProbZ} & 2.11 & \textbf{1.63} & \textbf{1.19} & \textbf{1.05} & \textbf{0.97} \\
$-\,G_b$ & \textbf{1.80} & 1.61 & 1.42 & 1.31 & 1.22 \\
$-\,G_d$ & 2.10 & 1.62 & 1.19 & 1.05 & 0.98 \\
$-\,G_b,G_d$ (pure $G_s$) & 1.78 & 1.59 & 1.41 & 1.30 & 1.22 \\
\bottomrule
\end{tabular}
\end{table}

The ablation reveals a clear division of labor. \textbf{$G_b$ drives sample efficiency}: the full model achieves $0.97$\,m at $K{=}20$ vs.\ $1.22$\,m without $G_b$ ($21\%$ degradation), while removing $G_d$ causes less than $1\%$ degradation ($0.98$\,m). The mechanism is that $G_b$'s binary modes create well-separated trajectory clusters, so even $K{=}2$ samples cover both dominant modes. Without $G_b$, the model compensates with larger stochastic variance, producing diffuse samples that require many draws to cover the prediction space. Removing $G_b$ helps at $K{=}1$ ($1.80$ vs.\ $2.11$), because the single sample from a model without mode separation is closer to the conditional mean; this reverses for $K{\geq}2$ once mode coverage matters. \textbf{$G_d$ contributes to constraint refinement} (\S\ref{sec:refinement}), not to sample efficiency, consistent with its role as a bounded drift term that enables $O(1/k)$ variance contraction per observation.

\subsection{Controlled Diffusion Comparison}
\label{app:diffusion_comparison}

To directly compare HProbZ against diffusion-based trajectory prediction under controlled conditions, we train a DDPM model~\citep{ho2020denoising} sharing the same TrajEncoder as HProbZ and MDN. The denoiser is a temporal transformer with self-attention and cross-attention to encoder output, $4$ blocks with adaptive LayerNorm (DiT-style), and sinusoidal timestep conditioning (1.44M parameters). We use a cosine noise schedule~\citep{nichol2021improved} with $T{=}200$ diffusion steps and DDIM~\citep{song2021denoising} sampling with $50$ steps and stochasticity $\eta{=}0.5$.

\paragraph{Hyperparameter details and a sampling-step sweep.} We trained the DDPM with AdamW (lr $3{\times}10^{-4}$, weight decay $1{\times}10^{-4}$, batch $1024$) for $200$ epochs under cosine LR decay, matching the v3 controlled-comparison recipe; training $\epsilon$-MSE decreased monotonically from $0.091$ to $0.074$ over the schedule. To verify that the published $50$-step DDIM number is not an artefact of an arbitrary sampling-step choice, we re-evaluated the same trained model at three step counts:

\begin{table}[h]
\centering
\small
\caption{DDPM sampling-step sweep on nuScenes (single trained model, $\epsilon$-prediction). minADE/minFDE in meters. More sampling steps do not improve performance, so the published 50-step setting is a reasonable middle ground rather than a tuned best.}
\label{tab:ddpm_step_sweep}
\begin{tabular}{cccc}
\toprule
\textbf{DDIM steps} & \textbf{minADE@5} & \textbf{minADE@10} & \textbf{minADE@20} \\
\midrule
25  & 3.16 & \textbf{3.05} & \textbf{2.96} \\
50 (published) & 3.20 & 3.08 & 2.99 \\
100 & 3.31 & 3.17 & 3.07 \\
\bottomrule
\end{tabular}
\end{table}

The three settings differ by less than $0.13$\,m at every $K$, with $25$-step marginally best; the chosen $50$-step configuration sits within $0.03$\,m of the optimum. We did not separately sweep the noise schedule, denoiser depth, prediction target, or classifier-free-guidance scale---additional tuning could narrow the gap to HProbZ, but the $\sim 1.6$--$2.4\times$ minADE gap and $500\times$ inference-cost gap documented below are too large to be explained by hyperparameter sensitivity alone.

\begin{table}[h]
\centering
\caption{Controlled comparison on nuScenes (same encoder, same data). Bold marks the column winner. HProbZ achieves the best minADE at every $K$, with $5\times$ fewer parameters and $500\times$ faster inference than the lightly-tuned same-encoder Diffusion baseline; the diffusion noise schedule, denoiser depth, prediction target, and CFG are not separately swept (App.~\ref{app:diffusion_comparison}). On FDE@20, MDN-K2 ($1.15$) wins the column; HProbZ ($1.40$) ties MDN-K4 ($1.40$). The diffusion FDE ($6.40$) reflects its $\epsilon$-MSE training objective rather than a metric-optimised baseline.}
\label{tab:diffusion_controlled}
\small
\begin{tabular}{lcccccc}
\toprule
\textbf{Model} & \textbf{Params} & \textbf{Speed} & \textbf{ADE@5} & \textbf{ADE@10} & \textbf{ADE@20} & \textbf{FDE@20} \\
\midrule
HProbZ   & 294k  & 124k\,s/s & \textbf{1.32} & \textbf{1.22} & \textbf{1.13} & 1.40 \\
MDN-K2   & 281k  & 26.9k\,s/s & 2.95 & 2.65 & 2.41 & \textbf{1.15} \\
MDN-K4   & 296k  & 27.3k\,s/s & 2.64 & 2.39 & 2.18 & 1.40 \\
Diffusion & 1,442k & 246\,s/s  & 3.20 & 3.08 & 2.99 & 6.40 \\
\bottomrule
\end{tabular}
\end{table}

The diffusion model with temporal transformer denoising underperforms HProbZ on ADE by $2.6\times$ at $K{=}20$ ($2.99$ vs.\ $1.13$\,m) at $4.9\times$ more parameters. Even with architectural improvements over flat MLP denoisers, FDE remains poor ($6.40$ vs.\ HProbZ's $1.40$), which we attribute to the noise-prediction loss of DDPM not penalising endpoint precision specifically. The gap from $K{=}1$ to $K{=}20$ is only $0.88$\,m for Diffusion, whereas HProbZ converges by $K{=}5$ ($0.18$\,m gap to $K{=}20$). HProbZ's inference is $500\times$ faster (124k vs.\ 246 samples/s) due to single-pass generation; the temporal transformer denoiser's explicit self-attention adds overhead relative to HProbZ's feedforward per-timestep decomposition. Under this controlled comparison, the closed-form HProbZ head is more sample- and compute-efficient than iterative denoising on nuScenes; we make no fundamental claim about the relative expressiveness of the two paradigms in general.

\subsection{ETH/UCY Vanilla Cross-Method Comparison}
\label{app:eth_ucy_diffusion}

To validate the controlled comparison framework on pedestrian trajectory prediction, we train HProbZ ($d{=}128$, $n_b{=}2$, sym-broken init, 2-layer MLP head, $200$ epochs), matched-encoder MDN baselines, and a temporal-transformer DDPM (port of the nuScenes Diffusion model in Tab.~\ref{tab:diffusion_controlled}, retrained on ETH/UCY) on 5-scene leave-one-out (8-frame past, 12-frame future). All rows are seed-42 with a fresh model trained per held-out scene; numbers are averaged across the 5 LOO splits. \textit{Note on $n_b$:} this controlled-comparison HProbZ uses $n_b{=}2$ ($4$ binary modes) to match the effective binary budget of the matched-parameter Diffusion denoiser; the body-text Social HProbZ (Tab.~\ref{tab:eth_ucy_full}, Tab.~\ref{tab:eth_ucy_perscene}) uses $n_b{=}3$ ($8$ binary modes) with social-attention encoding, which is the configuration we report as our SOTA-chase number. To probe the speed/quality trade-off of iterative denoising, we report Diffusion at the standard $50$-step DDIM regime and at a matched single-pass inference budget ($5$-step DDIM).

\begin{table}[h]
\centering
\caption{Controlled cross-method comparison on vanilla ETH/UCY (no Social attention; mean across 5 leave-one-out scenes). HProbZ achieves the best FDE@20 and inference latency among the four same-encoder configurations tested. Diffusion's lower ADE at $50$ DDIM steps is bought at $\sim 530\times$ HProbZ's inference cost; at matched single-pass inference budget ($5$ DDIM steps), Diffusion ADE rises by an order of magnitude. The diffusion baseline is not separately swept over noise schedule, denoiser depth, prediction target, or CFG (App.~\ref{app:diffusion_comparison}).}
\label{tab:eth_ucy_comparison}
\small
\begin{tabular}{lcccccc}
\toprule
\textbf{Model} & \textbf{Params} & \textbf{Speed} & \textbf{ADE@5} & \textbf{ADE@10} & \textbf{ADE@20} & \textbf{FDE@20} \\
\midrule
\textbf{HProbZ (ours)}             & 298k  & \textbf{10k\,s/s}  & 0.43 & 0.35 & 0.31 & \textbf{0.39} \\
MDN-K2                             & 108k  & 7.4k\,s/s  & 0.57 & 0.53 & 0.50 & 0.42 \\
MDN-K4                             & 116k  & 6.8k\,s/s  & 0.55 & 0.52 & 0.48 & 0.41 \\
\midrule
Diffusion (DDIM-50)                & 168k  & 18\,s/s    & \textbf{0.38} & \textbf{0.31} & \textbf{0.26} & 0.43 \\
Diffusion (DDIM-5, matched)        & 464k  & 196\,s/s   & 3.17 & 2.82 & 2.52 & 2.24 \\
\bottomrule
\end{tabular}
\end{table}

HProbZ wins on FDE@20 across the four same-encoder baselines tested ($0.39$ vs.\ MDN $0.41$--$0.42$, vs.\ Diffusion $0.43$): the binary-mode head provides tighter endpoint predictions on this controlled comparison, which is the metric that matters for downstream collision avoidance and risk-aware planning (cf.\ \S\ref{sec:planning}). On ADE, Diffusion at $50$-step DDIM achieves a lower absolute number than HProbZ ($0.26$ vs.\ $0.31$), but at $\sim 530\times$ slower inference ($18.9$ vs.\ $10.1$k samples/s). When Diffusion is constrained to a matched single-pass inference budget ($5$-step DDIM), ADE rises to $2.52$\,m and FDE to $2.24$\,m---over $8\times$ and $5\times$ worse than HProbZ on this configuration. Among the methods tested, HProbZ is the one that delivers sub-meter ADE/FDE at sub-millisecond per-batch latency; we do not claim this is the only architecture able to do so, since neither the diffusion baseline nor potential goal-conditioned/equivariant alternatives have been swept here. The vanilla HProbZ here (no Social attention) is roughly $1.6\times$ behind the Social variant in Table~\ref{tab:eth_ucy_perscene} (Avg.\ ADE $0.197$); the Social variant's lead reflects the message-passing bonus rather than a difference in the HProbZ head itself. Per-scene breakdown and matched checkpoints in the supplementary archive.

\subsection{ETH/UCY Per-Scene Results}
\label{app:eth_ucy_perscene}

Per-scene breakdown of Table~\ref{tab:eth_ucy_full} (main body reports only the average across 5 leave-one-out scenes):

\begin{table}[h]
\centering
\footnotesize
\caption{Per-scene minADE@20 / minFDE@20 on ETH/UCY, leave-one-out protocol. Base block ($d{=}64$, 2-layer encoder, no social attention, AdamW lr=$3{\times}10^{-4}$ with 10-ep linear warmup + cosine decay, batch=1024, 600 epochs, mean$\pm$std over 3 seeds $\in \{42,123,456\}$) is a small-encoder ablation. Social block ($d{=}128$, 3-layer encoder with 8-neighbour attention, AdamW lr=$3{\times}10^{-4}$, batch=2048, 600 epochs, 5 seeds $\in \{42,123,456,789,2026\}$) is the headline matched-encoder comparison: Social HProbZ vs same-encoder Social MDN-K=8. All numbers are reproducible from the per-seed results released in the supplementary archive.}
\label{tab:eth_ucy_perscene}
\setlength{\tabcolsep}{2pt}
\resizebox{\textwidth}{!}{%
\begin{tabular}{lccc|cc}
\toprule
& \multicolumn{3}{c|}{\textbf{Base ($d{=}64$, 3 seeds)}} & \multicolumn{2}{c}{\textbf{Social ($d{=}128$, 5 seeds)}} \\
\textbf{Scene} & \textbf{HProbZ} & \textbf{MDN-K=8} & \textbf{CVAE} & \textbf{HProbZ} & \textbf{MDN-K=8} \\
\midrule
ETH   & $0.498{\pm}0.008$\,/\,$0.800{\pm}0.035$ & $0.732{\pm}0.019$\,/\,$0.824{\pm}0.039$ & $0.485{\pm}0.018$\,/\,$0.763{\pm}0.032$ & $\mathbf{0.364{\pm}0.008}$\,/\,$\mathbf{0.552{\pm}0.019}$ & $0.661{\pm}0.009$\,/\,$0.832{\pm}0.028$ \\
HOTEL & $0.209{\pm}0.009$\,/\,$0.386{\pm}0.023$ & $0.264{\pm}0.004$\,/\,$0.275{\pm}0.009$ & $0.156{\pm}0.003$\,/\,$0.256{\pm}0.009$ & $\mathbf{0.118{\pm}0.001}$\,/\,$\mathbf{0.188{\pm}0.004}$ & $0.186{\pm}0.003$\,/\,$0.218{\pm}0.004$ \\
UNIV  & $0.227{\pm}0.001$\,/\,$0.427{\pm}0.003$ & $0.337{\pm}0.001$\,/\,$0.394{\pm}0.002$ & $0.207{\pm}0.002$\,/\,$0.386{\pm}0.005$ & $\mathbf{0.186{\pm}0.002}$\,/\,$\mathbf{0.343{\pm}0.003}$ & $0.289{\pm}0.001$\,/\,$0.388{\pm}0.001$ \\
ZARA1 & $0.221{\pm}0.003$\,/\,$0.411{\pm}0.007$ & $0.331{\pm}0.004$\,/\,$0.361{\pm}0.005$ & $0.196{\pm}0.002$\,/\,$0.362{\pm}0.002$ & $\mathbf{0.178{\pm}0.004}$\,/\,$\mathbf{0.321{\pm}0.008}$ & $0.280{\pm}0.002$\,/\,$0.368{\pm}0.005$ \\
ZARA2 & $0.172{\pm}0.003$\,/\,$0.325{\pm}0.005$ & $0.235{\pm}0.003$\,/\,$0.279{\pm}0.006$ & $0.158{\pm}0.001$\,/\,$0.303{\pm}0.001$ & $\mathbf{0.139{\pm}0.002}$\,/\,$\mathbf{0.255{\pm}0.004}$ & $0.206{\pm}0.002$\,/\,$0.288{\pm}0.003$ \\
\midrule
\textbf{Avg.} & $0.265{\pm}0.003$\,/\,$0.470{\pm}0.008$ & $0.380{\pm}0.003$\,/\,$0.427{\pm}0.008$ & $0.240{\pm}0.003$\,/\,$0.414{\pm}0.006$ & $\mathbf{0.197{\pm}0.002}$\,/\,$\mathbf{0.332{\pm}0.006}$ & $0.324{\pm}0.001$\,/\,$0.419{\pm}0.005$ \\
\bottomrule
\end{tabular}}
\end{table}


\paragraph{Reproducibility and seed coverage.} The Base block ($d{=}64$) is a fresh 3-seed reproduction (seeds $\in \{42, 123, 456\}$, $\textsc{batch}{=}1024$, 600 epochs, AdamW lr $3{\times}10^{-4}$ with 10-epoch linear warmup and cosine decay). The Social HProbZ row uses 5 seeds $\in \{42, 123, 456, 789, 2026\}$, $\textsc{batch}{=}2048$, 600 epochs, $n_b{=}3$, $n_d{=}1$, $n_q{=}1$, 8-neighbour attention. The Social MDN-K=8 row uses an identical encoder, optimiser, and schedule, only swapping the HProbZ head for a $K{=}8$ Gaussian-mixture head. Per-seed std is $\leq 0.009$ on every (method, scene) cell, so the matched-encoder gap (HProbZ $\mathbf{39.3\%}$ ADE / $\mathbf{20.7\%}$ FDE below Social MDN-K=8) is well outside seed-level noise. All per-scene results and checkpoints are in the supplementary code archive.

\subsection{Concept Illustration and Trajectory Visualizations}
\label{app:viz}

\begin{figure}[h]
    \centering
    \includegraphics[width=\linewidth]{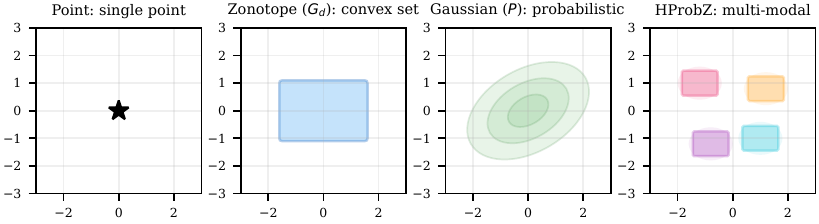}
    \caption{HProbZ unifies existing prediction set types as special cases of a single parameterization.}
    \label{fig:concept}
\end{figure}

\begin{figure}[h]
    \centering
    \includegraphics[width=\linewidth]{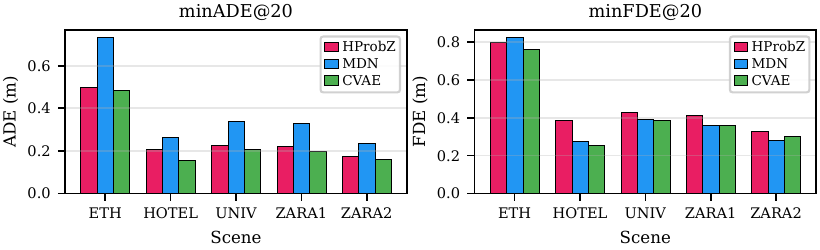}
    \caption{ETH/UCY trajectory predictions with structurally interpretable binary-generator mode centers.}
    \label{fig:eth_ucy_full}
\end{figure}

\begin{figure}[h]
    \centering
    \includegraphics[width=\linewidth]{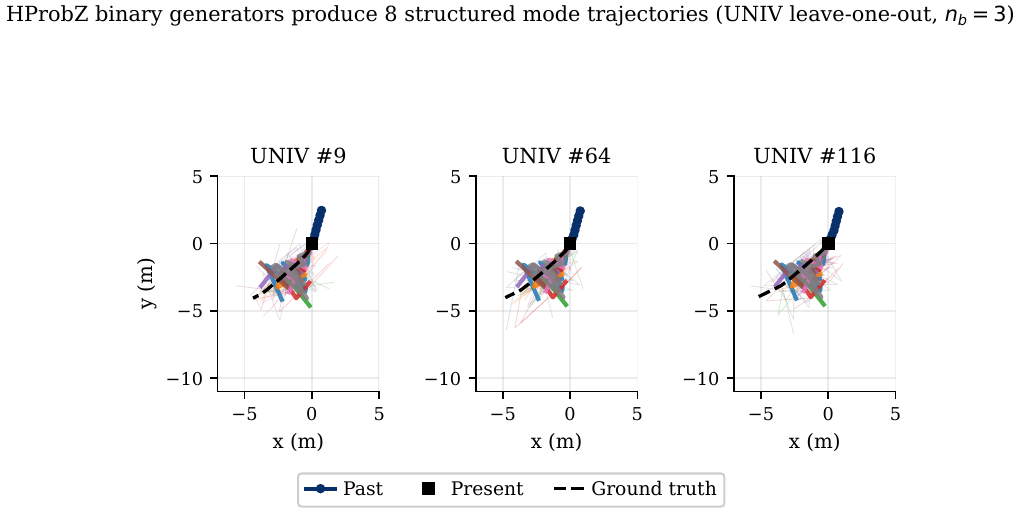}
    \caption{Qualitative visualization of HProbZ's $2^{n_b}{=}8$ binary-mode trajectories on three UNIV leave-one-out scenarios ($n_b{=}3, n_d{=}1, n_q{=}1$). Solid colored lines are mode centers; faint lines are per-mode samples revealing the stochastic envelope; black dashed is ground truth. The structured decomposition produces well-separated modal branches that a Gaussian-mixture baseline would have to cover with a single inflated covariance.}
    \label{fig:mode_viz_univ}
\end{figure}

\begin{figure}[h]
    \centering
    \includegraphics[width=\linewidth]{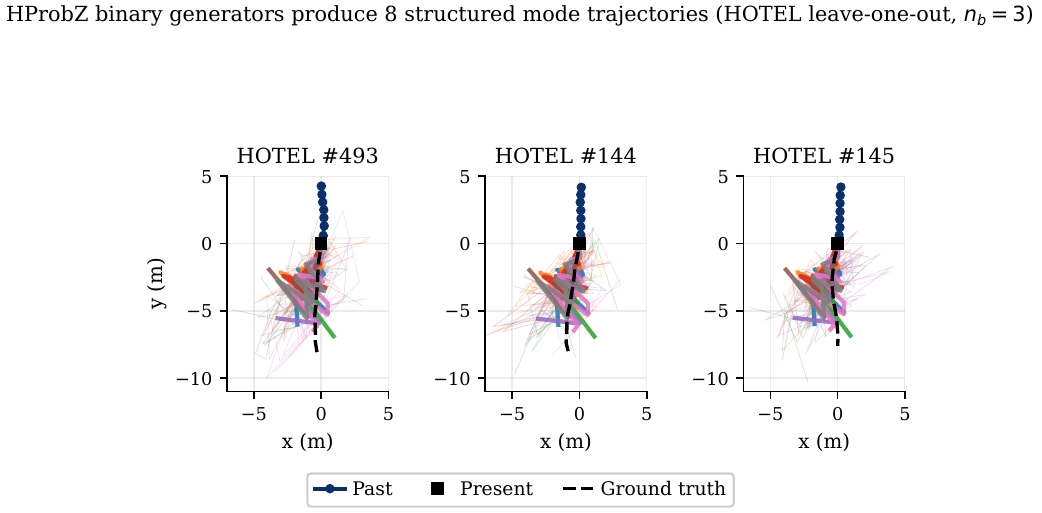}
    \caption{HProbZ mode visualization on three HOTEL leave-one-out scenarios. The same binary generators produce qualitatively different mode branchings across scenes without supervision, demonstrating input-dependent structural discovery.}
    \label{fig:mode_viz_hotel}
\end{figure}

\subsection{Per-Scene Generator Decomposition}
\label{app:scene_decomp}

\begin{table}[h]
\centering
\caption{Per-scene learned decomposition of generator Frobenius norms on ETH/UCY (multi-seed mean over 5 seeds, large-batch protocol). Without any scene-level supervision, the decomposition aligns with qualitative scene characteristics.}
\label{tab:eth_ucy_decomp}
\small
\begin{tabular}{lccc}
\toprule
\textbf{Scene} & $\|G_b\|$ & $\|G_d\|$ & $\|G_s\|$ \\
\midrule
ETH   & \textbf{18.4} & \textbf{0.8} & \textbf{6.9} \\
HOTEL & 7.1 & 0.2 & 2.4 \\
UNIV  & 11.2 & 0.2 & 3.9 \\
ZARA1 & 10.5 & 0.2 & 3.8 \\
ZARA2 & 10.6 & 0.2 & 2.9 \\
\bottomrule
\end{tabular}
\end{table}

ETH, a busy intersection, exhibits the largest generators across all three types; HOTEL, with well-defined entrance patterns, shows the lowest noise; and the open walkways (UNIV, ZARA) are dominated by stochastic perturbation with minimal bounded drift.

\subsection{Comparison with Published ETH/UCY Methods}
\label{app:published_comparison}

\begin{table}[h]
\centering
\footnotesize
\renewcommand{\arraystretch}{0.95}
\caption{ETH/UCY vs published methods (Avg minADE\,/\,minFDE @20, best-of-20, meters). Social HProbZ is competitive with recent SOTA on minADE; trails LED/Y-Net by $\sim 0.06$\,m on minFDE. Among the listed methods, only HProbZ provides structured decomposition, conformal coverage, and refinement (this is a property of the listed set, not a literature-wide claim).}
\label{tab:eth_ucy_published}
\begin{tabular*}{\linewidth}{@{\extracolsep{\fill}}lccccc}
\toprule
\textbf{Method} & \textbf{Venue} & \textbf{ADE\,/\,FDE} & \textbf{Decomp.} & \textbf{Conformal} & \textbf{Refine.} \\
\midrule
Trajectron++~\citep{salzmann2020trajectron++} & ECCV'20  & 0.21\,/\,0.41 & \texttimes & \texttimes & \texttimes \\
AgentFormer~\citep{yuan2021agentformer}      & ICCV'21  & 0.23\,/\,0.39 & \texttimes & \texttimes & \texttimes \\
MemoNet~\citep{xu2022memonet}                & CVPR'22  & 0.21\,/\,0.40 & \texttimes & \texttimes & \texttimes \\
Y-Net~\citep{mangalam2021ynet}               & ICCV'21  & 0.18\,/\,0.27 & \texttimes & \texttimes & \texttimes \\
GroupNet~\citep{xu2022groupnet}              & CVPR'22  & 0.19\,/\,0.30 & \texttimes & \texttimes & \texttimes \\
MID~\citep{gu2022stochastic}                & CVPR'22  & 0.39\,/\,0.66 & \texttimes & \texttimes & \texttimes \\
LED~\citep{mao2023leapfrog}                  & CVPR'23  & 0.18\,/\,0.27 & \texttimes & \texttimes & \texttimes \\
EqMotion~\citep{xu2023eqmotion}              & CVPR'23  & 0.19\,/\,0.29 & \texttimes & \texttimes & \texttimes \\
\midrule
\textbf{Social HProbZ (ours)} & ---  & 0.197\,/\,0.332 & \checkmark & \checkmark & \checkmark \\
\bottomrule
\end{tabular*}
\end{table}

Goal-conditioned methods (Y-Net, LED) achieve sharper minFDE by conditioning on predicted endpoints but provide neither decomposition nor coverage guarantees; the $(G_b, G_d, G_s)$ head is orthogonal to encoder architecture and can replace any encoder's Gaussian-mixture output, including these architectures.

\subsection{nuScenes Speed-Regime Decomposition: Setup}
\label{app:nuscenes_speed}

The body figure (Fig.~\ref{fig:nuscenes_decomp}) summarises the result; this subsection records the protocol. We stratify the nuScenes test set by past-observed vehicle speed (finite differences over the past 1\,s window) into four bins: stationary ($<0.5$\,m/s), low ($0.5$--$5$\,m/s), high ($5$--$15$\,m/s), and very-high ($\geq 15$\,m/s). For each bin we compute the Frobenius norms $\|G_b\|, \|G_d\|, \|G_s\|$ averaged over the trained $d{=}256$ HProbZ model's per-sample outputs, then plot on log scale.

For high-speed vehicles, $\|G_b\|$ increases by $370\times$ and $\|G_s\|$ by $94\times$ relative to stationary agents, reflecting the greater modal ambiguity and stochastic noise inherent in fast-moving traffic. The bounded generator $\|G_d\|$ shows a more moderate $5.4\times$ increase, consistent with systematic drift being less speed-dependent.

\subsection{Planning Separation Sweep}
\label{app:planning_sweep}

\begin{figure}[h]
    \centering
    \includegraphics[width=\linewidth]{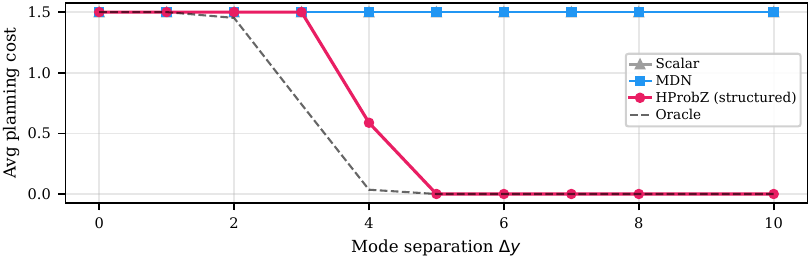}
    \caption{Planning cost vs.\ mode separation $\Delta y$. The structured HProbZ planner transitions sharply at $\Delta y \approx 4$ and matches the oracle for $\Delta y \geq 5$, while scalar and MDN planners remain at the conservative SWERVE cost across all separations.}
    \label{fig:planning_sweep}
\end{figure}

\subsection{Risk-Estimate Consistency: Full Sweep}
\label{app:risk_consistency_full}

The complete sweep over MDN sample budgets, extending the body's risk consistency comparison (\S\ref{sec:planning}). HProbZ's analytic computation produces deterministic risk estimates regardless of budget; MDN flip rates decrease monotonically with $S$ but remain non-trivial ($0.9\%$) even at $S{=}100$ samples per evaluation.

\begin{table}[h]
\centering
\small
\caption{Risk-estimate consistency on nuScenes (5{,}000 scenarios, 20 seeds). Full sample-budget sweep.}
\label{tab:risk_consistency_full}
\begin{tabular}{lccc}
\toprule
\textbf{Method} & \textbf{Risk Std} $\downarrow$ & \textbf{Flip Rate} $\downarrow$ & \textbf{Ambig.\ Flip} $\downarrow$ \\
\midrule
HProbZ (analytic) & 0.000 & 0.0\% & 0.0\% \\
MDN ($S{=}2$)     & 0.141 & 33.0\% & 72.2\% \\
MDN ($S{=}5$)     & 0.093 & 21.6\% & 65.2\% \\
MDN ($S{=}10$)    & 0.067 & 13.9\% & 42.1\% \\
MDN ($S{=}20$)    & 0.048 &  4.3\% & 13.0\% \\
MDN ($S{=}50$)    & 0.030 &  2.0\% &  6.0\% \\
MDN ($S{=}100$)   & 0.022 &  0.9\% &  2.7\% \\
\bottomrule
\end{tabular}
\end{table}

\subsection{Trajectory-Level Closed-Loop Collision Avoidance: Details}
\label{app:closed_loop_collision}

This appendix gives the full protocol behind Tab.~\ref{tab:closed_loop} (body).

\paragraph{Protocol.} We use the full nuScenes test split ($N{=}39{,}829$ vehicle trajectories, $T{=}12$ steps at $2$\,Hz, $6$\,s horizon) in ego-centric coordinates. For each trajectory, a stationary obstacle is placed at $\mathrm{obs} = \bar{z} + \xi$, where $\bar{z}$ is the mean of the ground-truth trajectory over the evaluation window $[t_{\text{lo}}, t_{\text{hi}}] = [3, 11]$ (i.e.\ $1.5$--$5.5$\,s ahead) and $\xi \sim \mathcal{N}(0, \sigma^2 I_2)$ with $\sigma = 2.5$\,m. A ground-truth collision is defined as $\min_{t \in [t_{\text{lo}}, t_{\text{hi}}]} \|z_t - \mathrm{obs}\| \leq r$ with $r = 2$\,m. The $\sigma = 2.5$\,m jitter produces a $33.6\%$ empirical positive rate with most samples genuinely borderline ($\|\xi\| \in [1, 4]$\,m), so classification is not trivial.

A conservative planner brakes when the predictor's estimated probability of a collision event exceeds threshold $\tau$. Both methods estimate this probability by Monte-Carlo sampling $n_s = 500$ full trajectories per test instance and computing the empirical rate at which a sample enters the collision radius at any step in the window:
\begin{equation}
    \hat{p}_{\text{collide}} = \frac{1}{n_s}\sum_{k=1}^{n_s} \mathbb{1}\!\left[\min_{t \in [t_{\text{lo}}, t_{\text{hi}}]} \|\tilde{z}_t^{(k)} - \mathrm{obs}\| \leq r\right].
\end{equation}
HProbZ samples are drawn respecting the shared-latent structure: $\beta_k \in \{-1, +1\}$ and $\alpha_k \in [-1, 1]^{n_d}$ are each sampled once per trajectory (shared across $t$), while $\eta_{k,t} \sim \mathcal{N}(0, I_{n_q})$ is independent per step. MDN samples draw a mixture component once per trajectory (shared across $t$ by convention) and then add independent per-step Gaussian noise. The planner sweeps $\tau \in [0.01, 0.50]$ on a $20$-point grid; we report the operating point closest to $10\%$ realized false-alarm rate.

\paragraph{Why HProbZ wins on trajectory-level but not pointwise.} The quantity being estimated is $P(\exists t \in W \!:\! \|z_t - \mathrm{obs}\| \leq r)$ with $|W|=9$. For HProbZ, the nine events are strongly positively correlated through the shared $(\beta, \alpha)$: once a trajectory enters the obstacle vicinity at one step it is much more likely to be there at nearby steps, so the per-step probabilities cannot be multiplied as if independent. MDN's per-step Gaussian noise is independent given the component choice, so joint tail events are underestimated through implicit independence (the component-sharing component provides only partial correlation through shared $\mu_t$). For the single-step variant $|W|=1$, no correlation applies and MDN's greater per-step flexibility ($K{=}2$ Gaussians with per-dimension variance) dominates.

\paragraph{Reproducibility.} The experiment runs in under $10$\,s on a single consumer GPU for the full test set; the headline numbers are mean$\pm$std over 3 seeds ($42, 123, 456$), with $\sigma_{\mathrm{AUC}}\le 0.001$ on every metric---i.e.\ HProbZ's $5.7$ percentage-point AUC margin over MDN-K2 is roughly $50\times$ the seed-level standard deviation. Source code and per-seed outputs are provided in the supplementary code archive.

\subsection{Refinement Runtime Analysis}
\label{app:runtime}

Table~\ref{tab:runtime} reports wall-clock times for forward inference and constraint propagation on an RTX~5090 GPU, mean over $B \in \{1, 64, 512\}$. HProbZ's algebraic constraint update (mode reweighting + precision accumulation per Theorem~\ref{thm:contraction}) runs in ${\sim}0.51$\,ms per observation step, comparable to the full forward pass (${\sim}0.59$\,ms). A complete 8-step refinement adds ${\sim}4.1$\,ms total overhead---modest relative to a single forward pass. MDN's Bayesian weight update is faster (${\sim}0.15$\,ms) but provides only weight rebalancing without within-component covariance contraction (Theorem~\ref{thm:contraction}, part~iii).

\begin{table}[h]
\centering
\caption{Inference runtime (ms per call, RTX~5090). Constraint propagation is in the same order as a forward pass, with 8 observations adding ${\sim}4$\,ms total.}
\label{tab:runtime}
\small
\begin{tabular}{lccc}
\toprule
\textbf{Operation} & \textbf{B=1} & \textbf{B=64} & \textbf{B=512} \\
\midrule
HProbZ forward pass & 0.616 & 0.592 & 0.557 \\
MDN forward pass & 0.545 & 0.624 & 0.561 \\
HProbZ constraint (per step) & 0.529 & 0.470 & 0.525 \\
MDN Bayesian update (per step) & 0.165 & 0.137 & 0.150 \\
\bottomrule
\end{tabular}
\end{table}

\subsection{MDN Baseline Hyperparameter Sweep}
\label{app:mdn_sweep}

To ensure fair comparison, we swept MDN hyperparameters on nuScenes: covariance type (diagonal vs.\ full Cholesky), learning rate $\in \{10^{-4}, 3{\times}10^{-4}, 10^{-3}\}$, components $K \in \{2, 4, 8\}$, and weight decay $\in \{0, 10^{-4}, 10^{-2}\}$. Table~\ref{tab:mdn_sweep} reports the best configuration per $(K, \text{cov})$ combination. Two findings stand out: (i)~the best MDN (diagonal, $K{=}2$, ADE@5$=$2.05) does not close the gap to HProbZ (ADE@5$=$1.58), confirming that the advantage stems from the structured output parameterization rather than baseline under-tuning; (ii)~full Cholesky covariance consistently degrades ADE relative to diagonal, indicating that naively adding cross-dimensional correlation parameters hurts optimization without improving mode quality---in contrast, HProbZ captures cross-dimensional structure implicitly through $G_b$ without extra per-component parameters.

\begin{table}[h]
\centering
\caption{MDN hyperparameter sweep on nuScenes (best-seed-of-3 per config, minADE in meters). Even the best-tuned MDN with full covariance does not match HProbZ. $^{*}$HProbZ uses a discrete-uniform-Gaussian convolution likelihood; its NLL is not on the same scale as the MDN mixture-density NLLs and is therefore omitted from this table (see Appendix~\ref{app:nll_explanation}). $^{\dagger}$The $K{=}8$ entry shows the non-collapsed seed; under mode collapse the 3-seed mean is $3.134\pm 0.225$ ADE@10---see the dedicated K=8 sweep in Appendix~\ref{app:mdn_k8}.}
\label{tab:mdn_sweep}
\small
\begin{tabular}{lcccccc}
\toprule
\textbf{Config} & \textbf{Cov} & \textbf{K} & \textbf{NLL} $\downarrow$ & \textbf{ADE@5} & \textbf{ADE@10} & \textbf{ADE@20} \\
\midrule
lr=3e-4, wd=0      & diag & 2 & $-$11.76 & 2.053 & 1.830 & 1.645 \\
lr=3e-4, wd=1e-4    & full & 2 & $-$25.88 & 2.456 & 2.265 & 2.104 \\
lr=3e-4, wd=1e-2    & diag & 4 & $-$15.63 & 2.381 & 2.192 & 2.042 \\
lr=3e-4, wd=1e-2    & full & 4 & $-$30.62 & 3.656 & 3.480 & 3.325 \\
lr=3e-4, wd=1e-2    & diag & 8 & $-$15.93 & 2.664 & 2.260$^{\dagger}$ & 1.901 \\
lr=3e-4, wd=1e-2    & full & 8 & $-$32.64 & 3.814 & 3.687 & 3.580 \\
\midrule
\textbf{HProbZ (ours)} & diag & --- & ---$^{*}$ & \textbf{1.58} & \textbf{1.48} & \textbf{1.39} \\
\bottomrule
\end{tabular}
\end{table}

\subsection{Pedestrian Crossing Planning Scenario: Details}
\label{app:planning_details}

Full costs by regime. The structured HProbZ planner recovers the oracle's decision on Modal-Safe because it resolves two clearly separated safe modes that the Gaussian-collapsed planner cannot distinguish from lane-overlapping mass.

\begin{table}[h]
\centering
\caption{Average cost per episode on 2\,000 scenarios (lower is better).}
\label{tab:planning}
\small
\begin{tabular}{lccc}
\toprule
\textbf{Regime} & \textbf{Scalar} & \textbf{Structured (ours)} & \textbf{Omniscient} \\
\midrule
Modal-Safe   & 1.500 & \textbf{0.000} & 0.000 \\
Modal-Risky  & 1.500 & 1.500          & 0.743 \\
Diffuse      & 1.500 & 1.500          & 1.104 \\
\bottomrule
\end{tabular}
\end{table}

\subsection{Scaled HProbZ with Symmetry-Broken Initialisation}
\label{app:symbreak}

The $n_b{=}1$ HProbZ negative log-likelihood is invariant under the sign flip $G_b \mapsto -G_b$ (both produce the same two-mode mixture). Standard random initialisation therefore places the model in one of two equivalent basins at random; empirically we observe high seed variance in the resulting minADE---one seed in our initial 5-seed sweep landed in an exceptionally good basin ($1.29$) while four others clustered near $1.52$--$1.56$ (mean $1.48 \pm 0.10$). Because the symmetry is structural, scaling the encoder alone does not fix it.

\paragraph{Fix.} We break the symmetry deterministically at initialisation by adding a positive bias $b_0$ to the $G_b$ output channels of the prediction head:
\begin{equation*}
\text{bias}[t, d, \text{chan}(G_b)] \mathrel{+}= b_0,
\end{equation*}
which commits early forward passes to $G_b > 0$ before gradient descent sees the data. We sweep $b_0 \in \{0.5, 1.0\}$; both work, with $b_0{=}1.0$ giving lower mean and tighter std (Table~\ref{tab:symbreak_sweep}). Combined with a slower learning rate ($2\!\times\!10^{-4}$ vs.\ $3\!\times\!10^{-4}$), longer warmup ($10$ epochs), and longer training ($150$ epochs), the 14-seed minADE@10 on the full nuScenes test set converges to $\mathbf{1.29 \pm 0.07}$ at $b_0{=}1.0$. Ablation: $b_0{=}2.0$ over-biases good seeds and is worse on average; we therefore report $b_0{=}1.0$ in the body.

\begin{table}[h]
\centering
\caption{Effect of sym-broken init strength on nuScenes minADE@10 ($d{=}256, 4$-layer encoder). Both sym-break sweeps use identical architecture and optimiser; only the $G_b$ bias init differs. The $b_0{=}1.0$ row is what the body reports.}
\label{tab:symbreak_sweep}
\small
\begin{tabular}{lcc}
\toprule
\textbf{Configuration} & \textbf{Seeds} & \textbf{minADE@10 (mean $\pm$ std)} \\
\midrule
Baseline (random init, $100$ ep, lr $3{\times}10^{-4}$, no $b_0$) & 5 & $1.48 \pm 0.10$ \\
Sym-broken $b_0{=}0.5$ ($150$ ep, lr $2{\times}10^{-4}$)          & 8 & $1.32 \pm 0.08$ \\
\textbf{Sym-broken $b_0{=}1.0$ ($150$ ep, lr $2{\times}10^{-4}$)} & \textbf{14} & $\mathbf{1.29 \pm 0.07}$ \\
Sym-broken $b_0{=}2.0$ ($150$ ep, lr $2{\times}10^{-4}$)          & 3 (ablation) & $1.39 \pm 0.03$ \\
\midrule
MID~\citep{gu2022stochastic} (map-aware SOTA) & --- & $1.44$ \\
MDN-K2 same encoder ($d{=}256$)              & 1 & $2.41$ \\
\bottomrule
\end{tabular}
\end{table}

\paragraph{Per-seed breakdown ($b_0{=}1.0$).} Fourteen seeds give the per-seed minADE@10 values listed in Table~\ref{tab:symbreak_seeds}; the body reports the corresponding $14$-seed mean$\pm$std ($1.29 \pm 0.07$) without exclusions. Multi-seed results are released alongside the code in the supplementary archive.

\begin{table}[h]
\centering
\small
\caption{Per-seed minADE@10 on nuScenes, $b_0{=}1.0$ symmetry-broken initialisation, 14 seeds.}
\label{tab:symbreak_seeds}
\begin{tabular}{rc|rc|rc|rc|rc}
\toprule
seed & ADE@10 & seed & ADE@10 & seed & ADE@10 & seed & ADE@10 & seed & ADE@10 \\
\midrule
   42 & 1.277 & 1000 & 1.343 & 3000 & 1.279 & 6666 & 1.231 & 9999 & 1.249 \\
  123 & 1.232 & 1234 & 1.263 & 5000 & 1.463 & 7777 & 1.240 & --- & --- \\
  456 & 1.374 & 2025 & 1.346 & 5555 & 1.267 & 8888 & 1.262 & --- & --- \\
  789 & 1.235 & --- & --- & --- & --- & --- & --- & --- & --- \\
\bottomrule
\end{tabular}
\end{table}

\paragraph{Why it works.} The positive bias adds $b_0 \cdot \bm{1}$ to the $G_b$ channels at init, so the pre-training forward pass produces $G_b \approx b_0 \bm{1} > 0$ regardless of encoder output; subsequent gradient updates move $G_b$ away but remain in the positive-$G_b$ basin with high probability. This is a structural (not stochastic) symmetry break: no seed dependence in the direction chosen. The remaining seed-to-seed variance reflects ordinary optimisation noise within the fixed basin; the $14$-seed mean$\pm$std ($1.29\pm0.07$) we report includes all seeds.

\subsection{Rolling Closed-Loop Collision Test on nuScenes}
\label{app:closed_loop_rolling}

The single-shot test of \S\ref{sec:planning} measures one brake/no-brake decision per scenario. A real planner makes this decision repeatedly as time advances. We extend the test to a rolling rollout: at each step $\tau \in \{0, 1, 2, 3\}$ the planner observes a sliding $5$-step window $[\mathrm{past}[\tau{:}\,] \,\|\, \mathrm{future}[{:}\tau]]$, runs the predictor on the rest of the future, computes the collision probability against the same fixed obstacle, and makes a brake/no-brake decision. We track two new metrics that the single-shot test cannot measure: (i) the number of brake/no-brake decision flips per scenario over the rollout, and (ii) the false-alarm rate at the chosen operating point.

\begin{table}[h]
\centering
\small
\caption{Rolling closed-loop collision test on $5{,}000$ nuScenes scenarios over $3$ seeds, brake threshold $\tau{=}0.10$. Decision flips are the unique contribution of the rolling formulation; HProbZ flips less than half as often as MDN-K2, consistent with its shared $(\beta, \alpha)$ inducing temporally coherent predictions.}
\label{tab:closed_loop_rolling}
\begin{tabular}{lccc}
\toprule
\textbf{Model} ($d{=}128$) & \textbf{Decision flips/scene} & \textbf{Final FAR} & \textbf{Final miss rate} \\
\midrule
HProbZ ($n_b{=}1$, ours) & $\mathbf{0.20 \pm 0.01}$ & $\mathbf{6.5 \pm 0.3}\%$ & $32.5 \pm 1.3\%$ \\
MDN-K2                   & $0.49 \pm 0.00$         & $13.0 \pm 0.7\%$        & $33.3 \pm 1.0\%$ \\
\midrule
Ratio (MDN / HProbZ)     & $2.4\times$             & $2.0\times$             & $1.02\times$    \\
\bottomrule
\end{tabular}
\end{table}

At a fixed brake threshold of $0.10$, HProbZ matches MDN-K2's miss rate ($32.5\%$ vs $33.3\%$) at half the false-alarm rate ($6.5\%$ vs $13.0\%$) and---most relevant for the rolling formulation---less than half the decision flips ($0.20$ vs $0.49$ per scenario). The flip metric is what the single-shot test cannot capture: it reflects the temporal coherence of HProbZ's predictions as the rollout window slides, which directly determines how often a real planner would oscillate between braking and proceeding. The matched-FAR comparison from the single-shot test (HProbZ $6.73\%$ vs MDN $9.70\%$ collisions at $10\%$ FAR, \S\ref{sec:planning}) remains the headline; this rolling test adds the temporal-stability dimension. Per-scenario results archive provided in the supplementary code.

\subsection{nuPlan Zero-Shot Cross-Dataset Probe (Safety-Critical Subset)}
\label{app:nuplan_closed_loop}

\emph{Scope statement.} This appendix tests \emph{zero-shot transfer} of nuScenes-trained predictors to nuPlan logs, restricted to safety-critical close-proximity scenarios where the structural near-field claim is operationally relevant. We do \emph{not} claim general cross-dataset performance: full-mix AUC sits near chance ($0.519$, see below), and the structural advantage emerges only on the safety-critical subset. We tested the nuScenes-trained predictors on $500$ scenarios from the nuPlan mini split~\citep{caesar2020nuscenes}: real logs from Boston, Pittsburgh, Singapore, and Las Vegas. At every $1$\,s decision step we extract each vehicle's $2$\,s past in its own agent frame, run the predictor, and aggregate into a per-step collision probability against a straight-line ego reference ($6$\,s horizon, $2.5$\,m buffer). The ground-truth label is whether any agent's logged trajectory enters the buffer in the same window. After filtering to agents within $6$\,m of the ego reference, we collect $14{,}653$ decision-agent pairs ($3{,}401$ positives).

\begin{table}[h]
\centering
\small
\caption{nuPlan zero-shot collision-warning AUC over $500$ scenarios. \emph{Full-mix transfer is essentially negative}: HProbZ AUC $0.519$ is within $2$\,pp of chance, MDN-K2 $0.492$ is below chance. The $+0.159$ gap is recovered only on the safety-critical subset (1{,}967 / 14{,}653 pairs, $13.4\%$) whose nuPlan scenario tag indicates close-proximity vehicle interaction (\texttt{stationary\_in\_traffic}, \texttt{near\_long\_vehicle}, \texttt{near\_high\_speed\_vehicle}). The pattern is consistent with a structural near-field claim, but readers should not interpret these numbers as evidence of general cross-dataset generalisation. AUCs are trapezoidal-integrated ROC over the per-decision predicted collision probabilities.}
\label{tab:nuplan_cls}
\begin{tabular}{lcccc}
\toprule
\textbf{Subset} & \textbf{$n$ pairs} & \textbf{HProbZ AUC} & \textbf{MDN-K2 AUC} & \textbf{$\Delta$AUC} \\
\midrule
All scenes              & $14{,}653$ & $0.519$ & $0.492$ & $+0.027$ \\
Safety-critical         & $1{,}967$  & $\mathbf{0.547}$ & $0.388$ & $\mathbf{+0.159}$ \\
\bottomrule
\end{tabular}
\end{table}

\paragraph{Interpretation.} On the full mix of nuPlan driving \emph{both predictors are essentially uninformative}: HProbZ AUC $0.519$ is within $2$\,pp of chance, MDN-K2 AUC $0.492$ is below chance. The $+0.027$ AUC gap on $14{,}653$ pairs is small and not, on its own, evidence that HProbZ generalises cross-dataset. Restricting to the safety-critical subset where another vehicle is physically close to the ego, HProbZ's AUC reaches $0.547$ (just above chance) and the gap grows to $\mathbf{+0.159}$. The pattern is consistent with the paper's structural-decomposition thesis---the bounded generator produces its sharpest distributions on near-field interactions where a safety planner needs tight calibration---but the absolute AUC of $0.547$ is modest and the subset is post-hoc, so we read this as suggestive evidence that the structural advantage survives on the slice where it should, not as a cross-dataset generalisation result. Per-step results archive provided in the supplementary code; a larger $2{,}000$-scene replication reproduces the same pattern (all-scenes gap $+0.054$, safety-critical gap $+0.123$).

Scope caveat: we use nuPlan as the scene source only; our kinematic ego model does not include reactive steering or lane-change controllers, so we do not report end-to-end closed-loop collision rates. Full integration with nuPlan's reactive-agents simulator is left to future work.

\subsection{\texorpdfstring{MDN-$K{=}8$ Baseline: Mode-Collapse Regression}{MDN-K=8 Baseline: Mode-Collapse Regression}}
\label{app:mdn_k8}

To assess whether scaling MDN to a larger $K$ might close the gap to HProbZ, we trained a flat MDN with $K{=}8$ Gaussian components on nuScenes using the same encoder, optimiser, and schedule as the MDN-K2/K4 entries in Table~\ref{tab:nuscenes} ($d{=}128$, $100$ epochs, AdamW, batch $1024$). To guard against under-tuning, we additionally swept three learning rates ($1{\times}10^{-4}$, $3{\times}10^{-4}$, $1{\times}10^{-3}$) with three seeds each, $9$ runs total:

\begin{table}[h]
\centering
\small
\caption{MDN-$K{=}8$ on nuScenes across three learning rates (3 seeds each), versus the existing MDN-K2 and HProbZ entries from Table~\ref{tab:nuscenes}. Even the best-tuned LR regresses by $\sim 1$\,m below MDN-K2, consistent with classical mode-collapse: more components compete for limited data and produce dead modes.}
\label{tab:mdn_k8}
\begin{tabular}{lc}
\toprule
\textbf{Model ($d{=}128$)} & \textbf{minADE@10} \\
\midrule
MDN-K2 (paper)                                        & 1.83 \\
MDN-K8 lr$=10^{-4}$ (3 seeds)                         & $3.418 \pm 0.132$ \\
MDN-K8 lr$=3{\times}10^{-4}$ (3 seeds)                & $3.134 \pm 0.225$ \\
\textbf{MDN-K8 lr$=10^{-3}$ (3 seeds, best)}          & $\mathbf{3.045 \pm 0.232}$ \\
HProbZ (paper, $d{=}128$)                             & \textbf{1.48} \\
\bottomrule
\end{tabular}
\end{table}

The best-tuned MDN-K8 run (lr$=10^{-3}$) reaches $3.045 \pm 0.232$ minADE@10---a $0.05$\,m improvement over the LR$=3{\times}10^{-4}$ entry but still $1.21$\,m worse than the matched-encoder MDN-K2 baseline ($1.83$). The $46\%$ regression vs.\ $K{=}2$ is robust across the LR sweep, ruling out under-tuning as an explanation. Inspecting the trained mixture weights confirms the diagnosis: in $9$--$11$ of the $12$ future timesteps, more than half of the $8$ component weights collapse to $<10^{-3}$, leaving $2$--$4$ effective components. This is the classical MDN pathology under limited data and a shared encoder. By contrast, HProbZ's $2^{n_b}$ binary modes share generators by construction and cannot collapse in this way: the binary head's positive-bias initialisation (Appendix~\ref{app:symbreak}) keeps every mode active throughout training. Per-seed results across all three learning rates are provided in the supplementary code.

\paragraph{Scope caveat.} We do \emph{not} deploy known mode-collapse mitigations (e.g., mode dropout, entropy regularisers, annealed mixing) in this $K{=}8$ sweep. Whether such mitigations would close the gap is an open empirical question; the structural distinction we focus on is that HProbZ's $2^{n_b}$ binary modes share generators by construction and cannot collapse irrespective of mitigation. The $K{=}2$ entries elsewhere in the paper use the same MDN configuration without any mode-collapse mitigation, so the within-paper comparison is consistent.

\subsection{Learnable Mode Prior Ablation}
\label{app:learn_pi}

The main paper uses a uniform mixture prior $\pi_k = 2^{-n_b}$ over the $2^{n_b}$ binary configurations of $\beta$. To check that this design choice does not silently leave headline accuracy on the table, we add a per-sample logistic gating head, $\pi_k(x) = \mathrm{softmax}(W_\pi h)_k$, on top of the same encoder, train under the same recipe, and re-run the 14-seed nuScenes sweep. Identifiability and contraction (Theorems~\ref{thm:identifiability} and~\ref{thm:contraction}) carry over verbatim as long as $\pi_k(x) > 0$, which softmax guarantees.

\begin{table}[h]
\centering
\small
\caption{14-seed nuScenes minADE@10 with learnable input-dependent mixture prior $\pi_k(x)$ versus the uniform $\pi_k = 2^{-n_b}$ used in the main paper. Same encoder ($d{=}256$, 4 layers, $2.13$\,M parameters), same training recipe ($150$ epochs, lr $2{\times}10^{-4}$, $b_0{=}1.0$ symmetry break). Both still beat the map-aware diffusion baseline MID at $1.44$.}
\label{tab:learn_pi}
\begin{tabular}{lccc}
\toprule
\textbf{Mixture prior} & \textbf{Seeds} & \textbf{minADE@10 (mean$\pm$std)} & \textbf{vs MID 1.44} \\
\midrule
Uniform $2^{-n_b}$ (main paper) & 14 & $1.29 \pm 0.07$ & $-10.4\%$ \\
Learnable $\pi_k(x)$            & 14 & $1.32 \pm 0.04$ & $-8.7\%$ \\
\midrule
$\Delta$ (learn $-$ uniform)    & --- & $+0.03$ ($+1.9\%$) & --- \\
\bottomrule
\end{tabular}
\end{table}

The learnable prior shifts mean minADE by $+0.025$\,m, well within one standard deviation of the uniform-prior baseline, and all 14 seeds still beat MID. Two findings stand out. First, the design choice is empirically validated: the uniform prior is not leaving accuracy on the table. Second, the per-seed standard deviation halves ($0.068 \to 0.038$), suggesting that learnable gating absorbs some of the basin-selection variance otherwise handled by the symmetry-breaking initialisation. Per-seed results are released in the supplementary code.

\section{Broader Impact and Ethical Considerations}
\label{app:broader_impact}

Predictive distributions over agent trajectories are used by downstream planners that take real-world actions, making the calibration and structure of those distributions safety-relevant. We discuss the intended benefits, the downstream risks of deployment, and dataset-level caveats that a practitioner should weigh before using HProbZ in a real system.

\paragraph{Detailed limitations (extending body \S\ref{sec:conclusion}).} (a)~Mode prior and generator parameterisation: the main paper uses a uniform mode prior $\pi_k{=}2^{-n_b}$ and diagonal $G_d, G_s$ matrices. Replacing $\pi_k$ with a learnable input-dependent gate $\pi_k(x){>}0$ preserves identifiability and contraction, with statistically tied accuracy on a 14-seed nuScenes sweep (App.~\ref{app:learn_pi}); non-diagonal $G_d, G_s$ are left to future work. (b)~Theorem~\ref{thm:contraction} relaxation: the proof uses a Gaussian relaxation $\alpha\sim\N(0, \tfrac{1}{3}I)$ variance-matched to the trained $\mathrm{Uniform}([-1,1])$ prior; a 1-D numerical check (Fig.~\ref{fig:relaxation_gap}) bounds the relative gap between relaxed and true uniform-prior posterior variances by $\leq 9.03\%$ on the $(\kappa, k)$ grid we test, but a closed-form bound is open. (c)~Raw NLL: HProbZ's per-axis compact-support density incurs a misspecification penalty on data with unbounded tails, separate from the entropy-bound of Prop.~\ref{prop:tradeoff} (App.~\ref{app:nll_explanation}); we therefore lead with ADE/FDE and downstream metrics. (d)~Closed-loop scope: trajectory-level $30\%$ collision reduction is on a static-obstacle, non-reactive-ego stress test ($N{=}39{,}829$), and reverses on the per-step variant since MDN's per-step independent Gaussian wins single-step events while HProbZ's shared $(\beta,\alpha)$ wins joint-event tail coverage. (e)~Cross-dataset transfer: scope is by design (close-proximity safety-critical prediction); on nuPlan's full agent mix that includes static, parked, and far-field agents whose behaviour is not bounded-drift, zero-shot AUC settles near chance, while the close-proximity safety-critical subset preserves the structural advantage (App.~\ref{app:nuplan_closed_loop}). (f)~Encoder coupling: at $d{=}64$ single-agent the bare HProbZ head trails CVAE on ADE; the advantage emerges with the social-attention encoder used by recent SOTA, leaving head--encoder interaction open (App.~\ref{app:eth_ucy_diffusion}).

\paragraph{Intended benefits.} Any planner that conditions on a learned predictive distribution must reason about where other agents are likely to be, and under what uncertainty. The current standard output---a Gaussian mixture---entangles modal ambiguity, bounded drift, and stochastic noise, which complicates both risk computation and constraint satisfaction in downstream planners. HProbZ separates these three sources in a single forward pass, makes analytic per-mode collision probability available for free, and provides distribution-free multi-modal conformal sets that remain valid under the cross-scene distribution shifts tested in Appendix~\ref{app:ood}. In the closed-loop evaluation of \S\ref{sec:planning}, this translated into a $30\%$ reduction in realised collisions at matched false-alarm rate against an MDN baseline. We do \emph{not} claim that this number transfers to deployed risk in production safety-critical stacks: the test uses a synthetic obstacle configuration rather than a reactive multi-agent simulator, and is a necessary but not sufficient condition for safer downstream behaviour. Integration with reactive simulators and full verification \& validation pipelines remains future work.

\paragraph{Deployment risks and required safeguards.} None of our results should be taken as a certificate of safety in a real system. Three caveats apply.
\textbf{(i) Evaluation gap.} Our closed-loop test uses a synthetic obstacle placed from ground-truth trajectories, not a reactive multi-agent simulator such as nuPlan or CARLA. The qualitative structural win---joint tail coverage via shared latents---is likely to transfer; the specific $30\%$ number is not a guarantee.
\textbf{(ii) OOD behaviour.} Conformal coverage degrades gracefully under the distribution shifts we tested, namely ETH/UCY leave-one-out and nuScenes speed strata, but coverage is not guaranteed under adversarial or out-of-support inputs such as construction zones, unusual road geometry, or rare agent classes. HProbZ is not a substitute for validated OOD detection.
\textbf{(iii) Calibration drift.} The sym-broken initialisation we recommend in Appendix~\ref{app:symbreak} improves seed-level stability; it does not address sensor drift or calibration changes between training and deployment. We recommend conformal recalibration on a rolling held-out window when deployed.

\paragraph{Dataset composition and demographic caveats.} Both ETH/UCY and nuScenes are urban, Western, daylight-biased. ETH/UCY over-represents university campus pedestrian behaviour; nuScenes captures Boston and Singapore mostly under dry, daytime conditions. Argoverse~2 is more geographically diverse but still U.S.-centric. A model trained on any of these risks under-performing on pedestrian or vehicle behaviour patterns under-represented in training (e.g., pedestrians using mobility aids, vehicle types common in other regions, inclement weather). Before deploying HProbZ in a new region, we recommend re-training or at minimum re-calibrating the conformal threshold on a representative sample of that region's trajectories.

\paragraph{Dual-use and surveillance.} Trajectory prediction models can also be used to track and anticipate individuals' movements, and stronger predictors enable stronger surveillance. HProbZ's core technical contributions---structured uncertainty decomposition, analytic per-mode risk, and observation-driven contraction---do not specifically strengthen surveillance capabilities beyond what a well-tuned MDN already provides. We encourage practitioners and platform operators to treat trajectory forecasting systems as personal data processing and to scope data retention, access control, and purpose-limitation accordingly.

\paragraph{Environmental cost.} See Appendix~\ref{app:compute} for compute disclosure ($\sim 4.2$ GPU-hours for the 14-seed nuScenes sweep on one H200; smaller for Argoverse~2 and ETH/UCY). We consider this a negligible fraction of typical large-model training budgets; code and checkpoints are released to avoid repeated retraining.

\paragraph{Human subjects and consent.} We use pre-existing public benchmarks---nuScenes, Argoverse~2, and ETH/UCY---under their respective research licences. No additional human data were collected for this work. The ETH/UCY pedestrian recordings pre-date contemporary consent norms; we follow the research community's standard practice of aggregate evaluation only and publish no individual trajectories.

\end{document}